\documentclass[10pt,twocolumn,letterpaper]{article}

\usepackage[pagenumbers]{cvpr} 

\definecolor{cvprblue}{rgb}{0.21,0.49,0.74}
\usepackage[pagebackref,breaklinks,colorlinks,allcolors=cvprblue]{hyperref}

\usepackage[T1]{fontenc}
\usepackage{booktabs}
\usepackage[misc]{ifsym}

\usepackage{graphicx}
\usepackage{booktabs}
\usepackage{amsfonts}
\usepackage{tcolorbox}
\usepackage{amsmath,amssymb}
\usepackage{csquotes}
\usepackage{booktabs}
\usepackage{multirow}
\usetikzlibrary{decorations.pathmorphing}
\usepackage{wrapfig}
\usepackage{subcaption}
\usepackage{mathtools}
\usepackage{enumitem}
\usepackage{algorithm}
\usepackage{algpseudocode}   
\usepackage{tikz}
\usetikzlibrary{decorations.pathmorphing, backgrounds, calc, shapes.arrows}
\definecolor{elegantBlue}{RGB}{65, 105, 225} 
\definecolor{elegantOrange}{RGB}{237, 125, 49} 
\usepackage[numbers]{natbib}
\usepackage{balance}

\usepackage{mathtools} 
\usepackage{booktabs} 
\usepackage{tikz} 

\usepackage[accsupp]{axessibility}  

\usepackage{amsthm}

\tcolorboxenvironment{theorem}{
  boxrule=0.3pt,
  colback=blue!5!white,
  colframe=black,
  sharp corners,
  before skip=4pt,
  after skip=4pt,
  fontupper=\small,
}

\newtheorem{theorem}{Theorem}
\newtheorem{lemma}[theorem]{Lemma} 

\def\paperID{*****} 
\def\confName{CVPR}
\def\confYear{2026}

\title{PAPA: Online Personalized Active Preference Alignment}

\author{Anindya Sarkar $^\ast$\\ 
Washington University in St. Louis\\
{\tt\small anindya@wustl.edu}
\and
Nasik Muhammad Nafi \thanks{Equal Contribution.}\\
Oak Ridge National Laboratory\\
{\tt\small nafinm@ornl.gov}
\and 
Isaac Lyngaas\\
Oak Ridge National Laboratory\\
{\tt\small lyngaasir@ornl.gov}
\and 
Muralikrishnan Gopalakrishnan Meena\\
Oak Ridge National Laboratory\\
{\tt\small gopalakrishm@ornl.gov}
\and 
Yevgeniy Vorobeychik \\
Washington University in St. Louis\\
{\tt\small yvorobeychik@wustl.edu}
}

\begin{document}
\maketitle

\begin{abstract}
Diffusion models are highly effective at modeling complex data distributions, including images and text. However, in applications like personalized recommender systems, the objective often shifts to modeling specific regions of the distribution that maximize user preferences—initially unknown but gradually uncovered through interactive feedback. This can naturally be framed as a reinforcement learning problem, where the goal is to fine-tune a diffusion model to maximize a reward function based on preferences. However, the main challenge lies in learning a parameterized reward model, which typically requires large-scale preference data—something that is often not feasible in practice. In this work, we introduce Personalized Active Preference Alignment (\emph{PAPA}), a novel method that bypasses the requirement for a parametrized reward model by directly optimizing the diffusion model using real-time user feedback. \emph{PAPA} enables feedback-efficient preference alignment, drawing inspiration from the variational inference framework. We demonstrate \emph{PAPA}'s effectiveness through extensive experiments and ablation studies across diverse class-conditioned and fine-grained alignment tasks. Additionally, based on theoretical insights, we propose an enhanced fine-tuning strategy, referred to as \emph{EPAPA}, that requires less computational budget and accelerates the fine-tuning process, further boosting \emph{PAPA}'s suitability for real-world deployment. Our code is made publicly available at \url{https://github.com/NasikNafi/papa}. 

\let\thefootnote\relax\footnotetext{Accepted at ECML PKDD 2026}


\end{abstract}

\section{Introduction}
\label{sec:intro}

Diffusion models are deep generative models that generate data by reversing a diffusion process, excelling at capturing complex spaces like natural image manifolds. However, in applications like personalized product recommendations, the goal is to steer generation toward items that align with individual user preferences, which are revealed over time through user activity. A similar challenge arises in other domains as well. For instance, in image generation, diffusion models are trained on vast datasets scraped from the internet, but practical applications often require images with high aesthetic quality. In fact, many other scenarios share this general structure, such as drug discovery, where the goal is to guide generation toward compounds with high bioactivity. This can be framed as a reinforcement learning (RL) problem, where the objective is to fine-tune the diffusion model to maximize a reward that reflects the desired properties of the user’s preferences. However, these methods rely on extensive preference data to learn the reward model, making them ineffective in scenarios like personalized recommendation platforms, where large-scale preference data for each user is unavailable but can be gathered through costly interactive feedback.

\begin{figure*}[t]
    \centering
    \includegraphics[width=0.9\textwidth]{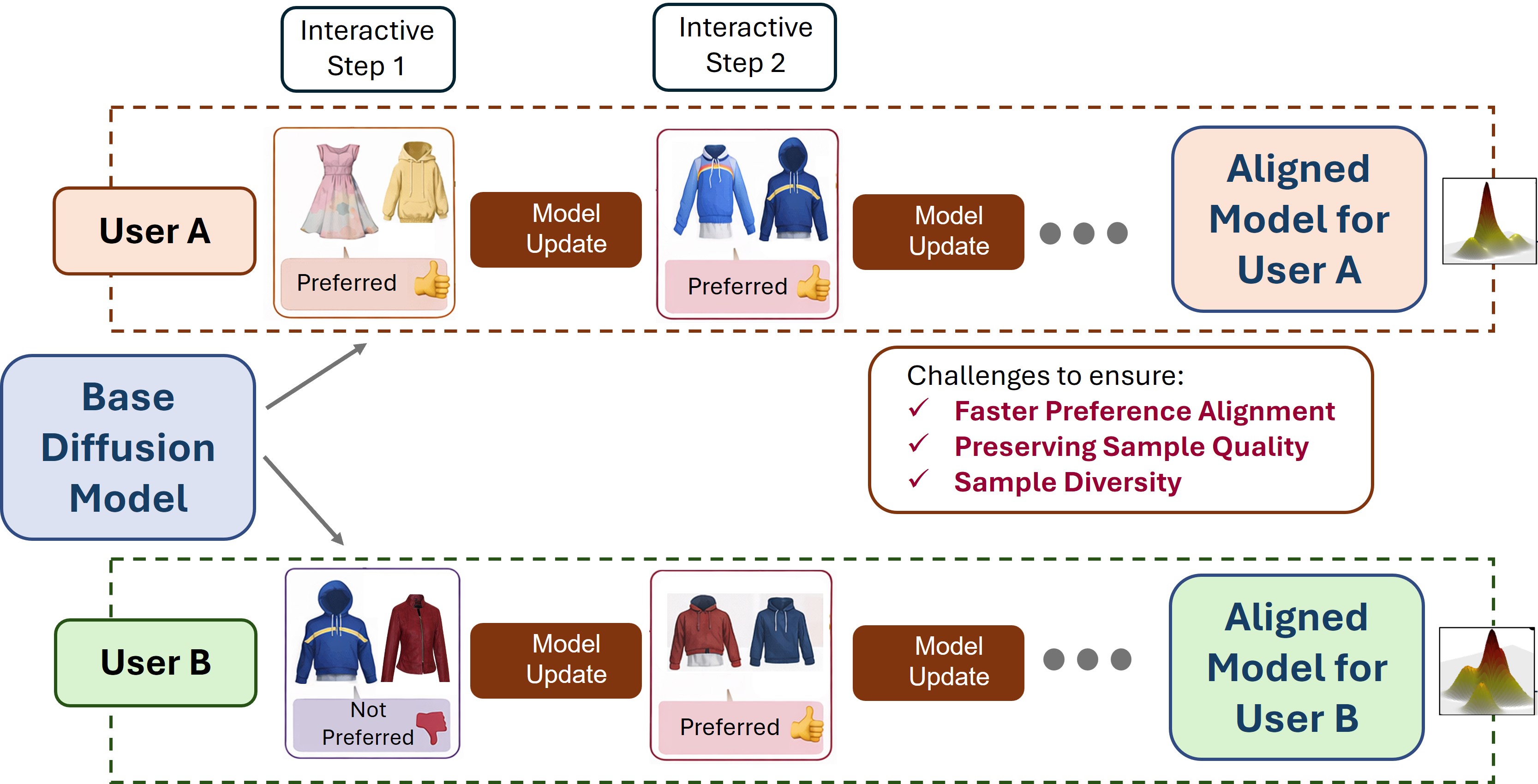}
    \caption{Overview of the considered personalized active preference alignment problem.}
    \label{fig:intro}
\end{figure*}


\noindent
The challenge is twofold: Firstly, achieving this objective requires efficient exploration. However, in high-dimensional spaces, such as those of natural images, this goes beyond simply discovering new regions. It also necessitates respecting the structural constraints of the problem. For instance, in areas like product recommendation, valid solutions—such as realistic-looking products—are typically confined to a lower-dimensional manifold within a much larger design space. Therefore, an effective, feedback-efficient fine-tuning method must explore this space while staying within the feasible area, as venturing outside would lead to wasteful invalid queries. Moreover, fine-tuning the diffusion model to aggressively optimize based on the preferences collected so far can reduce sample diversity. This is because human preferences are often multimodal, and the model, if overly focused on a narrow set of preferences, may fail to capture the full spectrum of diverse user preferences, leading to a less varied sample generation. Therefore, efficient exploration is crucial to maintaining the quality of generated samples and ensuring greater diversity in sample generation.

Secondly, a key challenge in many applications is the high cost of acquiring feedback for the ground-truth reward function. For instance, in a product recommendation system, determining user preferences requires subjective human judgment, which is both costly and time-consuming. This challenge is further compounded by the need for the model to not only explore new options but also to exploit the information it has gathered to generate samples that align with the user's preferences. If the model continues to explore without producing samples that meet the user's expectations, it risks disengaging the user. In a nutshell, the model must strike a balance between exploration and exploitation—effectively generating preference-aligned samples while minimizing costly reward queries. While several recent works have proposed RL-based fine-tuning methods for diffusion models~\citep{black2023training,fan2024reinforcement,uehara2024fine}, none directly tackle the challenge of feedback efficiency in an online setting.~\citep{uehara2024feedback} introduced a framework that accounts for the online nature of feedback but still relies on a separate parameterized reward model for optimization. 
Our goal instead is to develop a feedback-efficient online fine-tuning approach that entirely eliminates the need for a separate reward model, instead directly fine-tuning the diffusion model using real-time user feedback. 

Although direct preference optimization~\citep{liu2023zero} has enabled efficient offline alignment, these methods frequently falter in online interactive settings, due to its reliance on pre-existing large-scale preference data. Critically, prior approaches, such as~\cite{rafailov2024direct}, fail to disentangle sample quality and diversity from alignment itself, resulting in poor performance when preference spans multiple, heterogeneous classes intrinsic to real human feedback. To this end, we propose a principled feedback-efficient online fine-tuning method for diffusion models, derived using the tools of variational inference, aimed at Personalized Active Preference Alignment (PAPA). 

We showcase the effectiveness of PAPA through comprehensive experiments and ablation studies across various domains. Moreover, we propose a pruning-style fine-tuning strategy paired with a sampling approach that leverages both the pre-trained and fine-tuned models. This strategy significantly reduces the training computation requirements of PAPA, thereby accelerating the fine-tuning process. The effectiveness of this approach is supported by both theoretical insights and empirical results, which highlight two key observations: (i) the pre-trained diffusion model retains strong zero-shot denoising capabilities at low noise levels, and (ii) overfitting can occur when using the fine-tuned model for denoising at low noise levels. Below, we summarize key contributions:
\begin{tcolorbox}[colback=blue!5!white, colframe=blue!40!black, boxrule=0.5pt, arc=2mm]
\begin{itemize}[noitemsep,topsep=0pt, leftmargin=*]
    \item We introduce \emph{PAPA}, a novel online fine-tuning method for diffusion models that enables feedback-efficient preference alignment. 
    \item We also propose \emph{EPAPA}, Enhanced-PAPA, a pruning-style fine-tuning strategy that accelerates the fine-tuning process while reducing computational costs without affecting the performance. 
    \item We validate the effectiveness of each component of \emph{PAPA} and \emph{EPAPA} through comprehensive quantitative and qualitative ablation studies across diverse preference alignment settings.
    \item Compared to prior methods, PAPA adapts and aligns more quickly to fine-grained alignment tasks, making it well-suited for real-world interactive alignment scenarios. Notably, PAPA aligns preferences spanning across diverse classes, a capability lacking in previous methods.
\end{itemize}
\end{tcolorbox}

\section{Related Work}
\label{sec:rel_work}

\noindent\textit{Fine-Tune Diffusion model with RL:}
fine-tuning generative models with human feedback, such as user preferences, has become increasingly popular~\citep{ouyang2022training,touvron2023llama}.
Many prior studies have examined fine-tuning diffusion models by optimizing reward functions through methods like supervised learning~\citep{lee2023aligning}, control-based techniques~\citep{ajay2022conditional,janner2022planning}, or policy-gradient~\citep{black2023training}. However, these methods often rely on a static reward model, treating rewards as either fixed ground truth or not allowing for online feedback queries. In contrast, our approach focuses on an online setting that enables interactive preference learning, where the user preferences are initially unknown and gradually revealed through sequential user feedback.
\cite{dong2023raft} presents a general online learning approach for generative models. However, their method is not specifically tailored for diffusion models, and it relies on a separate reward model to select high-quality samples for fine-tuning. These key differences are what fundamentally set our work apart. More recently, ~\cite{uehara2024feedback} proposed a feedback-efficient online fine-tuning approach for diffusion models. However, their method still relies on learning a separate reward model. The SFT method introduced by~\cite{fan2023optimizing} applied RL to diffusion models to improve the performance of existing fast DDPM samplers. ReFL~\cite{xu2024imagereward} leverages the RLHF framework. It first trains a model based on human preferences, and then fine-tunes the diffusion model through reinforcement learning. 
DDPO~\cite{black2023training} frames the denoising process of diffusion models as a Markov Decision Process (MDP) in order to fine-tune the models using multiple reward functions. 
All of these models require a powerful reward model, which in turn necessitates a large dataset of images and comprehensive human evaluations.

\smallskip
\noindent\textit{Direct Preference Optimization:}
In RL, exploring policies based on preferences rather than explicit rewards has gained attention through various methods. Preference-based Reinforcement Learning~\citep{liu2023zero} learns from binary preferences derived from a hidden scoring function, rather than explicit rewards. 
The DPO approach~\cite{rafailov2024direct} was recently proposed to fine-tune LLMs directly using preferences. It capitalizes on the relationship between reward functions and optimal policies, effectively tackling the challenge of reward maximization within a single phase of policy training.
The D3PO method~\cite{yang2024using} was recently introduced to fine-tune diffusion models directly. It operates similarly to DPO but is more cost-effective and reduces computational overhead for diffusion model training. However, these methods are not designed for online settings; they assume access to pre-existing large-scale pairwise preference data to optimize parameters.
\section{Preliminaries}
Denoising \emph{diffusion models (DMs)} 
generate samples from a learned target distribution by progressively denoising white Gaussian noise. More formally, diffusion models generate samples by reversing a forward diffusion process with $T$ steps that start from a data point $x_0$ and evolve as $x_t = \sqrt{\alpha_t} x_{t-1} + \sqrt{(1 - \alpha_t)} \tilde{\epsilon}_t$, $t=1,\ldots, T$, where $\{\tilde{\epsilon}_t\}$ are i.i.d standard Gaussian vectors. Samples from this forward diffusion process can be alternatively expressed as
\begin{equation}
    x_t = \sqrt{\bar{\alpha}_t} x_0 + \sqrt{(1 - \bar{\alpha}_t)} \epsilon_t, \quad \epsilon_t \sim \mathcal{N}(0, I)
\end{equation}
\text{where } $\bar{\alpha}_t = \prod_{s=1}^{t} \alpha_s$. 
\( \{\alpha_t\} \) are chosen such that \( \{\bar{\alpha}_t\} \) forms a monotonic sequence with \( \bar{\alpha}_T \approx 0 \). This ensures that the density \( p_{x_T} \) is close to the normal distribution $\mathcal{N}(0, I)$. The reverse diffusion process is learned by modeling the distribution of \( x_{t-1} \) given \( x_t \) as Gaussian with mean 
\begin{equation}
    \mathbb{E}[x_{t-1} \mid x_t] = \frac{1}{\sqrt{\alpha_t}}\left( x_t - \frac{1 - \alpha_t}{\sqrt{1 - \bar{\alpha}_t}} \epsilon_\theta(x_t, t)\right)
\end{equation}
and covariance \( \sigma_t I \), where \( \epsilon_\theta(\cdot, \cdot) \) is a neural network and \( \{\sigma_t\} \) are fixed hyperparameters. Training is performed by minimizing the ELBO loss, which simplifies to a series of MSE terms as defined below:
\begin{equation}
    L(\theta) = \sum_{t=1}^{T} \mathbb{E}_{x_0, \epsilon_t} \left[ \left\| \epsilon_\theta(x_t, t) - \epsilon_t \right\|^2_2 \right].
\end{equation}
Upon reaching the optimal solution, the neural network approximates the posterior mean, which depends on the timestep as 
\(
    \epsilon_\theta(x_t, t) = \mathbb{E}[\epsilon_t \mid x_t = x_t].
\)

To generate samples, DMs sample \( x_T \sim \mathcal{N}(0, I) \) and then iteratively follow the learned reverse probabilities
to produce a sample of \( x_0 \). Specifically, at each timestep \( t \), the DM takes \( x_t \) as input and predicts the noise \( \epsilon_t \) 
from which \( x_{t-1} \) is obtained by sampling from the reverse distribution~\citep{ho2020denoising}. 

\section{Problem Formulation}
Consider a pre-trained Diffusion model, such as DDPM, denoted as \( g_{\theta^*} \), with initial parameters \( \theta^* \in \Theta \subseteq \mathbb{R}^d \), where \( \Theta \) represents the full parameter space. The model has been trained on a specific dataset \( \mathcal{D} \), consisting of \( m \) independent and identically distributed (i.i.d.) samples \( \{x_i\}_{i=1}^m \) drawn from a distribution \( \mathcal{P}_X \) over the data space \( X \). The goal of the model is to learn the underlying data distribution \( \mathcal{P}_X \). For instance, when the goal is to generate natural images, the cardinality of $|X|$ is enormous. Although the raw design space is vast, the actual feasible and meaningful solutions typically lie within a complex, yet potentially low-dimensional manifold embedded in $X$, denoted as $\mathcal{X}_{feas}$. We consider feedback-efficient online fine-tuning of \( g_{\theta^*} \) for personalized active preference alignment. Specifically, we work in a setting where we do not have any data with feedback initially, but we have a pre-trained diffusion model \( g_{\theta^*} \) trained on \( \mathcal{P}_X \). We aim to fine-tune \( g_{\theta^*} \) to produce a sequence of new models \( g_{\theta^{ft}_t} \) over the period of online interaction so as to maximize the expected user reward
\begin{equation}
    \max_{\{g_{\theta^{ft}_t}\}}\sum_{t=1}^{t=\mathcal{B}}\mathbb{E}_{x_t \sim g_{\theta^{ft}_t}}[r(x_t)],
\end{equation}
where $r(x)$ is the (unknown) user binary reward function with $r(x) = 1$ if the user likes (approves) and $r(x) = 0$ if they dislike the input $x$,  
\( g_{\theta^{ft}_t} \) is the fine-tuned diffusion model used to sample $x_t$ at the $t$'th online interaction step, and $\mathcal{B}$ is the number of input queries we can make to the user to elicit associated rewards.
We assume that $r(x)=0$ whenever $x$ is outside the support of the feasible (or meaningful) set of inputs $\mathcal{X}_{\text{feas}}$.
By leveraging the sequence of fine-tuned models, we aim to efficiently explore the vast sample space $X$ to uncover the user's diverse preferences, while also exploiting the gathered information to generate samples that align with those preferences. Next, we present our proposed approach to tackle the problem at hand. 

\section{Algorithm}
At the very first step of the online interaction process, we leverage the pre-trained diffusion model \( g_{\theta^*} \) to generate samples $\mathcal{D}^1$, since no user feedback data is available at the outset. Based on the recommended samples $\mathcal{D}^1$, the user provides binary feedback on each sample to express their preferences. We denote the set of preferred and non-preferred samples as $\mathcal{D}_{p}$ and $\mathcal{D}_{np}$ respectively, where $\mathcal{D}^1 = \{\mathcal{D}_{p} \cup \mathcal{D}_{np} \}$. A natural objective for fine-tuning the diffusion model is to maximize the probability of generating samples from $\mathcal{D}_{p}$. To this end, we propose a deceptively simple loss function, $\mathcal{L}^{p}(\theta, \mathcal{D}_{p})$, derived by optimizing the variational lower bound on the negative log-likelihood of the preferred data $\mathcal{D}_{p}$, following the approach in~\cite{ho2020denoising}. The definition of $\mathcal{L}^{p}(\theta, \mathcal{D}_{p})$ is as follows:
\begin{equation}
\begin{aligned}
= & \underbrace{\gamma \left( \sum_{x_0 \in \mathcal{D}_{p}}\sum_{t=1}^{T} \frac{1- \alpha_t}{\alpha_t (1 - \bar{\alpha}_{t-1})} \left\| \epsilon_0 - \epsilon_\theta(x_t, t) \right\|^2 \right)}_{\textit{Preference Aligner}}
\label{eq:papa_p_1}
\end{aligned}
\end{equation}
We now demonstrate that optimizing $\mathcal{L}^p(\theta, \mathcal{D}_{p})$ guarantees continuous, monotonic improvement in preference alignment at each fine-tuning step.
\begin{theorem}\label{th:online}
    Let $\mathcal{P}(x_0)$ be the data distribution from the pre-trained diffusion model, and $w(x_0)$ a non-negative weighting function integrable with respect to $\mathcal{P}(x_0)$, proportional to the preference alignment objective. Assume the diffusion model $g_{\theta}$ perfectly learns the distribution implied by the preference alignment loss term~\ref{eq:papa_p_1} at each online interaction step, as defined below: 
    \[
     \small{\underbrace{\{ \sum_{x_0 \in \mathcal{D}_{p}}\gamma\}}_{\textit{$w(x_0)$}} \left(\sum_{t=1}^{T} \frac{1- \alpha_t}{\alpha_t (1 - \bar{\alpha}_{t-1})} \left\| \epsilon_0 - \epsilon_\theta(x_t, t) \right\|^2 \right)}
    \]
    Then, the learned data distribution after $H$ online interaction steps is given by:
    \[
        \mathcal{P}^H_{\theta}(x_o) = \frac{w(x_0)^H \mathcal{P}(x_0)}{N_H} ; N_H = \int w(x_0)^H \mathcal{P}(x_0) \, dx_0
    \]    
\end{theorem}

\vspace{5pt}
\noindent
\textbf{\emph{Remark:}}\label{par:ins_p} Here, $N_H$ is the normalization constant that ensures $\mathcal{P}^H_{\theta}(x_o)$  is a valid probability distribution. We provide detailed proof of Th.~\ref{th:online} in the Appendix. According to Th.~\ref{th:online}, fine-tuning the diffusion model leads to a strict improvement in preference alignment, assuming $w(x_0) > 0$. Interestingly, as $H \rightarrow \infty$, the learned distribution $\mathcal{P}^H_{\theta}(x_o)$ converges to a Dirac-delta function that perfectly aligns with one of the user's preference (Proof is in the Appendix). Since human preferences are inherently multimodal, $\mathcal{L}^p(\theta, \mathcal{D}_{p})$ alone is insufficient for effective active alignment, highlighting the necessity of a diversity-enhancing objective. 

Furthermore, to effectively solve the active preference alignment problem, a sample-efficient approach is essential, one that fully leverages all available user feedback. $\mathcal{L}^p(\theta, \mathcal{D}_{p})$, however, overlooks non-preference data, which is crucial for improving the efficiency of preference alignment learning. According to neuroscientists~\cite{davis2017biology}, in order to actively learn something new, the brain often needs to \enquote{forget} or weaken certain existing neural connections. This is because our brain has limited capacity and must prioritize storing the most relevant information, a process known as \enquote{memory consolidation}. Building on these insights, we propose a novel objective function that utilizes non-preference data for active memory consolidation and addresses all the challenges mentioned (such as diversity) in a unified manner. 

To accomplish this, we draw insights from the Bayesian inference literature.  
According to~\citep{blundell2015weight,ghahramani2000online}, fine-tuned parameters $\theta^{1}_{ft}$ are a sample from the fine-tuned model's parameter posterior distribution $\mathcal{P}(\theta \mid \mathcal{D}_p)$ i.e., $\theta^{1}_{ft} \sim \mathcal{P}(\theta \mid \mathcal{D}_p)$. Similarly, the pre-trained model's parameters $\theta^{*} \sim \mathcal{P}(\theta \mid \mathcal{D}_{p}, \mathcal{D}_{np})$. Motivated by this, at time step 1, we aim to approximate the posterior distribution of the fine-tuned model's parameters \( \mathcal{P}(\theta \mid \mathcal{D}_p) \) as follows:
\begin{equation}
\small
\begin{aligned}
\mathcal{P}(\theta \mid \mathcal{D}_p, \mathcal{D}_{np}) &\propto \mathcal{P}(\mathcal{D}_{p}, \mathcal{D}_{np} \mid \theta) \mathcal{P}(\theta) \\
&\propto \mathcal{P}(\mathcal{D}_{np} \mid \theta) \mathcal{P}(\mathcal{D}_p \mid \theta) \mathcal{P}(\theta) \\
&\propto \mathcal{P}(\mathcal{D}_{np} \mid \theta) \mathcal{P}(\theta \mid \mathcal{D}_p)
\label{eq:posterior}
\end{aligned}
\end{equation}
Eqn.~\ref{eq:posterior} directly follows from Bayes' rule, neglecting the normalizing constant and assuming independence between $\mathcal{D}_{p}$ and $\mathcal{D}_{np}$. As evident from Eqn.~\ref{eq:posterior}, that the posterior distribution $\mathcal{P}(\theta \mid \mathcal{D}_p)$ is intractable, necessitating an approximation in the form $\mathrm{Mapp}(\mathcal{P}(\theta \mid \mathcal{D}_{p}) \approx \phi^*(\theta)$. Here, $\mathrm{Mapp}(.)$ denotes a projection function that transforms an intractable, unnormalized distribution into a normalized one. We employ variational KL-divergence minimization as $\mathrm{Mapp}(.)$ in our approach, as previous studies~\cite{bui2016deep} have shown it outperforms other inference methods~\citep{wild2022generalized,nguyen2017variational,knoblauch2022optimization} for complex models. Therefore, our approach is framed as a variational KL-divergence minimization over a collection of plausible approximate posterior distributions $\mathcal{\psi}$ as follows:
\begin{equation}
\small
    \phi^{*}(\theta) = \arg\min_{\phi(\theta) \in \mathcal{\psi}} \, \mathbb{D}_{KL} \left[\phi(\theta) \, \| \, Z \cdot \frac{\mathcal{P}(\theta \mid \mathcal{D}_p, \mathcal{D}_{np})}{\mathcal{P}(\mathcal{D}_{np} \mid \theta)} \right]
    \label{eq:kl}
\end{equation}
Here, $Z$ represents the intractable normalization constant, which does not depend on the parameter $\theta$. Next, we present a theoretical result that provides the foundation for deriving a tractable optimization objective for Eqn.~\ref{eq:kl}. 
\begin{theorem}
    Assume a Gaussian mean-field approximation in the parameter space, i.e., if the variational prior distribution is \( \phi(\theta) = \prod_{i=1}^{d} \mathcal{N}(\theta_i, \sigma_i^2) \) and the posterior distribution with full data is \( \mathcal{P}(\theta \mid \mathcal{D}_{np}, \mathcal{D}_p) = \prod_{i=1}^{d} \mathcal{N}(\mu_i^*, \sigma_i^{*2}) \), then the following holds:
    \begin{equation}
    \begin{aligned}
    \mathbb{D}_{KL} \left[ \phi(\theta) \, \| \, Z \cdot \frac{\mathcal{P}(\theta \mid \mathcal{D}_{np}, \mathcal{D}_{p})}{\mathcal{P}(\mathcal{D}_{np} \mid \theta)} \right] \gtrsim \\ \quad - \sum_{x_0 \in \mathcal{D}_{np}}\sum_{t=2}^{T} \frac{(1 -\alpha_t)}{ \alpha_t \cdot (1 - \bar{\alpha}_{t-1})} \left\| \epsilon_0 - \epsilon_\theta(x_t, t) \right\|^2 \\
     + \sum_{i=1}^{d} \left[ \frac{(\theta_i - \mu_i^*)^2}{2 \sigma_i^{*2}} -\frac{1}{2} + \mathrm{log}\frac{\sigma_i^{*}}{\sigma_i} + \frac{\sigma_i^{2}}{2 \sigma_i^{2}} \right].
    \end{aligned}
    \end{equation}
    \label{th:1}
\end{theorem}

\begin{figure*}[t]
    \centering
    \includegraphics[width=0.80\textwidth]{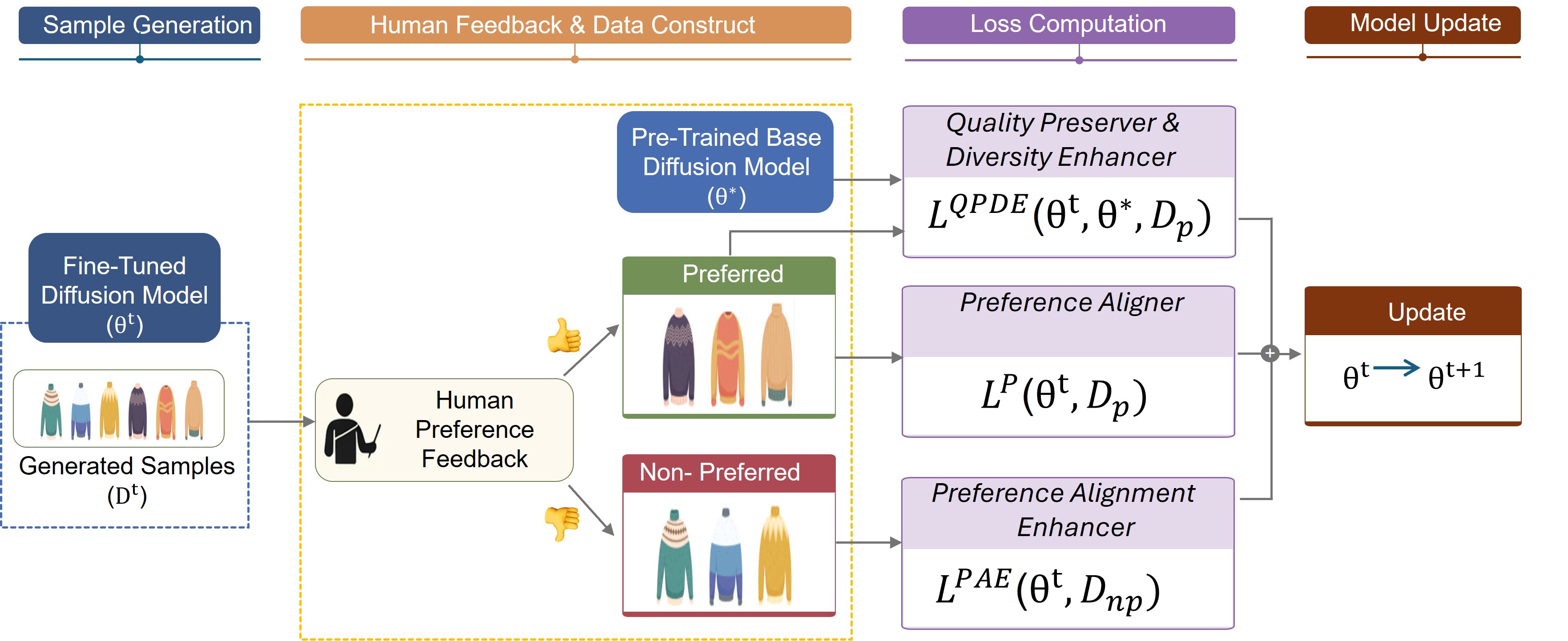}
    \caption{\small An overview of PAPA at step $t$.}
    \label{fig:PAPA}
\end{figure*}

\noindent
We present the proof in Appendix. Finally, when the variational prior distribution is \( \phi(\theta) = \prod_{i=1}^d \mathcal{N}(\theta_i, \sigma_i^2) \) and the posterior distribution with full data is \( \mathcal{P}(\theta \mid \mathcal{D}_p, \mathcal{D}_{np}) = \prod_{i=1}^d \mathcal{N}(\mu^*, \sigma^{*2}) \), Eqn.~\ref{eq:kl} results in the minimization of the following loss function, which we define as the personalized active preference alignment loss $\mathcal{L}^{np}(\theta, \theta^*, \mathcal{D})$
\begin{equation}
\small
\begin{aligned}
= & \underbrace{- \left( \sum_{x_0 \in \mathcal{D}_{np}}\sum_{t=1}^{T} \frac{1- \alpha_t}{\alpha_t (1 - \bar{\alpha}_{t-1})} \left\| \epsilon_0 - \epsilon_\theta(x_t, t) \right\|^2 \right)}_{\textit{Preference Alignment Enhancer}(\mathcal{L}^{\mathrm{PAE}}(\theta, 
 \mathcal{D}_{np}))} \\
&\quad \quad \quad \quad + \underbrace{\beta \sum_{i=1}^{d} \frac{ (\theta_i - \mu_i^*)^2}{2 \sigma_i^{*2}}}_{\textit{Quality Preserver \& Diversity Enhancer} (\mathcal{L}^{\mathrm{QPDE}}(\theta^*, \theta, 
 \mathcal{D}_p))}
\label{eq:papa_np}
\end{aligned}
\end{equation}
The minimization of the $\mathcal{L}^{np}(\theta, \theta^*, \mathcal{D})$ loss function is equivalent to minimizing Eqn.~\ref{eq:kl} by realizing that the first term on the right-hand side of the inequality in Theorem~\ref{th:1} resembles the preference alignment enhancer term from the first part of the loss function $\mathcal{L}^{np}(\theta, \theta^*, \mathcal{D})$ as defined in Eqn.~\ref{eq:papa_np}. Furthermore, the second term in Theorem~\ref{th:1} appears to be the same as the \emph{quality preserver and Diversity Enhancer} term in Equation~\ref{eq:papa_np} if we assume $\sigma_i = \sigma^{*}_i$.
Thus we utilize the proposed loss function $\mathcal{L}^{np}(\theta, \theta^*, \mathcal{D})$ for optimizing the pre-trained model during the fine-tuning process. $\mathcal{L}^{np}(\theta, \theta^*, \mathcal{D})$ consists of two main components: the first, called the \emph{“Preference Alignment Enhancer,”} minimizes the log-likelihood of the non-preferred data, hence aiding in preference alignment, while the second, the \emph{“Quality Preserver and Diversity Enhancer,”} penalizes deviations in the model’s parameters to prevent excessive divergence from their pre-trained values during the fine-tuning process, thus ensuring quality preservation. It also boosts sample diversity by letting the model enjoy the core strengths of the pre-trained model, like its diversity, by staying close to its original parameters $\theta^{*}$.
The sequence \( \{ \alpha_t : t \in T \} \) represents the noise scheduler of the diffusion model, where \( \bar{\alpha}_t = \prod_{i=1}^{t} \alpha_i \). Here, \( \epsilon_0 \) denotes the true added noise, \( \epsilon_\theta(x_t, t) \) is the model’s predicted noise at time \( t \) given the noisy sample $x_t$, and \( d \) is the dimension of the parameter space. To implement \emph{“Quality Preserver and Diversity Enhancer”}, we align the representations of the pre-trained ($\epsilon_{\theta^*}$) and fine-tuned denoisers ($\epsilon_{\theta}$), both processing the same noisy sample ($x_t$), as defined below: 
\begin{equation}
  \small
 \mathcal{L}^{\mathrm{QPDE}}(\theta^*, \theta, \mathcal{D}_p) = \sum_{x_0 \in \mathcal{D}_{p}}\sum_{t=1}^{T} \left\| \epsilon_{\theta^*}(x_t, t)) - \epsilon_\theta(x_t, t) \right\|^2 
\end{equation}
It is important to note that, we consider solely samples from the preference set ($\mathcal{D}_p$) for this objective, as optimizing it for non-preference set elements ($\mathcal{D}_{np}$) contradicts the preference alignment enhancer goal. 
Finally, our proposed method, PAPA, integrates the preference alignment objectives, $\mathcal{L}^{\mathrm{PAE}}(.)$ and $\mathcal{L}^{\mathrm{P}}(.)$, driving exploitation by optimizing the model based on the accumulated preference data, alongside $\mathcal{L}^{\mathrm{QPDE}}(.)$, which fosters exploration by ensuring the fine-tune model remains close to the pre-trained model. We define our final $\mathcal{L}_{\mathrm{PAPA}}(.)$ objective below:
\begin{equation}\label{eq:unified}
  \small
  \underbrace{\mathcal{L}^{p}(\theta, \mathcal{D}_{p}) +\mathcal{L}^{\mathrm{PAE}}(\theta, \mathcal{D}_{np})}_{\textit{Exploitation}} + \beta \cdot \underbrace{\mathcal{L}^{\mathrm{QPDE}}(\theta, \theta^{*}, \mathcal{D}_{p})}_{\textit{Exploration}}
\end{equation}

\begin{algorithm}[t]
\caption{Personalized Active Preference Alignment}
\label{alg:1}
\small
\begin{algorithmic}[1]
\State \textbf{Input:}  \( \mathcal{D}_p = \emptyset \), \( \mathcal{D}_{np} = \emptyset \), Pre-trained diffusion model parameter \( \theta^* \), \( \mathcal{B} \), \( \alpha \) \( \eta \), \( \gamma \), Fine-tune model parameter \( \theta^1 \).
\State \textbf{Initialize:} \( \theta^1 \leftarrow \theta^* \)
\For{each interaction step \( t = 1 \) to \( \mathcal{B} \)}
    \State Generate samples $\mathcal{D}^t = \mathcal{D}^t_p \cup \mathcal{D}^t_{np}$ following a series of reverse-diffusion steps with $\theta^{t}$ as the \emph{denoiser}. 
    \State Collect user's preference for samples in $\mathcal{D}^t$ \& update \\ $\mathcal{D}_p \leftarrow \mathcal{D}_p \cup \{\mathcal{D}^t \setminus \mathcal{D}^t_{np}$\}; $\mathcal{D}_{np} \leftarrow \mathcal{D}_{np} \cup \{\mathcal{D}^t \setminus \mathcal{D}^t_{p}\}$.
    \State Compute the loss function \( \mathcal{L}^{np}(\theta^t, \theta^*, \mathcal{D}) \) and \( \mathcal{L}^{p}(\theta^t, \mathcal{D}_{p}) \) as defined in~\ref{eq:papa_np} and~\ref{eq:papa_p_1} respectively.
    \State Update the fine-tune diffusion model parameters: \\ \( \theta^{t+1} \leftarrow \theta^t - \eta \nabla_\theta \mathcal{L}_{\mathrm{PAPA}}(\theta_t, \theta^*, \mathcal{D}_{p}, \mathcal{D}_{np} ) \).
\EndFor
\State \textbf{Output:} Final parameters \( \theta_{\mathcal{B}} \).
\end{algorithmic}
\end{algorithm}
Where $\beta$ controls the balance between exploration and exploitation. Importantly, $\beta$ introduces a crucial degree of freedom, allowing independent control over alignment, sample quality, and diversity. This explicit control enables PAPA to remain efficient even when preferences span diverse classes. 
We outline the \emph{PAPA} algorithm in~\ref{alg:1} and provide an illustrative overview in Fig.~\ref{fig:PAPA}. Next, we present \emph{EPAPA}, an enhanced version of PAPA that accelerates the fine-tuning process while reducing the computational cost.

\cite{zhong2024diffusion} observed that a pre-trained diffusion model serves as an effective universal denoiser for lightly corrupted data, adept at identifying and correcting subtle distortions. This capability results in enhanced generation quality when the fine-tuned model is replaced by the original pre-trained one, especially in scenarios with low distortion. However, the suboptimal performance of fine-tuned models for lightly corrupted data points to potential issues such as overfitting, mode collapse, or catastrophic forgetting. Experiments conducted with varying denoising steps demonstrate that the fine-tuning objective should emphasize high-level shaping, which is tied to domain-specific characteristics.

Building on this insight, we propose a sampling strategy that leverages the pre-trained model, which excels at low-level denoising, with the fine-tuned model, which specializes in high-level shaping skills. The inverse relationship between the optimality of the pre-trained and fine-tuned diffusion models is formally justified by the following Theorem.

\begin{figure*}[t] 
    \centering
        \centering
        \includegraphics[width=0.9\linewidth]{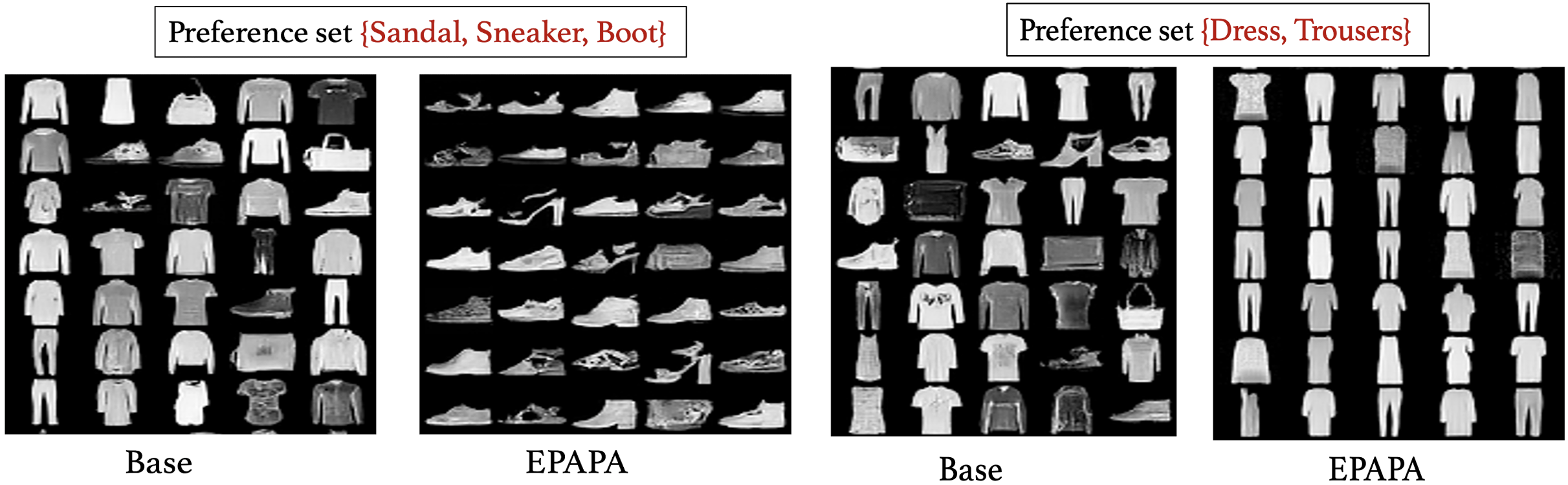}
        \caption{Alignment results for diverse preference set from Fashion MNIST dataset.}
        \label{fig:pa}
        \smallskip
        \centering
        \includegraphics[width=0.9\linewidth]{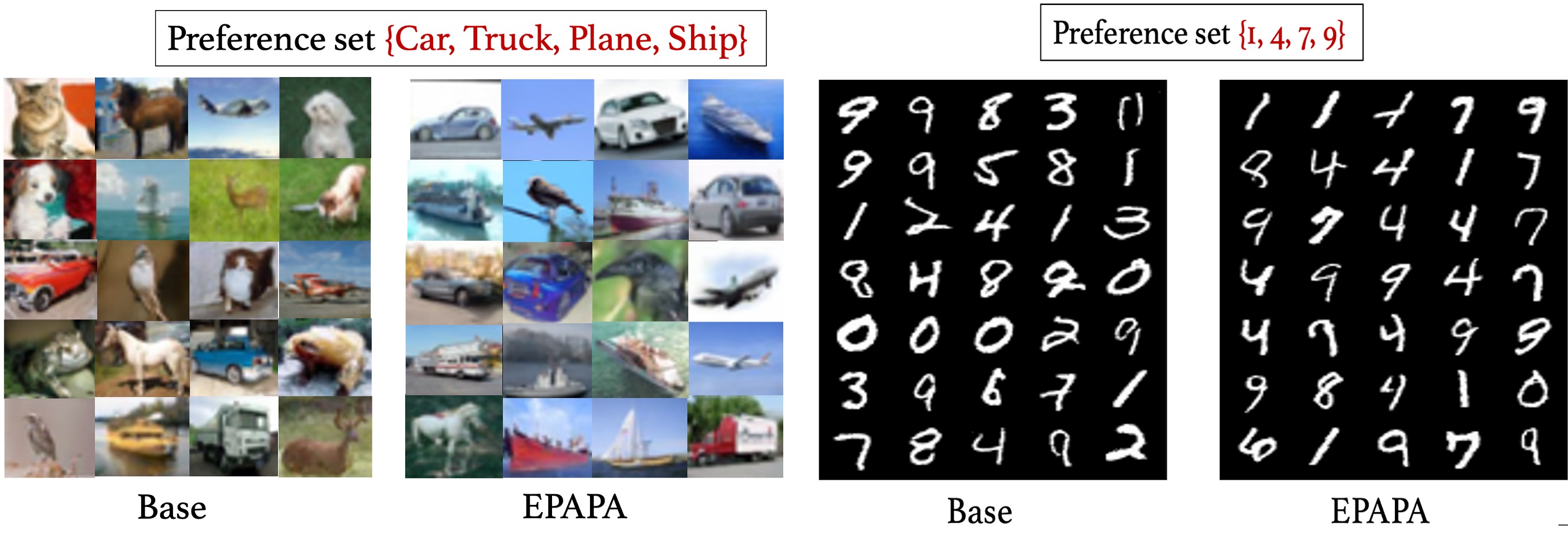}
        \caption{Alignment with different preference set from CIFAR-10 and MNIST.}
        \label{fig:pa-right}
\end{figure*}

\begin{theorem}
    Suppose a diffusion model with $ \lim_{t \to 0} \alpha_t = 1 \quad \text{and} \quad \lim_{t \to T} \alpha_t = 0$ over finite samples. Then the ideal denoiser \( F \) satisfies: (1) \(\lim_{t \to 0} F(x_t) = \arg \min_{x_0} \|x_0 - x_t\|\), i.e., the closest sample in the dataset. (2) \(\lim_{t \to T} F(x_t) = \mathbb{E}_{x_0 \sim p_D(x_0)}[x_0]\), i.e., the mean of the data distribution.
\label{th:2}
\end{theorem}

\vspace{5pt}

\noindent
\textbf{Remark:} Derivation for the above Theorem can be found in~\citep{zhong2024diffusion}. We have added the proof in the Appendix for completeness. According to Th.~\ref{th:2}, as $t \rightarrow 0$, a model trained on a dataset $\mathcal{D}$ can effectively perform zero-shot denoising within the vicinity of the support set $\mathrm{supp}(\mathcal{D})$. As the training dataset grows in scale, the coverage of $\mathrm{supp}(\mathcal{D})$ broadens, empowering diffusion models to serve as robust zero-shot denoisers for data associated with small $t$. This implies that a pre-trained diffusion model, trained on a large-scale dataset, is particularly adept at low-level denoising, as lightly corrupted samples remain close to the support of $\mathcal{D}$. Interestingly, as $t \rightarrow T$, the diffusion model’s generalization power is profoundly shaped by the distribution distance $\mathrm{dist}$($\mathbb{E}_{\mathcal{D}}[x_0], \mathbb{E}_{\mathcal{D}^{ft}}[x_0 ]$), where $\mathcal{D}$ is the dataset used for pre-training and $\mathcal{D}^{ft}$ represents the fine-tuned dataset, defined as $\mathcal{D}^{ft} = \mathcal{D}_{p} \cup \mathcal{D}_{np}$. This insight underscores the need for a fine-tuned diffusion model that excels at denoising heavily corrupted data, honing high-level shaping capabilities.

Inspired by these observations and theoretical insights, we propose \emph{EPAPA}, which utilizes the low-level denoising capabilities of a pre-trained model during the reverse diffusion phase, allowing the fine-tuning process (as defined in~\ref{eq:papa_np},~\ref{eq:papa_p_1}) to focus solely on handling higher levels of noise. Specifically, we exclude the denoising loss terms from both loss functions, \(\mathcal{L}^{np}(\theta^t, \theta^*, \mathcal{D}) \) and \( \mathcal{L}^{p}(\theta^t, \mathcal{D}_{p}) \), as defined in Eqn.~\ref{eq:papa_np} and~\ref{eq:papa_p_1} respectively, for values of \( t \leq K \), where \( 1 \ll K \ll T \). 
This omission speeds up the fine-tuning process and cuts down on computational costs significantly, making it a more efficient alternative to \emph{PAPA}.
Furthermore, by concentrating exclusively on higher noise levels during the fine-tuning process, we empower the fine-tuned diffusion model to specialize in denoising heavily corrupted data, thereby enhancing its high-level shaping capabilities. As a result, during the reverse diffusion process, we employ the fine-tuned model as the denoiser for higher noise levels (larger $t$) and the pre-trained model for lower noise levels. Our sampling process, which utilizes both models, is outlined in~\ref{alg:sampling_epapa}. Moreover, by utilizing a pre-trained model, \emph{EPAPA} inherently promotes both diversity and quality by inheriting the properties of the pre-trained model. 





\section{Experiments and Results}

\begin{algorithm}[t]
\small
\caption{Sampling Strategy of EPAPA}
\begin{algorithmic}[1]
\State \textbf{Input:} \( x_T \sim \mathcal{N}(0, I) \), Pre-trained model $\theta^{*}$, Fine-tuned model $\theta$, $\{\alpha_{t}\}_{t=0}^T, \{\sigma_{t}\}_{t=0}^T$, $K$, $z \sim \mathcal{N}(0, I)$.
\For{$t = T, T-1, \dots, 1$}
    \If{$t \geq K$} 
      \State Calculate \( \epsilon_\theta(x_t, t) \) using the \emph{fine-tuned model}.
      \State Update \( x_{t-1} \) according to:
      $x_{t-1} = \frac{1}{\sqrt{\alpha_t}} \left( x_t - \frac{1 - \alpha_t}{\sqrt{1 - \bar{\alpha_t}}} \epsilon_\theta(x_t, t) \right) + \sigma_t z,$
    \Else 
      \State Calculate \( \epsilon_{\theta^{*}}(x_t, t) \) using the \emph{pre-trained model}.
      \State Update \( x_{t-1} \) according to:
      $x_{t-1} = \frac{1}{\sqrt{\alpha_t}} \left( x_t - \frac{1 - \alpha_t}{\sqrt{1 - \bar{\alpha_t}}} \epsilon_{\theta^{*}}(x_t, t) \right) + \sigma_t z,$
      \EndIf      
\EndFor 
\State \textbf{Output:} Generated sample $x_0$.
\end{algorithmic}
\label{alg:sampling_epapa}
\end{algorithm}

\begin{table*}[t]
\caption{Quality \& diversity analysis with diverse pref. sets for MNIST \& F-MNIST.}
\label{tb:mnist_fmnist_combined}
    \scriptsize          
    \centering

    \begin{minipage}[t]{0.28\linewidth}
        \centering
        \begin{tabular}{p{0.84cm}p{1.4cm}p{1.0cm}}
            \toprule
            \multicolumn{3}{c}{\footnotesize MNIST: set $s_1$} \\
            \midrule
            Method & $\>\>\>\>$FID$\downarrow$ & IS$\uparrow$ \\
            \midrule
            Base   & 39.23$\pm$2.3  & 2.0$\pm$.02 \\
            D3PO   & 41.1$\pm$7.9    & 2.0$\pm$.03 \\
            \hline
            PAPA   & 25.01$\pm$5.2   & 2.0$\pm$.03 \\
            EPAPA  & \textbf{18.87$\pm$5.5} & \textbf{2.0$\pm$.05} \\
            \bottomrule
        \end{tabular}
    \end{minipage}\hfill
    \begin{minipage}[t]{0.22\linewidth}
        \centering
        \begin{tabular}{p{1.5cm}p{1.1cm}}
            \toprule
            \multicolumn{2}{c}{\footnotesize MNIST: set $s_2$} \\
            \midrule
            FID$\downarrow$ & IS$\uparrow$ \\
            \midrule
            30.09$\pm$12.6  & 1.8$\pm$.08 \\
            59.9$\pm$39.8   & \textbf{2.0$\pm$.07} \\
            \hline
            24.08$\pm$12.9  & 1.9$\pm$.05 \\
            \textbf{22.98$\pm$12.1} & 1.9$\pm$.06 \\
            \bottomrule
        \end{tabular}
    \end{minipage}\hfill
    \begin{minipage}[t]{0.22\linewidth}
        \centering
        \begin{tabular}{p{1.5cm}p{1.1cm}}
            \toprule
            \multicolumn{2}{c}{\footnotesize F-MNIST: set $s_1$} \\
            \midrule
            FID$\downarrow$ & IS$\uparrow$ \\
            \midrule
             119.07$\pm$2.2  & 2.6$\pm$.04 \\
             154.33$\pm$50.8 & \textbf{3.4$\pm$.54} \\
            \hline
             80.70$\pm$31.4  & 2.8$\pm$.41 \\
             \textbf{36.72$\pm$16.9} & 2.8$\pm$.34 \\
            \bottomrule
        \end{tabular}
    \end{minipage}\hfill
    \begin{minipage}[t]{0.23\linewidth}
        \centering
        \begin{tabular}{p{1.5cm}p{1.1cm}}
            \toprule
            \multicolumn{2}{c}{\footnotesize F-MNIST: set $s_2$} \\
            \midrule
            FID$\downarrow$ & IS$\uparrow$ \\
            \midrule
            102.39$\pm$2.38 & 2.8$\pm$0.06 \\
            111.62$\pm$10.1 & \textbf{3.6$\pm$0.11} \\
            \hline
            72.76$\pm$14.2  & 2.7$\pm$0.54 \\
            \textbf{40.67$\pm$8.8} & 2.4$\pm$0.30 \\
            \bottomrule
        \end{tabular}
    \end{minipage}

\end{table*}

\begin{figure*}[t]
    \centering
    \begin{subfigure}{0.49\linewidth}
        \centering
        \includegraphics[width=\linewidth]{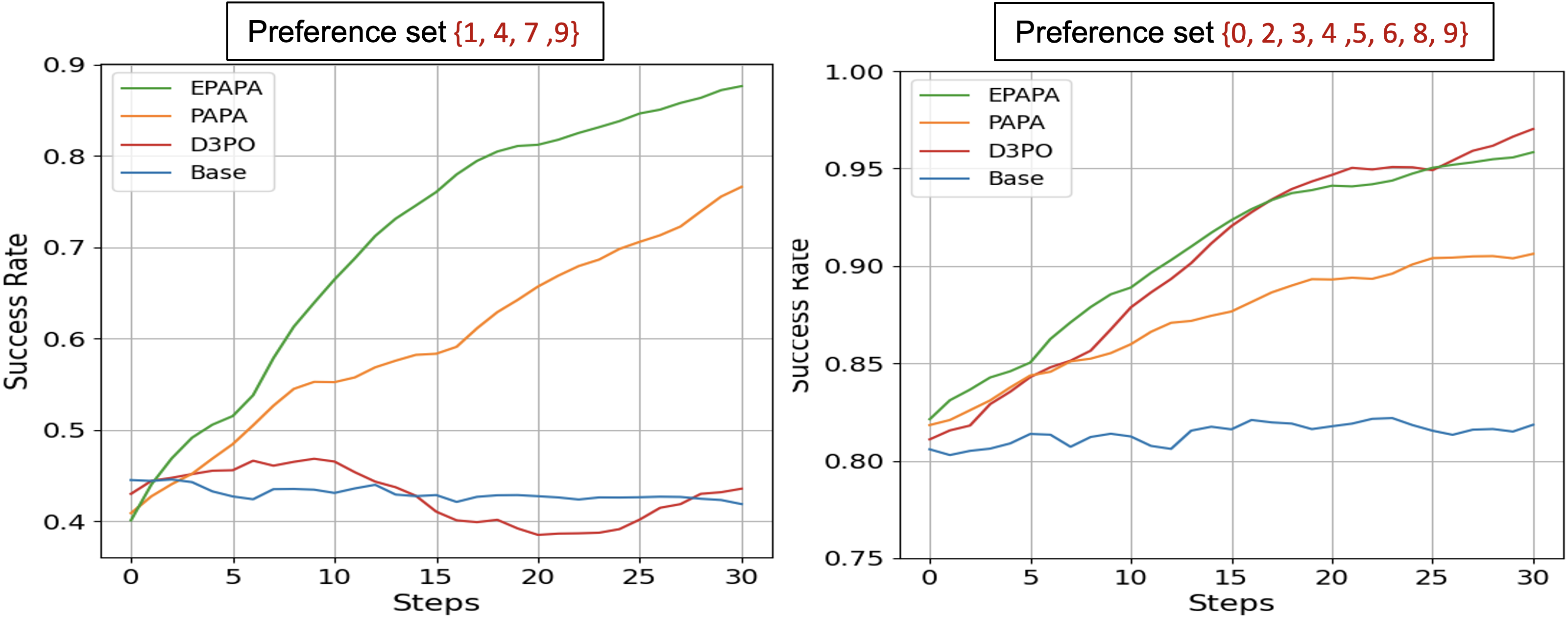}
    \end{subfigure}
    \hfill
    \begin{subfigure}{0.49\linewidth}
        \centering
        \includegraphics[width=\linewidth]{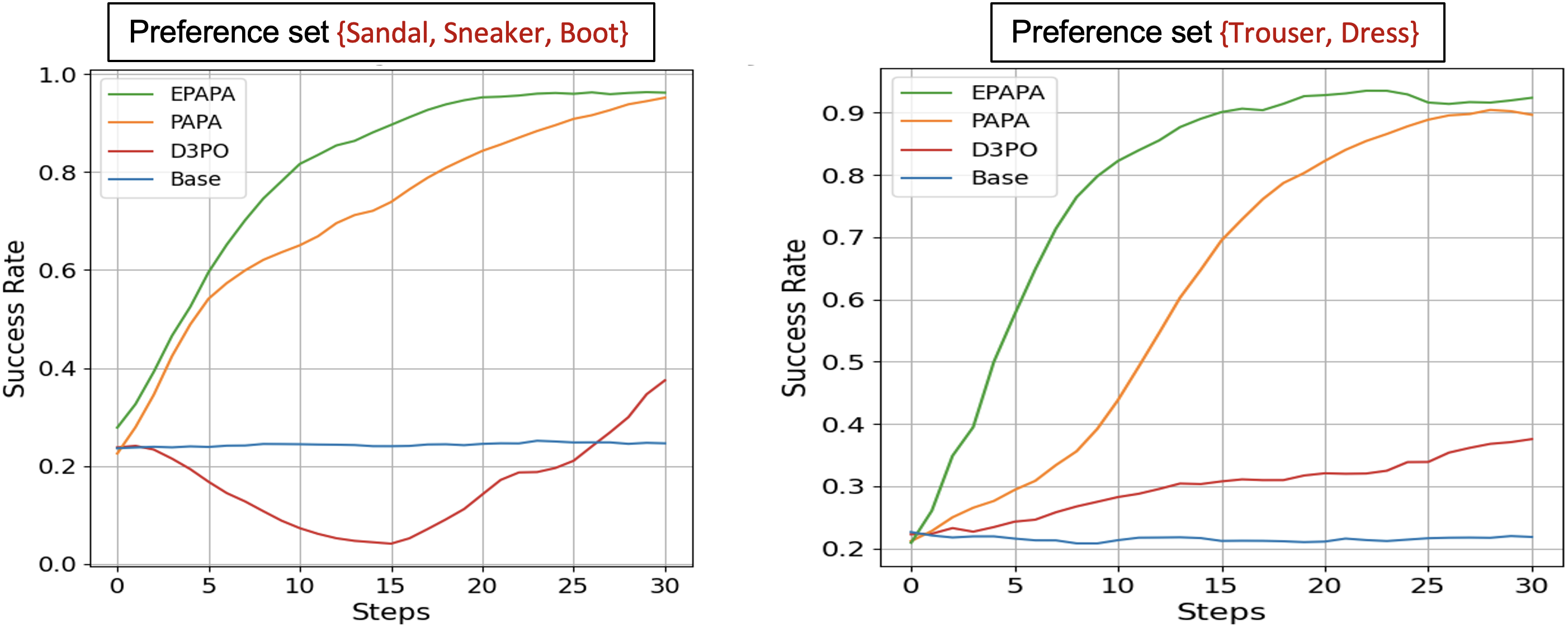}
    \end{subfigure}
 \caption{Success rates across interaction steps for different preference sets from MNIST (left two) and Fashion MNIST (right two).}
 \label{fig:sr}
\end{figure*}

\paragraph{\underline{Evaluation Metric}}
Since \emph{PAPA} aims to align with user preferences while maintaining sample diversity and preserving sample quality, we evaluate its performance using the following metrics, each addressing a specific aspect of the assessment. We evaluate preference alignment using \emph{Success Rate} (SR), which is the proportion of generated samples that are preferred by the user. Specifically, in the $t$-th interaction step, SR is calculated as: $SR = \frac{\mathcal{D}^t_p}{\mathcal{D}^t_p + \mathcal{D}^t_{np}}$.
For a comprehensive evaluation of both the diversity and quality of the generated samples throughout the fine-tuning process, we rely on FID~\citep{heusel2017gans} and IS~\citep{salimans2016improved} metrics. The FID score is calculated by comparing the generated samples from the fine-tuned model with real samples produced by the pre-trained diffusion model trained solely on the preference set as the real distribution. 
We compute the FID and IS scores after each fine-tuning step throughout the interaction phase. Once we accumulate these scores, we calculate their mean and standard deviation to summarize the model's performance. At each online step, we generate 5K samples for evaluation but present only 8 for user feedback, reflecting real-world scenarios where users interact with a limited subset of products before making a decision. This approach models the feedback process with a minimal sample set, aligning with the constraints of actual user interactions. We employ a pre-trained, near-accurate class prediction model or score models as a surrogate for real human feedback. Following sections present details on feedback models and evaluations. 


\begin{figure*}[!t] 
    \centering
    \begin{minipage}{0.48\textwidth} 
        \centering
        \includegraphics[width=1.0\linewidth]{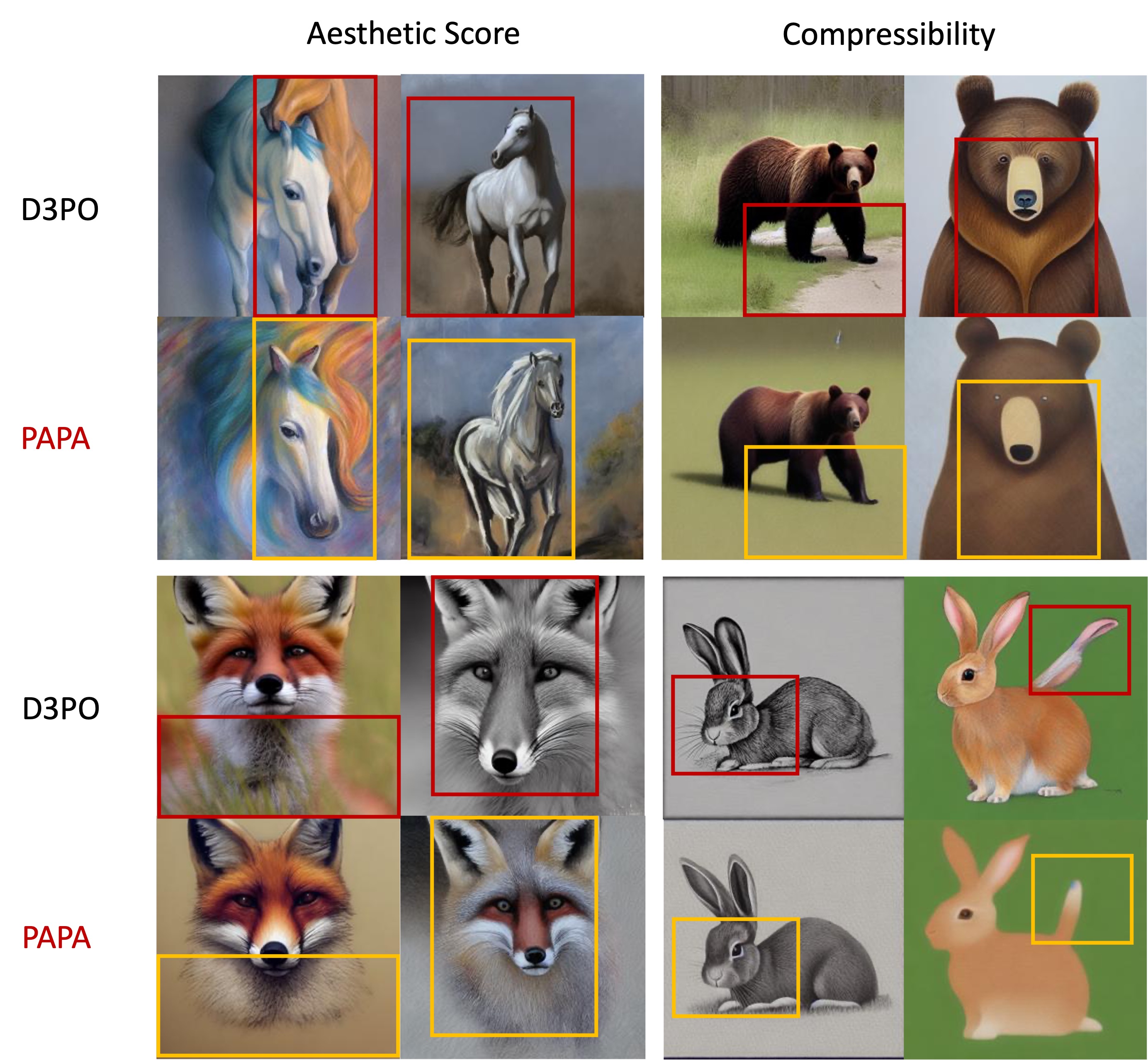}
        \caption{Images with \textit{Aesthetic Quality} and \textit{Compressibility} as preference.}
        \label{fig:stable_diff}
    \end{minipage}\hfill 
    \begin{minipage}{0.48\textwidth} 
        \centering
        \includegraphics[width=0.65\linewidth]{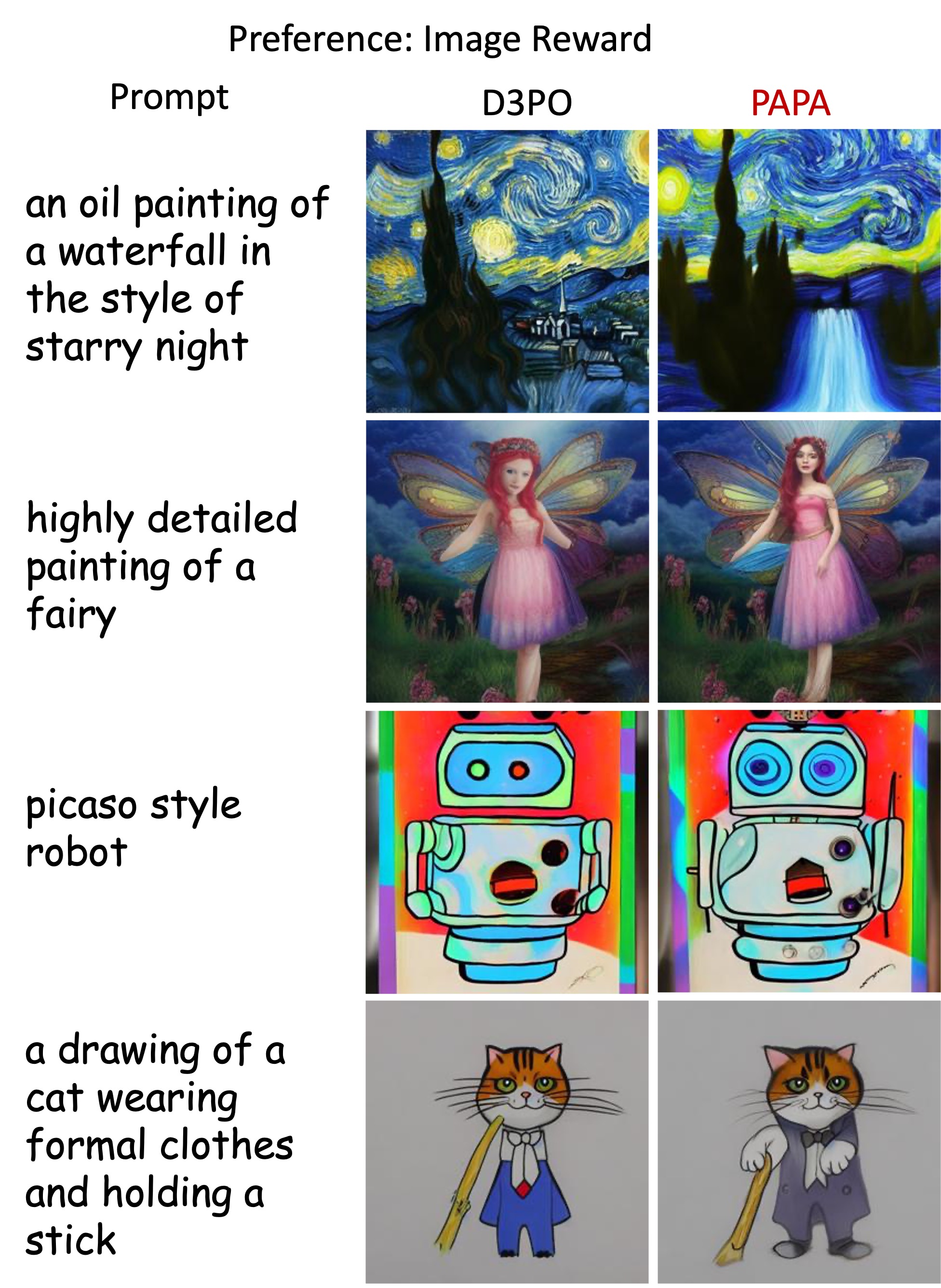}
        \caption{\textit{Image Reward} as preference.}
        \label{fig:stable_diff1}
    \end{minipage}
\end{figure*}

\paragraph{\underline{Datasets and Baselines}} 

We assess the performance of PAPA and EPAPA across: (i) MNIST, (ii) Fashion-MNIST, and (iii) CIFAR-10. We further validate PAPA on preference alignment tasks that are hard to capture through text prompts alone, such as image compressibility and aesthetic quality~\cite{yang2024using}. We demonstrate that our method effectively improves prompt-image alignment through Image Reward \cite{xu2024imagereward} on \textit{Art} subset of Parti Prompt dataset \cite{yuscaling}. We benchmark against D3PO~\cite{yang2024using}, a SOTA reward-free baseline.

\subsection{Class-targeted Alignment}
\label{r:m}
We first evaluate performance on two Fashion-MNIST preference sets with complementary challenges: $s1 \in \{Boot, Sandal, Sneaker\}$, and $s2 \in \{Dress, Trouser\}$. Figure~\ref{fig:sr} (\emph{right}) reports SR for both cases, showing that EPAPA’s SR steadily improves with fine-tuning and additional preference data, outperforming the pre-trained \emph{Base} model.   Similar results on MNIST Fig.~\ref{fig:sr} (\emph{left}) with diverse preference sets $s1 \in \{ 1, 4, 7, 9\}$, where the target set is small, and $s2 \in \{0, 2, 3, 4, 5, 6, 8, 9\}$, which covers most classes, further highlight PAPA’s advantage. Notably, while D3PO performs competitively on $s2$, it fails on $s1$, indicating PAPA’s robustness under larger distribution shifts. D3PO struggles to align diverse human preferences primarily because, unlike PAPA, it lacks an explicit mechanism to independently control sample quality and diversity apart from alignment. Additional analysis in the appendix explaining why D3PO struggles to align diverse preferences. We evaluate sample quality and diversity using FID and IS, with results presented in Table~\ref{tb:mnist_fmnist_combined}. The result indicates that our approach effectively maintains high quality (as our approach achieves a similar IS score compared to \emph{Base} (trained on pref. set)) while ensuring diversity within the preferred set (as suggested by lower FID), as visualized in Figure~\ref{fig:pa-right}. These empirical outcomes also reveal the relative benefits of EPAPA compared to PAPA. Therefore, unless stated otherwise, we use EPAPA with $K=400$ throughout the paper. Additionally, we evaluate our approach on CIFAR-10 and observe consistent results (in Appendix).  Fig.~\ref{fig:pa} and Fig.~\ref{fig:pa-right} showcases EPAPA-generated samples, further confirming its efficacy in producing preference-aligned samples. Qualitative comparisons and more generated samples are in the Appendix.
\noindent

\begin{figure}[t]
    \centering
    \includegraphics[width=0.98\linewidth]{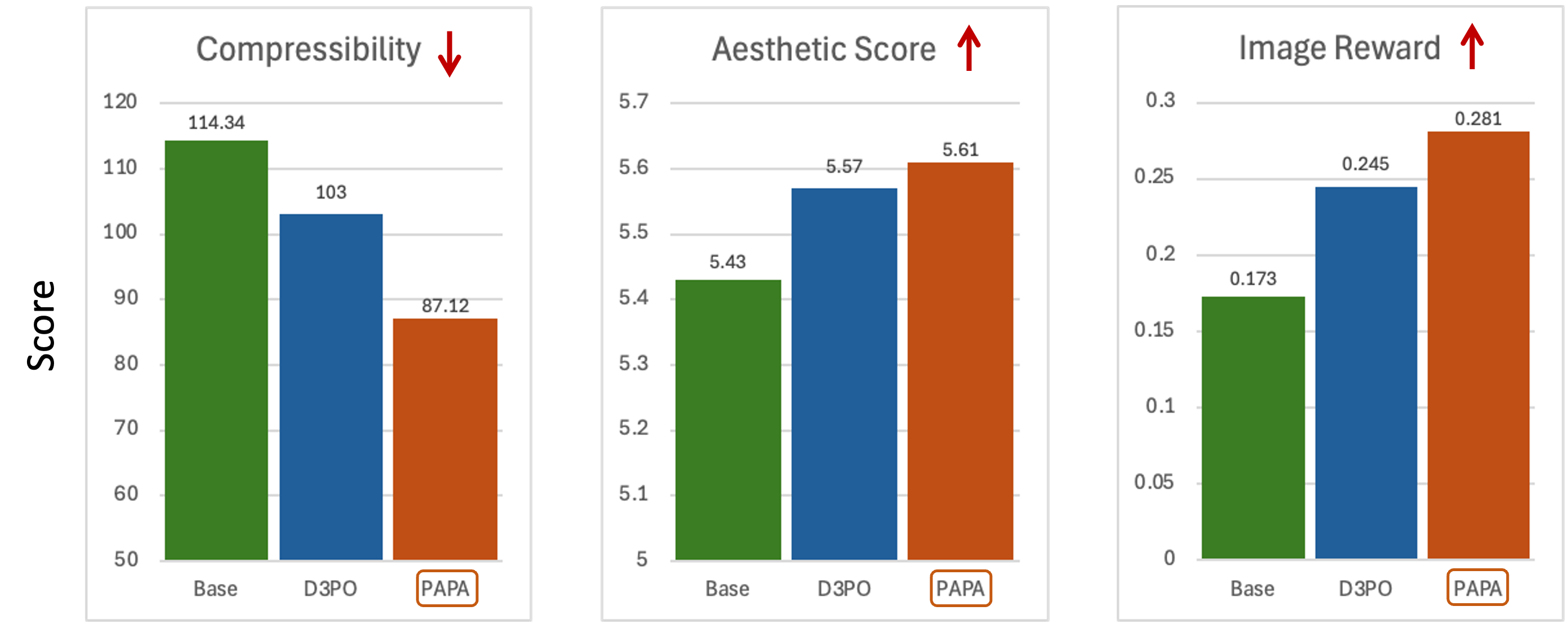}
    \caption{\small Reward comparison on fine-grained objective and prompt-image alignment.}
    \label{fig:score_comparison}
\end{figure}

\begin{figure*}[t]
    \centering
    
    \begin{minipage}[c]{0.50\textwidth}
        \centering
        \scriptsize
        \captionof{table}{Effect of  $\mathcal{L}^{\mathrm{QPDE}}$ (top), $K$ (bottom).}
        \label{tb:k}
        \begin{tabular}{p{1.8cm}p{2.2cm}p{1.5cm}}
      \hline 
      \toprule
      Obj. & FID$\downarrow$ & IS$\uparrow$ \\
      \midrule
      $\mathcal{L}^{noqp}$ & \scriptsize{39.93 $\pm$ 18.3} & \scriptsize{2.8 $\pm$ 0.38} \\
      $\mathcal{L}_{\mathrm{PAPA}}$ & \scriptsize{36.72 $\pm$ 16.9} & \scriptsize{2.9 $\pm$ 0.34} \\
      \bottomrule 
      \medskip
    \end{tabular}
    \label{tb:diversity}
    
    \begin{tabular}{p{1.8cm}p{2.2cm}p{1.5cm}}
      \hline 
      \toprule
      $K$ & FID$\downarrow$ & IS$\uparrow$ \\
      \midrule
      100 & \scriptsize{50.62$\pm$13.25} & \scriptsize{2.9$\pm$0.30} \\
      400 & \scriptsize{\textbf{36.72$\pm$16.9}} & \scriptsize{\textbf{2.9$\pm$0.34}} \\
      700 & \scriptsize{435.40$\pm$106.5} & \scriptsize{1.4$\pm$0.70} \\
      \bottomrule
    \end{tabular}
    
  \end{minipage}
    \begin{minipage}[c]{0.48\textwidth}
        \centering
        \includegraphics[width=0.65\linewidth]{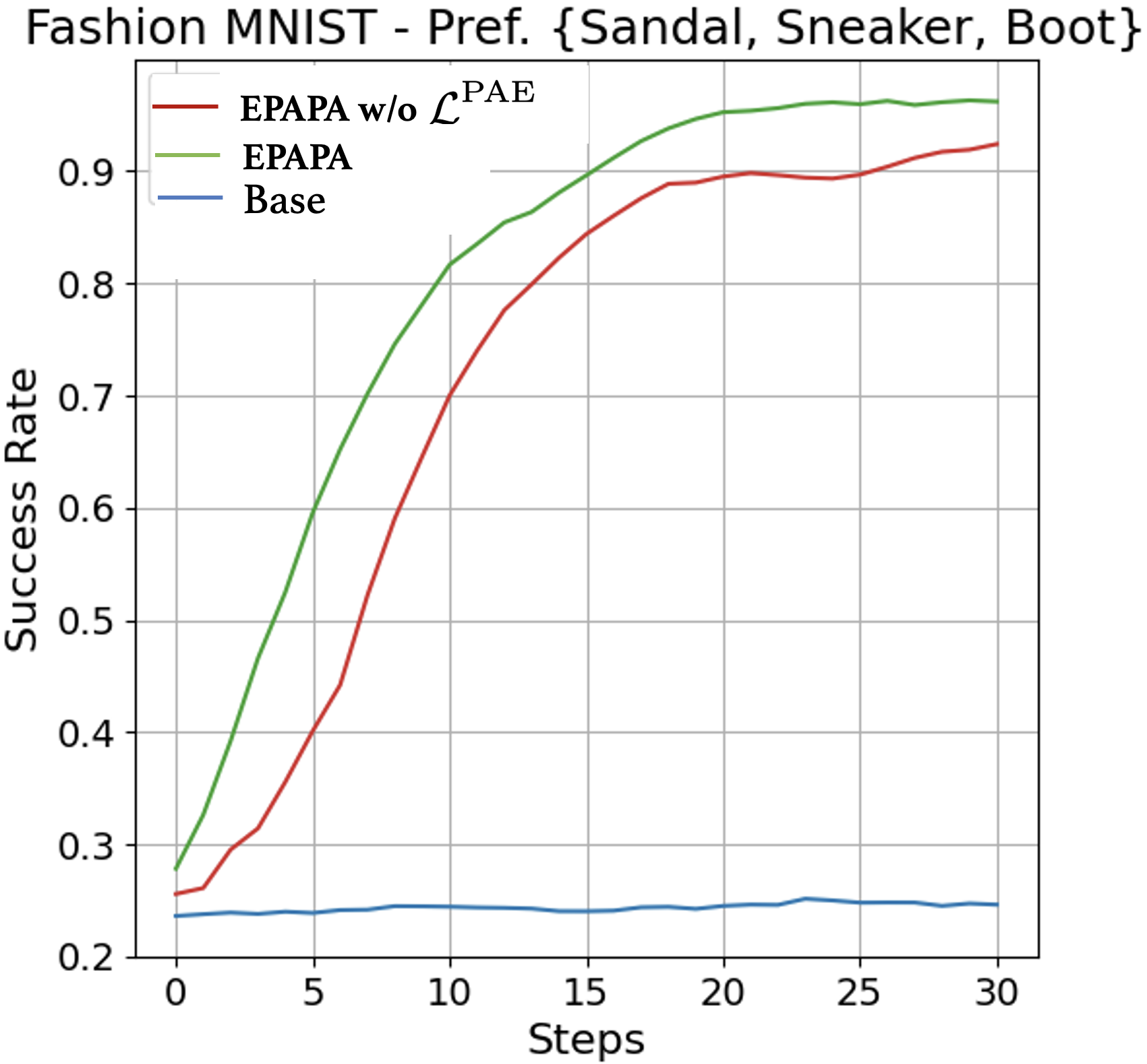} 
        \captionof{figure}{\small Ablation $\mathcal{L}^{\mathrm{PAE}}$.}
        \label{fig:right_figure}
    \end{minipage}
\end{figure*}

\begin{figure}[t]
    \centering
    \includegraphics[width=0.9\linewidth]{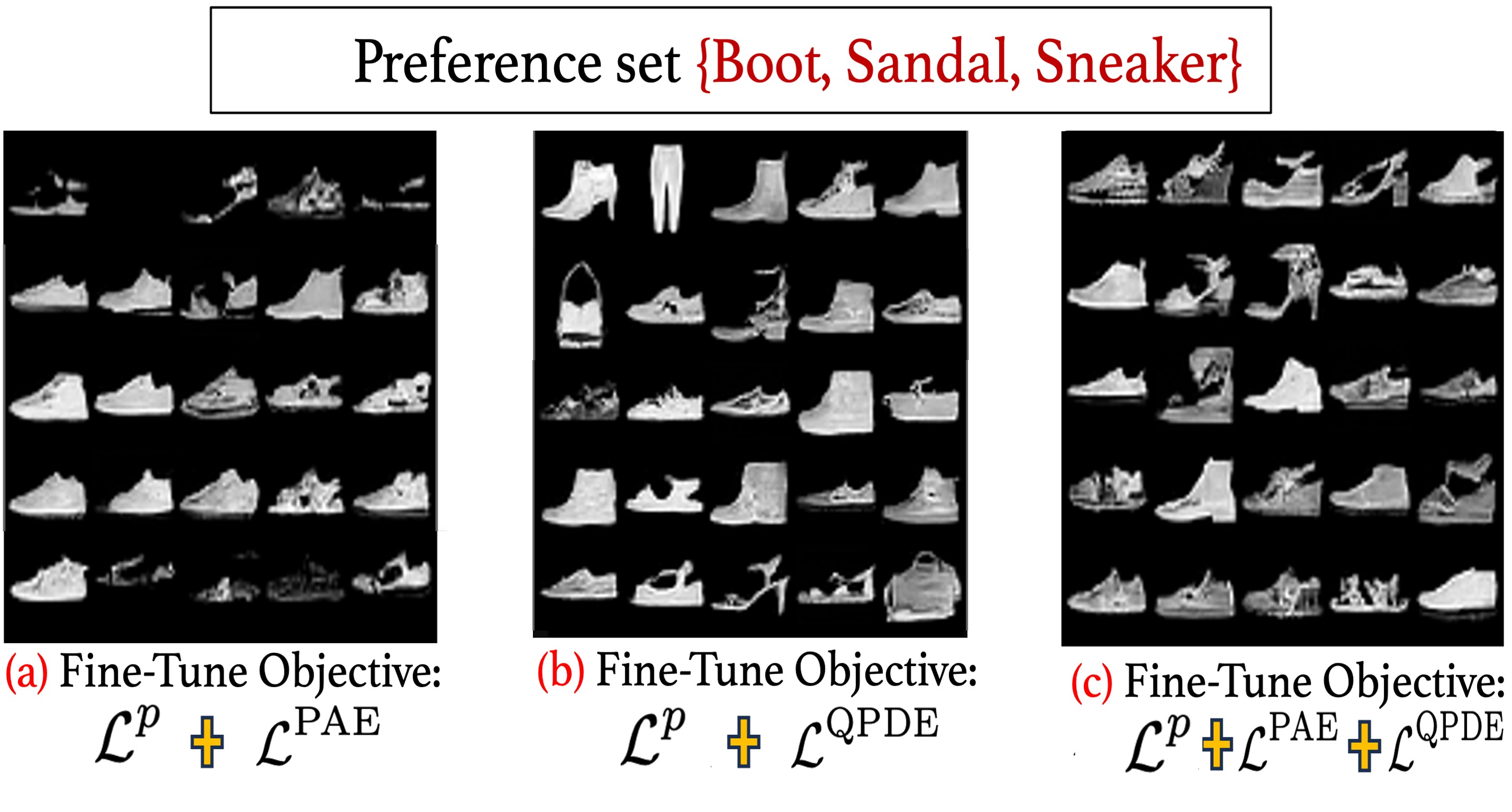} 
    \captionof{figure}{Qualitative performance comparison across variants with individual components removed. Results validate the necessity of each component.}
    \label{fig:left_figure}
\end{figure}

\subsection{Fine Grained Preference Alignment}
We further evaluate the efficacy of PAPA on fine-grained preference alignment tasks with pre-defined quantifiable objectives. We consider three diverse alignment tasks: a) \emph{\textbf{Compressibility}}, in which an image with a smaller size is regarded as better; b) \emph{\textbf{Aesthetic Quality}}, in which we use the LAION aesthetic score predictor~\cite{schuhmann2022laion} to automatically assign aesthetic ratings to images, 
enabling objective reward assignment based on visual quality without requiring human evaluation; c) \emph{\textbf{Prompt-Image Alignment}}, that assesses how well the generated images align with the given text prompts. We use Image Reward \cite{xu2024imagereward}, a general-purpose text-to-image human preference score, to generate feedback and evaluation. For these analysis, we use Stable Diffusion v1.5 as the base model and compare its performance with D3PO~\cite{yang2024using}. Our experimental findings, as reported in Fig.~\ref{fig:score_comparison}, indicate that PAPA rapidly adapts to the preference with comparatively fewer online feedback responses than D3PO (we train on 10-20 times less number of samples), justifying its suitability for sophisticated preference alignment in interactive environments. 
In Figs.~\ref{fig:stable_diff} and \ref{fig:stable_diff1}, we qualitatively compare samples generated from identical Gaussian noises to ensure a fair comparison between the two models. PAPA consistently better reflects the qualitative criteria of the preference objective and is less prone to deformities. For the compressibility preference, although it is quantified via file size, qualitatively it corresponds to images with simplified subject details and smoother, often single-color backgrounds while preserving overall structure. In Fig.~\ref{fig:stable_diff} (rabbit and bear examples), D3PO retains fine details and often introduces structural distortions, whereas PAPA produces coherent, meaningful shapes and removes superfluous details leading to smaller-size images. For the aesthetic preference, which is measured via the LAION aesthetic score~\cite{schuhmann2022laion}, PAPA generates images with sharp subject details, blurred backgrounds (bokeh), appealing composition, and harmonious colors. In Fig.~\ref{fig:stable_diff} (horse and fox examples), D3PO exhibits various deformities and lack of vibrancy, while PAPA preserves body structure, adds stylistic components, enhances fine details, and more clearly separates foreground from background, demonstrating rapid adaptation across fine-grained alignment tasks. Figure \ref{fig:stable_diff1} shows that PAPA generates more semantically aligned images in text-to-image setting. It is evident from the generated images that PAPA outperforms D3PO considering artistic style, object instantiation, spatial arrangement, concept realization, and fidelity. 


\subsection{Analysis and Ablation}
\noindent\underline{\textbf{$\mathcal{L}^{\mathrm{QPDE}}$ as Quality Preserver \& Sample Diversity:}}\label{par:qpde}
To analyze the role of $\mathcal{L}^{\mathrm{QPDE}}$ in preserving quality and diversity, we compare our approach with a variant, $\mathcal{L}^{noqp}$, which is identical to ours but excludes the $\mathcal{L}^{\mathrm{QPDE}}$ loss term from $\mathcal{L}^{np}$ as defined in  Eqn.~\ref{eq:papa_np}. For this comparison, we select $\{$Sandal, Sneaker, Boot$\}$ as the preference set, with results shown in Table~\ref{tb:k} (top). Omitting the $\mathcal{L}^{\mathrm{QPDE}}$ term from $\mathcal{L}^{np}$ leads to a notable decline in both FID and IS scores. Qualitative visualizations, such as the comparison between Figures~\ref{fig:left_figure}(a) and~\ref{fig:left_figure}(c), further highlight how the quality-preserving term maintains the structure of generated samples while ensuring sample diversity. Additionally, we observe that excluding $\mathcal{L}^{\mathrm{QPDE}}$ does not affect SR, as both $\mathcal{L}^{noqp}$ and $\mathcal{L}_{\mathrm{PAPA}}$ achieves similar SR across the interaction steps, as discussed in the Appendix with  visualizations.\\
\noindent
\underline{\textbf{$\mathcal{L}^{\mathrm{PAE}}$ as Preference Alignment Enhancer:}}
To assess the impact of $\mathcal{L}^{\mathrm{PAE}}$ on
preference alignment, we compare our approach with a variant, $\mathcal{L}^{nopa}$, which is identical to ours but omits the $\mathcal{L}^{\mathrm{PAE}}$ term from $\mathcal{L}^{np}$ as defined in Eqn.~\ref{eq:papa_np}. For this comparison, we select $\{$Sneaker, Sandal, Boot$\}$ as the preference set, with results presented in Figure~\ref{fig:right_figure}. Excluding $\mathcal{L}^{\mathrm{PAE}}$ from $\mathcal{L}^{np}$ results in a consistent and substantial decline in SR, underscoring the importance of $\mathcal{L}^{\mathrm{PAE}}$ as the preference-alignment enhancer. Qualitative visualizations (see Figures~\ref{fig:left_figure}(b) and~\ref{fig:left_figure}(c)) further demonstrate how the inclusion of this term helps produce samples that better align with the user preferences. \\ 

\noindent
\textbf{\underline{Effect of $K$:}}
To realize the optimal value of $K$, we conduct
experiments with different choices of K, and compare the performance in terms of FID and IS scores. For this analysis, we select $\{$Sandal, Sneaker, Boot$\}$ as the preference set, with results in Table~\ref{tb:k} (down). Our experiments reveal that both extremely high and low values of $K$ are ineffective. A very high $K$ is problematic because pre-trained diffusion models struggle to denoise effectively at elevated noise levels. Conversely, a very low $K$ is also not ideal, as these fine-tuned models are optimized for higher noise levels and struggle with low noise levels. Therefore, an optimal range for $K$ lies in the middle, where the model performs most effectively. More analysis on the effect of K, $\beta$, and PAPA's generalizability under non-binary feedback is in Appendix.

\section{Conclusion} 
We introduce a novel feedback-efficient approach for fine-tuning diffusion models to achieve personalized active preference alignment. It enables diffusion models to align actively with diverse human preferences while preserving the quality of generated samples. We validate its effectiveness through comprehensive experiments and ablation studies across diverse class-conditioned and fine-grained alignment tasks. We hope this work will pave the way for future research into personalized active preference alignment. Deploying PAPA across a range of scientific problems, from drug discovery to materials generation, would be a potential direction to explore.\\


\noindent
\textbf{Acknowledgments:}
This research used resources of the Oak Ridge Leadership Computing Facility at the Oak Ridge National Laboratory, which is supported by the Advanced Scientific Computing Research programs in the Office of Science of the U.S. Department of Energy under Contract No. DE-AC05-00OR22725. This work was partially supported by the NSF (IIS-2214141), ARO (W911NF-25-1-0059), ONR (N000142412663), Foresight Institute, and Amazon. We thank William Hsu and Beocat High-Performance Computing (HPC) cluster at Kansas State University for initial compute support.

This manuscript has been authored in part by UT-Battelle, LLC, under contract DE-AC05-00OR22725 with the US Department of Energy (DOE). The US government retains and the publisher, by accepting the article for publication, acknowledges that the US government retains a nonexclusive, paid-up, irrevocable, worldwide license to publish or reproduce the published form of this manuscript, or allow others to do so, for US government purposes. DOE will provide public access to these results of federally sponsored research in accordance with the DOE Public Access Plan ( https://www.energy.gov/doe-public-access-plan ).

\balance
{
    \small
    \bibliographystyle{ieeenat_fullname}
    \bibliography{main}
}

\appendix
\onecolumn
\section*{\centering {PAPA: Online Personalized Active Preference Alignment}}



\section{Omitted Proofs}

\subsection{Proof of Theorem 4.1} \label{app:th1}

\begin{proof}

Assume the distribution induced by the pre-trained generative model $\mathcal{P}(x_0)$. Given the standard DDPM loss function:
\[
\mathcal{L} (\theta) = \mathbb{E}_{x_0 \sim \mathcal{P}(x_0)} \underbrace{\left[ \sum_{t=1}^{T} \frac{1- \alpha_t}{\alpha_t (1 - \bar{\alpha}_{t-1})} \left\| \epsilon_0 - \epsilon_\theta(x_t, t) \right\|^2 \right]}_{\textit{Let's denote it as $L(x_0)$}}  \tag{1}
\]
Here, $L(x_0)$ is a loss function defined for sample $x_0$. Now, if we fine-tune the parameters of the diffusion model (i.e., $\theta$) with the objective defined in 1, then according to~\cite{peng2019advantage} the induced data distribution $\mathcal{P}^{induced}(x_0)$ can be defined as:
\[
\mathcal{P}^{induced}(x_0) \propto \mathcal{P}(x_0) \exp (-\gamma L(x_0))    \tag{2}
\]
with $\gamma > 0$ being a positive constant. Now, consider a weighted probability distribution $\mathcal{P}^{weighted}(x_0) = w(x_0)\mathcal{P}(x_0)$. In that case, we can represent the induced distribution $\mathcal{P}^{induced}(x_0)$ in terms of $\mathcal{P}(x_0)$, as follows:
\[
\mathcal{P}^{induced}(x_0) \propto w(x_0)\mathcal{P}(x_0) \exp (-\gamma L(x_0))    \tag{3}
\]
We derive 3, by observing that 
\[
\mathcal{L}(\theta) = \mathbb{E}_{x_0 \sim \mathcal{P}^{weighted}(x_0)}[L(x_0)]\]
This, in turn, implies that 
\[\mathcal{P}^{induced}(x_0) \propto \underbrace{\mathcal{P}^{weighted}(x_0)}_{\textit{$w(x_0)\mathcal{P}(x_0)$}} \exp (-\gamma L(x_0))\]

We can re-write $\mathcal{L}^{p}(\theta)$ as follows:
\[
 = \mathbb{E}_{t \sim U[0,1], x_0 \sim \mathcal{P}(x_0), x_t \sim p_t(x_t|x_0)} \left[ w(x_0) \cdot \underbrace{\frac{1- \alpha_t}{\alpha_t (1 - \bar{\alpha}_{t-1})}}_{\textit{= c > 0}} \| \epsilon_{\theta}(x_t, t) - \epsilon_0) \|^2 \right]  \tag{4}
\]
Where
\[
w(x_0) = \gamma \quad \text{if} \quad x_0 \in \mathcal{D}_{P}, \quad \text{else} \quad w(x_0) = 0.
\]
Note that, $\gamma > 0$. Assuming, there is at least a single element in $\mathcal{D}_p$, implies $w(x_0) > 0$.
\[
\mathcal{L}^p(\theta) = \int_0^1  \int_{\mathcal{X}} w(x_0) \mathcal{P}(x_0) \int_{\mathcal{X}} \| \epsilon_{\theta}(x_t; t) - \epsilon_0) \|^2 \mathcal{P}_t(x_t | x_0) \, dx_t \, dx_0 \, dt \tag{5}
\]
Then, we can define the reweighted distribution \( p^{\text{induced}}(x_0) \) as:
\[
p^{\text{induced}}(x_0) = \frac{w(x_0) \mathcal{P}(x_0)}{Z}, \quad \text{where} \quad Z = \int w(x_0) \mathcal{P}(x_0) \, dx_0 \tag{6}
\]
Since \( w(x_0) \geq 0 \) and \( Z < \infty \), \( p^{\text{induced}}(x_0) \) is a valid pdf over \( \mathcal{X} \).

Substituting $p^{\text{induced}}(x_0)$ (from 6) into the loss function in 5, we have:
\[
\mathcal{L}^p(\theta) = Z \int_0^1  \int_{\mathcal{X}} p^{\text{induced}}(x_0) \int_{\mathcal{X}} \| \epsilon_{\theta}(x_t; t) - \epsilon_0) \|^2 \mathcal{P}_t(x_t | x_0) \, dx_t \, dx_0 \, dt \tag{7}
\]

We can re-write the above expression as:
\[
 = Z \cdot \mathbb{E}_{t \sim U[0,1], x_0 \sim \mathcal{P}^{\text{induced}}(x_0), x_t \sim p_t(x_t|x_0)} \left[ \| \epsilon_{\theta}(x_t, t) - \epsilon_0) \|^2 \right] \>\>\>\> \tag{8} 
\]
Therefore, the gradient of the above expression is:
\[
\mathcal{L}^p(\theta) = (Z)\> \mathbb{E}_{t \sim U[0,1], x_0 \sim \mathcal{P}^{\text{induced}}(x_0), x_t \sim p_t(x_t|x_0)} \left[ \nabla_{\theta} \| \epsilon_{\theta}(x_t, t) - \epsilon_0) \|^2\right]  \tag{9}
\]
 Note that the normalizing factor Z does not depend on the optimization variable $\theta$. Hence, it does not affect the optimization process. Therefore, minimizing $\mathcal{L}^p(\theta)$ is equivalent to minimizing the expected loss under the distribution $\mathcal{P}^{\text{induced}}(x_0)$.
 
 Following a similar result as in Equation 2, we can express the learned data distribution by optimizing Equation 9, denoted as $\mathcal{P}^1_{\theta}(x_0)$ as follows:
 \[
 \mathcal{P}^{1}_{\theta}(x_0) \propto \mathcal{P}^{\text{induced}}(x_0) \exp (-\gamma \mathcal{L}^p(x_0;\theta) ) \tag{10}
 \]
 Utilizing the relation from Equation 6, we can write:
 \[
 \mathcal{P}^{1}_{\theta}(x_0) \propto w(x_0) \mathcal{P}(x_0) \exp (-\gamma \mathcal{L}^p(x_0;\theta) )   \tag{11}
 \]  
 Assuming the $\mathcal{L}^p(x_0;\theta)$ 
 loss converges to $0$ after the fine-tuning step, we can write the above expression as:
 \[
 \mathcal{P}^{1}_{\theta}(x_0) \propto w(x_0) \mathcal{P}(x_0)   \tag{12}
 \] 
 By normalizing, with normalization constant Z, we can write:
 \[
 \mathcal{P}^{1}_{\theta}(x_0) = \frac{w(x_0) \mathcal{P}(x_0)}{Z}   \tag{13}
 \]    

 We can now iteratively repeat these steps, and by induction, express the learned data distribution after $H$ update steps (i.e., $H$ online interaction steps) as follows:
 \[
 \mathcal{P}^{H}_{\theta}(x_0) = \frac{w(x_0)^{H} \mathcal{P}(x_0)}{Z^H} , \text{where,} Z^H = \int_{\mathcal{X}}w(x_0)^H\mathcal{P}(x_0) dx_0   \tag{14}
 \]   
 This completes the proof.
\end{proof}

\subsection{Limiting case: as $H \rightarrow \infty$}\label{app:th1_inf}
\begin{proof}
    \textbf{(Limiting Case).} Let's now consider an interesting limiting case where \( H \to \infty \) :
    
Assume \( w(x_0) \) attains its maximum at \( x_0^* \).  We define
\[
\phi(x_0) = \frac{w(x_0)}{w(x_0^*)} \tag{15}
\]
which implies \( \phi(x_0^*) = 1 \) since \( \frac{w(x_0^*)}{w(x_0^*)} = 1 \), and \( 0 \leq \phi(x_0) < 1 \) for \( x_0 \neq x_0^* \), since \( w(x_0) < w(x_0^*) \).

Then, we can rewrite \( \mathcal{P}^H_{\theta}(x_0) \) using \( \phi(x_0) \) as:
\[
\mathcal{P}^H_{\theta}(x_0) = \frac{[w(x_0^*)]^H \phi(x_0)^H \mathcal{P}(x_0)}{Z^H} \tag{16}
\]

Then, for \( x_0 \neq x_0^* \), we have \( \phi(x_0)^H \to 0 \) as \( H \to \infty \) since \( \phi(x_0) < 1 \). 

Thus, 
\[ \mathcal{P}^H_{\theta}(x_0) \to 0 \>\>\>\>\text{as}\>  H \to \infty ,  \forall x_0 \neq x_0^*    \tag{17}
\] 

And for \( x_0 = x_0^* \), we have \( \phi(x_0^*)^H = 1 \), and \( \mathcal{P}^H_{\theta}(x_0^*) = \frac{[w(x_0^*)]^H \mathcal{P}(x_0^*)}{Z^H} \).

The normalization constant can be written as:
\[
Z^H = \int_{\mathcal{X}} w(x_0)^H \mathcal{P}(x_0) \, dx_0  = [w(x^*_0)]^H \int_{\mathcal{X}} \phi(x_0)^H \mathcal{P}(x_0) dx_0\tag{18}
\]

Similarly, we can obtain \( Z^H \approx [w(x_0^*)]^H \mathcal{P}(x_0^*) \). 

As \( H \to \infty \): Then, we have the limit behavior:
For \( x_0 \neq x_0^* \):
\[
\mathcal{P}^H_{\theta}(x_0) = \frac{[w(x_0^*)]^H \phi(x_0)^H\mathcal{P}(x_0)}{[w(x_0^*)]^H \mathcal{P}(x_0^*)} \quad = \quad \frac{\phi(x_0)^H \mathcal{P}(x_0)}{\mathcal{P}(x_0^*)} \rightarrow 0. \tag{19}
\]
For \( x_0 = x_0^* \):
\[
\mathcal{P}^H_{\theta}(x_0^*) = \frac{[w(x_0^*)]^H \mathcal{P}(x_0^*)}{[w(x_0^*)]^H \mathcal{P}(x_0^*)}\quad = 1. \tag{20}
\]

Therefore, we conclude:
\[
\lim_{H \to \infty} \mathcal{P}^H_{\theta}(x_0) = \delta(x_0 - x_0^*). \tag{21}
\]

\end{proof}

\subsection{Proof of Theorem 4.2}\label{app:th2}
\begin{proof}
The objective function in equation 7 (in the main paper) can be expressed as follows: 

\[
D_{KL} \left( \phi(\theta) \, \parallel \, Z \cdot \frac{\mathcal{P}(\theta | \mathcal{D}_p, \mathcal{D}_{np})}{\mathcal{P}(\mathcal{D}_{np} \mid \theta)} \right) = \mathbb{E}_{\phi(\theta)} \left[ \ln \frac{\phi(\theta) \mathcal{P}(\mathcal{D}_{np} \mid \theta)}{Z \cdot \mathcal{P}(\theta | \mathcal{D}_p, \mathcal{D}_{np})} \right] \tag{1} 
\]

\[
= \mathbb{E}_{\phi(\theta)} \left[ \ln \frac{\phi(\theta)}{ \mathcal{P}(\theta | \mathcal{D}_p, \mathcal{D}_{np})} \right] + \mathbb{E}_{\phi(\theta)} \left[ \ln \mathcal{P}(\mathcal{D}_{np} | \theta) \right] \textit{; (we ignore $Z$ as it is independent of $\theta$)} \tag{2}
\]

\[
=\mathbb{E}_{\phi(\theta)} \left[ \ln \frac{\phi(\theta)}{ \mathcal{P}(\theta | \mathcal{D}_p, \mathcal{D}_{np})} \right] + \mathbb{E}_{\phi(\theta)} \left[ \sum_{x_0 \in \mathcal{D}_{np}} \underbrace{\ln \mathcal{P}(x_0 | \theta)}_{\textit{assuming i.i.d assumption on the data}} \right]   \tag{3}
\]


\[
=\underbrace{D_{KL} \left( \phi(\theta) \, \parallel \, \mathcal{P}(\theta | \mathcal{D}_p, \mathcal{D}_{np}) \right)}_{\textit{term-I}} + \underbrace{\mathbb{E}_{\phi(\theta)} \left[ \sum_{x_0 \in \mathcal{D}_{np}} \ln \mathcal{P}(x_0 | \theta) \right]}_{\textit{term-II}}  \tag{4}
\]

Now, let's focus on $\textit{term-I}$. The $\textit{term-I}$ in Eq. (4) with the parameter prior distribution \( \phi(\theta) = \prod_{i=1}^{d} \mathcal{N}(\theta_i, \sigma^2) \) and the posterior distribution with full data \( \mathcal{P}(\theta | \mathcal{D}_p, \mathcal{D}_{np}) = \prod_{i=1}^{d} \mathcal{N}(\mu^*, \sigma^{*2}) \) becomes:

\[
D_{KL}(\phi(\theta) \parallel \mathcal{P}(\theta | \mathcal{D}_{np}, \mathcal{D}_p)) = \sum_{i=1}^{d} \left( \ln \frac{\sigma_i^*}{\sigma_i} + \frac{\sigma^{2}_i + (\theta_i - \mu_i^*)^2}{2\sigma_i^{*2}} - \frac{1}{2} \right) \tag{5}
\]
Equation (5) is derived using the following well-known standard lemma, which is presented in \cite{mackay2003information} and stated as follows:

\begin{lemma}
The Kullback-Leibler divergence for two multivariate normal distributions can be expressed as follows:
\[
\begin{aligned}
D_{KL}\!\left( \mathcal{N}(x; \mu_x, \Sigma_x) \parallel \mathcal{N}(y; \mu_y, \Sigma_y) \right)
= \frac{1}{2} \Big(
&\log \lvert \Sigma_y \rvert
- \log \lvert \Sigma_x \rvert
- d \\
&+ \operatorname{tr}(\Sigma_y^{-1} \Sigma_x) \\
&+ (\mu_y - \mu_x)^{\top} \Sigma_y^{-1} (\mu_y - \mu_x)
\Big)
\end{aligned}
\]

\end{lemma}

Next, $term-II$ can be expressed using Monte Carlo estimation as follows:
\[
\mathbb{E}_{\phi(\theta)} \left[ \sum_{x_o \in \mathcal{D}_{np}} \ln \mathcal{P}(x_0 | \theta) \right] \approx \frac{1}{B} \sum_{b=1}^{B} \ln \mathcal{P}(x_0 | \theta_m) \tag{6}
\]

\[
\geq \quad \frac{1}{B} \sum_{b=1}^{B} \left[ - \sum_{x_0 \in \mathcal{D}_{np}} \sum_{t=2}^T \mathbb{E}_{q(x_t \mid x_o)} \text{D}_{KL} [\left( q(x_{t-1} | x_t, x_0) \| p_\theta(x_{t-1} | x_t) \right)] \right] \tag{7}
\]

Equation (7) is derived using the following Lemma~\cite{lu2022fast}:
\begin{lemma} The log-likelihood under the backward diffusion process kernel is given by:
\[
\ln p_{\theta}(x_0) \geq - \sum_{t=2}^T \mathbb{E}_{q(x_t \mid x_0)} \left[ \text{D}_{KL} \left( q(x_{t-1} \mid x_t, x_0) \| p_\theta(x_{t-1} \mid x_t) \right) \right]
\]
\end{lemma}

\begin{proof}
Let \( x_0 \) denote the true data sample. In order to increase the log-likelihood of the data, we maximize the ELBO as follows:
\[
\ln p(x_0) = \ln \int p(x_{0:T}) \, dx_{1:T} \, 
\]
\[
= \ln \int \frac{p(x_{0:T})}{q(x_{1:T} \mid x_0)} \, q(x_{1:T} \mid x_0) dx_{1:T} \, 
\]
\[
= \ln \mathbb{E}_{q(x_{1:T} \mid x_o)} \frac{p(x_{0:T})}{q(x_{1:T} \mid x_0)}
\]

\[
\geq \mathbb{E}_{q(x_{1:T} \mid x_o)} \left[ \ln \frac{p(x_{0:T})}{q(x_{1:T} \mid x_0)} \right] \quad \text{(by applying Jensen's inequality)}
\]

\[
= \mathbb{E}_{q(x_{1:T} \mid x_o)} \left[ \ln \frac{p(x_{T})\prod_{t=1}^T p_{\theta}(x_{t-1} \mid x_t)}{\prod_{t=1}^T q(x_{t} \mid x_{t-1})} \right]
\]
\text{(Utilizing markovian property of forward process)}

\[
= \mathbb{E}_{q(x_{1:T} \mid x_o)} \left[ \ln \frac{p(x_{T}) p_{\theta}(x_0 \mid x_1)\prod_{t=2}^T p_{\theta}(x_{t-1} \mid x_t)}{q(x_{1} \mid x_{0})\prod_{t=2}^T q(x_{t} \mid x_{t-1})} \right]
\]

\[
= \mathbb{E}_{q(x_{1:T} \mid x_o)} \left[ \ln \frac{p(x_{T}) p_{\theta}(x_0 \mid x_1)\prod_{t=2}^T p_{\theta}(x_{t-1} \mid x_t)}{q(x_{1} \mid x_{0})\prod_{t=2}^T q(x_{t} \mid x_{t-1}, x_0)} \right]
\]

\[
= \mathbb{E}_{q(x_{1:T} \mid x_o)} \left[ \ln \frac{p(x_{T}) p_{\theta}(x_0 \mid x_1)}{q(x_{1} \mid x_{0})} + \ln \prod_{t=2}^T \frac{ p_{\theta}(x_{t-1} \mid x_t)}{ q(x_{t} \mid x_{t-1}, x_0)} \right]
\]

\[
= \mathbb{E}_{q(x_{1:T} \mid x_o)} \left[ \ln \frac{p(x_{T}) p_{\theta}(x_0 \mid x_1)}{q(x_{1} \mid x_{0})} + \ln \prod_{t=2}^T \frac{ p_{\theta}(x_{t-1} \mid x_t)}{ \frac{q(x_{t-1} \mid x_{t}, x_0) q(x_t \mid x_0)}{q(x_{t-1} \mid x_0)}} \right] \quad \text{(Bayes' Rule)}
\]

\[
= \mathbb{E}_{q(x_{1:T} \mid x_o)} \left[ \ln \frac{p(x_{T}) p_{\theta}(x_0 \mid x_1)}{q(x_{1} \mid x_{0})} + \ln \prod_{t=2}^T \frac{ p_{\theta}(x_{t-1} \mid x_t)}{q(x_{t-1} \mid x_{t}, x_0)} + \ln \frac{ q(x_1 \mid x_0)}{ q(x_T \mid x_0)} \right]
\]

\[
= \mathbb{E}_{q(x_{1:T} \mid x_o)} \left[ \ln \frac{p(x_{T}) p_{\theta}(x_0 \mid x_1)}{q(x_{T} \mid x_{0})} + \sum_{t=2}^T \ln \frac{ p_{\theta}(x_{t-1} \mid x_t)}{q(x_{t-1} \mid x_{t}, x_0)} \right]
\]

\[
\begin{aligned}
= \mathbb{E}_{q(x_{1:T} \mid x_0)} [\ln p_{\theta}{(x_0 \mid x_1})] + \mathbb{E}_{q(x_{1:T} \mid x_0)} \left[ \ln \frac{p(x_T)}{q(x_T \mid x_0)} \right] \\ + \sum_{t=2}^T \mathbb{E}_{q(x_{1:T} \mid x_0)} \left[ \ln \frac{ p_{\theta}(x_{t-1} \mid x_t)}{q(x_{t-1} \mid x_{t}, x_0)} \right]
\end{aligned}
\]

\[
\begin{aligned}
= \mathbb{E}_{q(x_{1} \mid x_0)} [\ln p_{\theta}{(x_0 \mid x_1})] + \mathbb{E}_{q(x_{T} \mid x_0)} \left[ \ln \frac{p(x_T)}{q(x_T \mid x_0)} \right] \\
+ \sum_{t=2}^T \mathbb{E}_{q(x_{t}, x_{t-1} \mid x_0)} \left[ \ln \frac{ p_{\theta}(x_{t-1} \mid x_t)}{q(x_{t-1} \mid x_{t}, x_0)} \right]
\end{aligned}
\]

\[
\begin{aligned}
= \mathbb{E}_{q(x_{1} \mid x_0)} [\ln p_{\theta}{(x_0 \mid x_1})] - D_{KL}(q(x_T \mid x_0) || p(x_T)) \\ 
- \sum_{t=2}^T \mathbb{E}_{q(x_t \mid x_0)} [D_{KL}(q(x_{t-1} \mid x_t, x_0)|| p_{\theta}(x_{t-1} \mid x_t))]
\end{aligned}
\] 

\[
= - \sum_{t=2}^T \mathbb{E}_{q(x_t \mid x_0)} [D_{KL}(q(x_{t-1} \mid x_t, x_0)|| p_{\theta}(x_{t-1} \mid x_t))]
\] 
\text{(Other terms can be ignored as they are insignificant.)}
\end{proof}

Before continuing with $term-II$, we present a well-known lemma, which can be found in any diffusion model literature, such as in \cite{lu2022fast}:

\begin{lemma}
Assuming that in the forward diffusion process, the transition kernel is denoted as \( q(x_t | x_{t-1}) \), with the joint posterior distribution given by 
\[
q(x_{1:T} | x_0) = \prod_{t=1}^{T} q(x_t | x_{t-1}),
\]
where each \( q(x_t | x_{t-1}) = \mathcal{N}(x_t; \sqrt{\alpha_t} x_{t-1}, (1 - \alpha_t) I) \). Similarly, for the backward diffusion process, the transition kernel is denoted as \( p(x_{t-1} | x_t) \), with the joint distribution
\[
p(x_{0:T}) = p(x_T) \prod_{t=1}^{T} p_{\theta}(x_{t-1} | x_t),
\]
where \( p(x_T) = \mathcal{N}(x_T; 0, I) \). Thus, after optimizing the diffusion model, the sampling procedure proceeds by sampling Gaussian noise from \( p(x_T) \) and iteratively applying the denoising transitions \( p_{\theta}(x_{t-1} | x_t) \) for \( T \) steps to generate a new sample \( x_0 \). Assuming all the transition kernels are Gaussian, the following holds:

\begin{itemize}
    \item $ q(x_{t-1} | x_t, x_0) = \mathcal{N}(x_{t-1}; \mu_q(t), \sigma_q^2(t) I) \quad \text{with} \quad \mu_q(t) = \frac{1}{\sqrt{\alpha_t}} x_t - \frac{1 - \alpha_t}{\sqrt{1 - \overline{\alpha_t}} \sqrt{\alpha_t}} \epsilon_0$
    \item
     $p_\theta(x_{t-1}|x_t) = \mathcal{N}(x_{t-1}; \mu_\theta(t), \sigma_q^2(t) I) \quad
     \text{w/} \quad \mu_\theta(t) = \frac{1}{\sqrt{\alpha_t}} x_t - \frac{1 - \alpha_t}{\sqrt{1 - \overline{\alpha_t}} \sqrt{\alpha_t}} \cdot \epsilon_{\theta}(x_t, t)$
    \item $\sigma_q^2(t) = \frac{(1 - \alpha_t)(1 - \overline{\alpha}_{t-1})}{ (1 - \overline{\alpha}_t)}$

\end{itemize}
\end{lemma}

Now, utilizing the results of Lemma .1 and .3, we can rewrite expression 7 as follows:

\[
\begin{aligned}
\mathbb{E}_{\phi(\theta)} \left[ \sum_{x_o \in \mathcal{D}_{np}} \ln \mathcal{P}(x_0 | \theta) \right] \geq \\
\quad \frac{1}{B} \sum_{b=1}^{B} \left[ - \sum_{x_0 \in \mathcal{D}_{np}} \sum_{t=2}^T \mathbb{E}_{q(x_t \mid x_o)} \text{D}_{KL} [\left( q(x_{t-1} | x_t, x_0) \| p_\theta(x_{t-1} | x_t) \right)] \right]
\end{aligned}
\]
\[
\geq \quad \frac{1}{B} \sum_{b=1}^{B} \left[ - \sum_{x_0 \in \mathcal{D}_{np}} \sum_{t=2}^T \mathbb{E}_{q(x_t \mid x_o)} D_{KL} \left( \mathcal{N}(x_{t-1}; \mu_q, \Sigma_q(t)) \parallel \mathcal{N}(x_{t-1}; \mu_\theta, \Sigma_q(t)) \right) \right] \tag{8}
\]
\[
\begin{aligned}
\geq\;& \frac{1}{B} \sum_{b=1}^{B} \Bigg[
- \sum_{x_0 \in \mathcal{D}_{np}} \sum_{t=2}^T
\mathbb{E}_{q(x_t \mid x_o)}
\Bigg\|
\frac{1}{2 \sigma_q^2(t)}
\Bigg(
\frac{1}{\sqrt{\alpha_t}} x_t
-
\frac{1 - \alpha_t}{\sqrt{1 - \overline{\alpha_t}} \sqrt{\alpha_t}}
\, \epsilon_\theta(x_t, t)
\\
&\hspace{4.5cm}
- \frac{1}{\sqrt{\alpha_t}} x_t
+ \frac{1 - \alpha_t}{\sqrt{1 - \overline{\alpha_t}} \sqrt{\alpha_t}}
\, \epsilon_0
\Bigg)
\Bigg\|_2^2
\Bigg]
\end{aligned}
\tag{9}
\]

\[
\geq \quad \frac{1}{B} \sum_{b=1}^{B} \left[ - \sum_{x_0 \in \mathcal{D}_{np}} \sum_{t=2}^T \mathbb{E}_{q(x_t \mid x_o)}  \frac{(1 -\alpha_t)^2}{2 \sigma_q^2(t)(1 - \overline{\alpha}_t)(\alpha_t)}  \left\| \epsilon_0 - \epsilon_\theta(x_t, t) \right\|^2_2 \right]  \tag{10}
\]
\text{(By rearranging the terms.)}
\[
\geq \quad \frac{1}{B} \sum_{b=1}^{B} \left[ - \sum_{x_0 \in \mathcal{D}_{np}} \sum_{t=2}^T \mathbb{E}_{q(x_t \mid x_o)}  \frac{(1 -\alpha_t)}{(1 - \overline{\alpha}_{t-1})(\alpha_t)}  \left\| \epsilon_0 - \epsilon_\theta(x_t, t) \right\|^2_2 \right]  \tag{11}
\]
\text{(Using the expression of $\sigma_q^2(t)$.)}
\[
\geq - \sum_{x_0 \in \mathcal{D}_{np}} \sum_{t=2}^T \frac{(1 -\alpha_t)}{(1 - \overline{\alpha}_{t-1})(\alpha_t)}  \left\| \epsilon_0 - \epsilon_\theta(x_t, t) \right\|^2_2 \text{(Using the expression of $\sigma_q^2(t)$.)}  \tag{12}
\]

By combining Equation 5 and Equation 12, we can write:
\[
\begin{aligned}
\mathbb{D}_{KL} \left[ \phi(\theta) \, \| \, Z \cdot \frac{\mathcal{P}(\theta \mid \mathcal{D}_{np}, \mathcal{D}_{p})}{\mathcal{P}(\mathcal{D}_{np} \mid \theta)} \right] \gtrsim \\ \quad - \sum_{x_0 \in \mathcal{D}_{np}}\sum_{t=2}^{T} \frac{(1 -\alpha_t)}{ \alpha_t \cdot (1 - \bar{\alpha}_{t-1})} \left\| \epsilon_0 - \epsilon_\theta(x_t, t) \right\|^2 \\
     + \sum_{i=1}^{d} \left[ \frac{(\theta_i - \mu_i^*)^2}{2 \sigma_i^{*2}} -\frac{1}{2} + \mathrm{log}\frac{\sigma_i^{*}}{\sigma_i} + \frac{\sigma_i^{2}}{2 \sigma_i^{2}} \right].
\end{aligned}
\]
\end{proof}

\subsection{Proof of Theorem 4.3}\label{app:th3}
\begin{proof}

Let's assume that the dataset consists of finite bounded samples \( \mathcal{D} = \{ x^{(1)}, x^{(2)}, \dots, x^{(n)} \} \), and \( f \) is the denoiser designed to minimize \( L(\theta) \) as defined below:
    \[
L(\theta) = \mathbb{E}_{t, x_0, \epsilon} \left[ \left\| \epsilon - f_\theta \left( \sqrt{\alpha_t} x_0 + \sqrt{(1 - \alpha_t)} \epsilon \right) \right\|^2 \right]
\]
Where \( x_0 \sim \mathcal{X} \) represents real samples, \( \epsilon \sim \mathcal{N}(0, I) \) denotes the noise signal, and \( x_t = \sqrt{\alpha_t} x_0 + \sqrt{1 - \alpha_t} \epsilon \) is the perturbed sample at timestep \( t \). Sampling from diffusion models follows a Markov chain, which iteratively denoises from \( x_T \sim \mathcal{N}(0, I) \) to \( x_0 \). For convenience, we convert the denoiser into an \( x_0 \)-parameterization by rearranging the terms and defining \( F(x_t) = \frac{x_t - \sqrt{1 - \alpha_t}f(x_t)}{\sqrt{\alpha_t}} \), and the objective becomes:
\[
L = \mathbb{E}_{t, x_0, x_t} \left[ \| x_0 - F(x_t) \|^2 \right]
 \tag{13}
 \]

    An ideal denoiser \( F \) should minimize the value \( F(x_t) \) for all \( t, x_t \), implying an objective for \( F(x_t) \):
\[
    L_{t, x_t}(F(x_t)) = \mathbb{E}_{x_0 \sim p(x_0 | x_t)} \left[ \| x_0 - F(x_t) \|^2 \right]. \tag{14}
\]

By taking the derivative, it holds that
\[
    0 = \nabla_{F(x_t)} L_{t, x_t}(F(x_t)) = \mathbb{E}_{x_0 \sim p(x_0 | x_t)} [-2(x_0 - F(x_t))]. \tag{15}
\]

And finally,
\[
    F(x_t) = \mathbb{E}_{x_0 \sim p(x_0 | x_t)} [x_0]. \tag{16}
\]

That is,
\[
    F(x_t) = \int_{x_0} x_0 \cdot p(x_0 | x_t) \, dx_0. \tag{17}
\]

Using Bayes' rule, we can rewrite,
\[
    F(x_t) = \int_{x_0} \frac{x_0 \cdot p_{\mathcal{D}}(x_0) p(x_t | x_0) \, dx_0}{p_{\mathcal{D}}(x_t)}. \tag{18}
\]

Or equivalently,
\[
    F(x_t) = \frac{\int_{x_0}  x_0 \cdot p_\mathcal{D}(x_0) p(x_t | x_0) \, dx_0}{\int_{x_0} p_\mathcal{D}(x_0) p(x_t | x_0) \, dx_0}. \tag{19}
\]

Using a normal distribution,
\[
    F(x_t) = \frac{\int_{x_0} \mathcal{N}(x_t; \sqrt{\alpha_t} x_0, (1 - \alpha_t) I) \cdot x_0 \cdot p_\mathcal{D}(x_0) \, dx_0}{\int_{x_0} \mathcal{N}(x_t; \sqrt{\alpha_t} x_0, (1 - \alpha_t) I) \cdot p_\mathcal{D}(x_0) \, dx_0}. \tag{20}
\]

Using Monte-Carlo estimates, we can rewrite the above expression as:
\[
    F(x_t) = \frac{\sum_{x_0 \in \mathcal{D}} \mathcal{N}(x_t; \sqrt{\alpha_t} x_0, (1 - \alpha_t) I) \cdot x_0}{\sum_{x_0 \in \mathcal{D}} \mathcal{N}(x_t; \sqrt{\alpha_t} x_0, (1 - \alpha_t) I) }. \tag{21}
\]

\textbf{Case when \( t \to T \)} 

As \( t \to T \), \( \alpha_t \to 0 \), and thus \( \mathcal{N}(x_t; \sqrt{\alpha_t} x_0, (1 - \alpha_t) I) \to \mathcal{N}(x_t; 0, I) \), which is a constant for varying \( x_0 \). Bringing this back to Eq. (21), it follows that

\[
F(x_t) = \frac{1}{n} \sum_{x_0 \in \mathcal{D}} x_0, \tag{22}
\]

\textbf{Case when} $t \to 0$. As $t \to 0$, $\alpha_t \to 1$. For simplicity, assume that the closest sample to $\mathbf{x}_t$ is unique. Let
\[
    \mathbf{x_0}_{\text{closest}} = \arg \min_{\mathbf{x}_0 \in \mathcal{D}} \left\| \sqrt{\alpha_t} \mathbf{x}_0 - \mathbf{x}_t \right\|^2, \tag{23}
\]

\[
    d = \min_{\mathbf{x}_0 \in D \setminus \{\mathbf{x_0}_{\text{closest}}\}} \left( \left\| \sqrt{\alpha_t} \mathbf{x}_0 - \mathbf{x}_t \right\|^2 - \left\| \sqrt{\alpha_t} \mathbf{x_0}_{\text{closest}} - \mathbf{x}_t \right\|^2 \right) > 0, \tag{24}
\]
\[
0 \leq \left\| \frac{\sum_{x_0 \in \mathcal{D}} \mathcal{N}(x_t; \sqrt{\alpha_t}\mathbf{x}_0, (1 - \alpha_t)\mathbf{I} ) \cdot \mathbf{x_0} }{\sum_{x_0 \in \mathcal{D}} \mathcal{N}(x_t; \sqrt{\alpha_t}\mathbf{x_0}_{\text{closest}}, (1 - \alpha_t)\mathbf{I} )} - \mathbf{x_0}_{\text{closest}} \right\| \tag{25}
\]

\[
\leq \sum_{\mathbf{x_0} \in \mathcal{D} \setminus \{\mathbf{x_0}_{\text{closest}}\}} \left\|  \frac{1}{\sqrt{2 \pi (1- \alpha_t)}} \exp \left( \frac{- \left\| \sqrt{\alpha_t} \mathbf{x}_0 - \mathbf{x}_t \right\|^2 + \left\| \sqrt{\alpha_t} \mathbf{x_0}_{\text{closest}} - \mathbf{x}_t \right\|^2 }{2(1 - \alpha_t)} \right)  \right\| \tag{26}
\]
\[
\leq \sum_{\mathbf{x_0} \in \mathcal{D} \setminus \{\mathbf{x_0}_{\text{closest}}\}} \left\|  \frac{1}{\sqrt{2 \pi (1- \alpha_t)}} \exp\left(-\frac{d}{2(1 -\alpha_t)} \right)  \right\| \to 0 \tag{27}
\]
By setting, $\alpha_t$ $\to$ 1, 
\[
 \frac{\sum_{\mathbf{x}_0 \in \mathcal{D}}\mathcal{N}\left( \mathbf{x}_t; \sqrt{\alpha_t} \mathbf{x}_0, (1 - \alpha_t) \mathbf{I} \right) \cdot \mathbf{x}_0}{\sum_{x_0 \in \mathcal{D}} \mathcal{N}\left( \mathbf{x}; \sqrt{\alpha_t} \mathbf{x_0}_{\text{closest}}, (1 - \alpha_t) \mathbf{I} \right)} \to \mathbf{x_0}_{\text{closest}}  \tag{28}
\]
Similarly, we can write,
\[
 \frac{\sum_{\mathbf{x}_0 \in \mathcal{D}}\mathcal{N}\left( \mathbf{x}_t; \sqrt{\alpha_t} \mathbf{x}_0, (1 - \alpha_t) \mathbf{I} \right)}{\sum_{x_0 \in \mathcal{D}} \mathcal{N}\left( \mathbf{x}; \sqrt{\alpha_t} \mathbf{x_0}_{\text{closest}}, (1 - \alpha_t) \mathbf{I} \right)} \to 1
\]
Hence, $F(\mathbf{x}_t)$ $\to \mathbf{x_0}_{\text{closest}}$ (utilizing the relation in 28). It completes the proof.
\end{proof}

\clearpage

\noindent
\section{Additional Results On Fine Grained Preference Alignment Task}
\begin{figure}[!h] 
    \centering
    \includegraphics[width=0.80\textwidth]{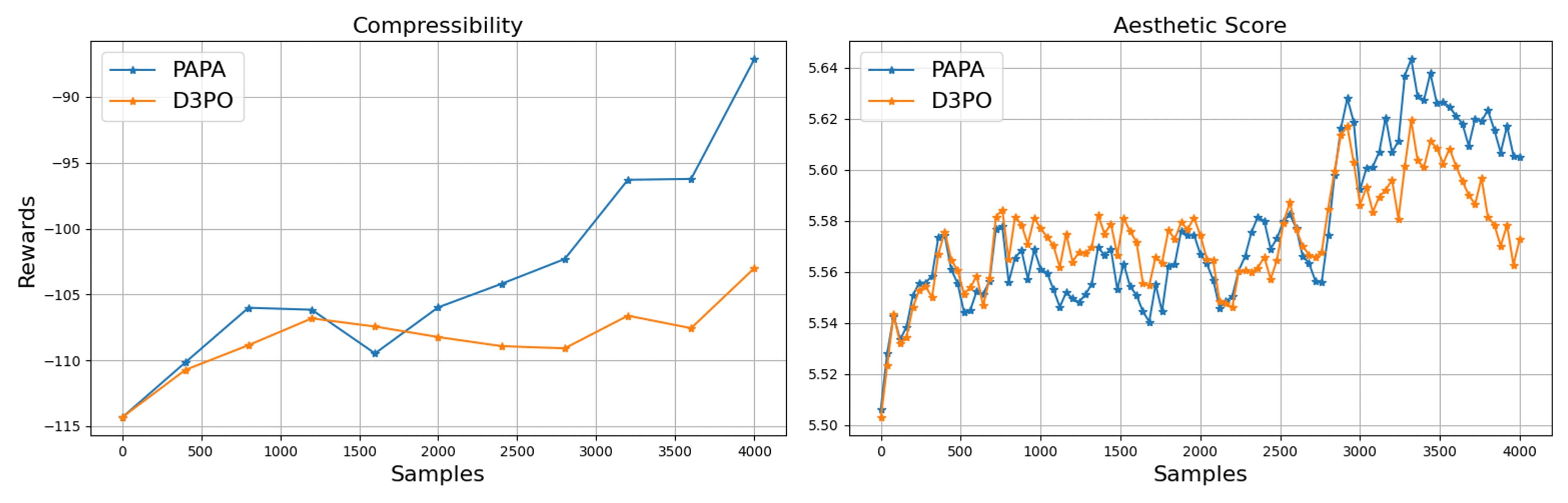}
    \caption{\small{Reward Comparison on Fine-Grained Alignment Tasks}.}
    \label{fig:Compress_aesthetic}
\end{figure}
We further evaluate the efficacy of PAPA on a more fine-grained preference alignment task with pre-defined quantifiable objectives. We consider two diverse alignment tasks: a) \emph{Compressibility}, in which an image with a smaller size is regarded as better; b) \emph{Aesthetic Quality}, in which we use the LAION aesthetic score predictor~\cite{schuhmann2022laion} to automatically assign aesthetic ratings to images, enabling objective reward assignment based on visual quality without requiring human evaluation. For this analysis, we use Stable Diffusion v1.5 as the base model and compare its performance with D3PO~\cite{yang2024using}. Our experimental findings, as reported in ~\ref{fig:Compress_aesthetic}, indicate that PAPA rapidly adapts to the preference with comparatively fewer online feedback responses than D3PO, justifying its suitability for fine-grained preference alignment in interactive environments.

\section{Additional Visualizations of PAPA on Fine-Grained Alignment Task}
Figure~\ref{fig:add_vis_fg} shows additional examples of images generated by PAPA for the fine-grained alignment task.

\begin{figure}[!h] 
    \centering
    \includegraphics[width=0.9\textwidth]{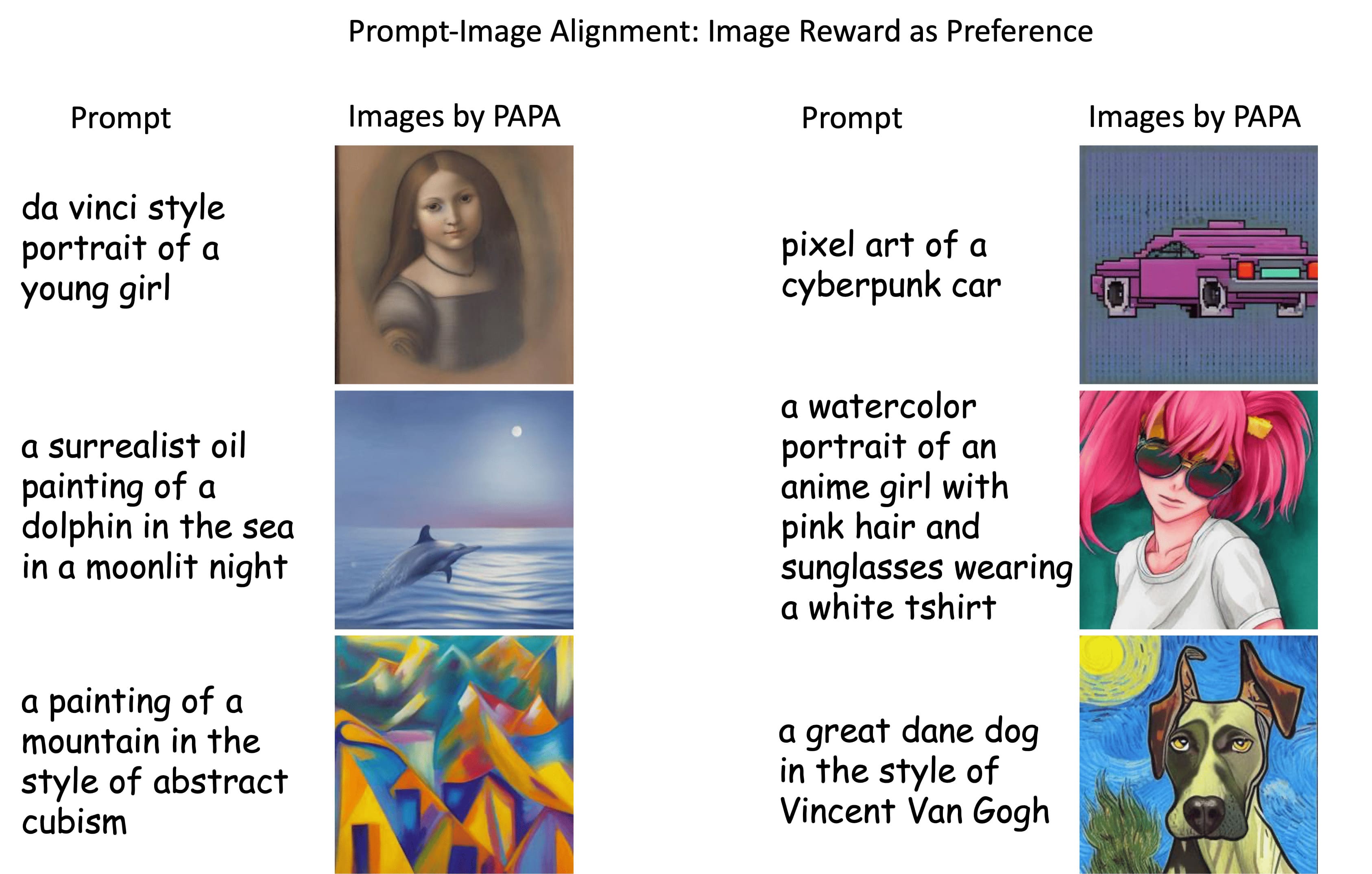}
    \caption{Additional visualizations of PAPA on the Fine-Grained Alignment Task.}
    \label{fig:add_vis_fg}
\end{figure}

\clearpage

\section{Qualitative Comparisons of Proposed PAPA with D3PO} \label{r:d3po_papa}
\begin{figure}[!h] 
    \centering
    \includegraphics[width=0.7\textwidth]{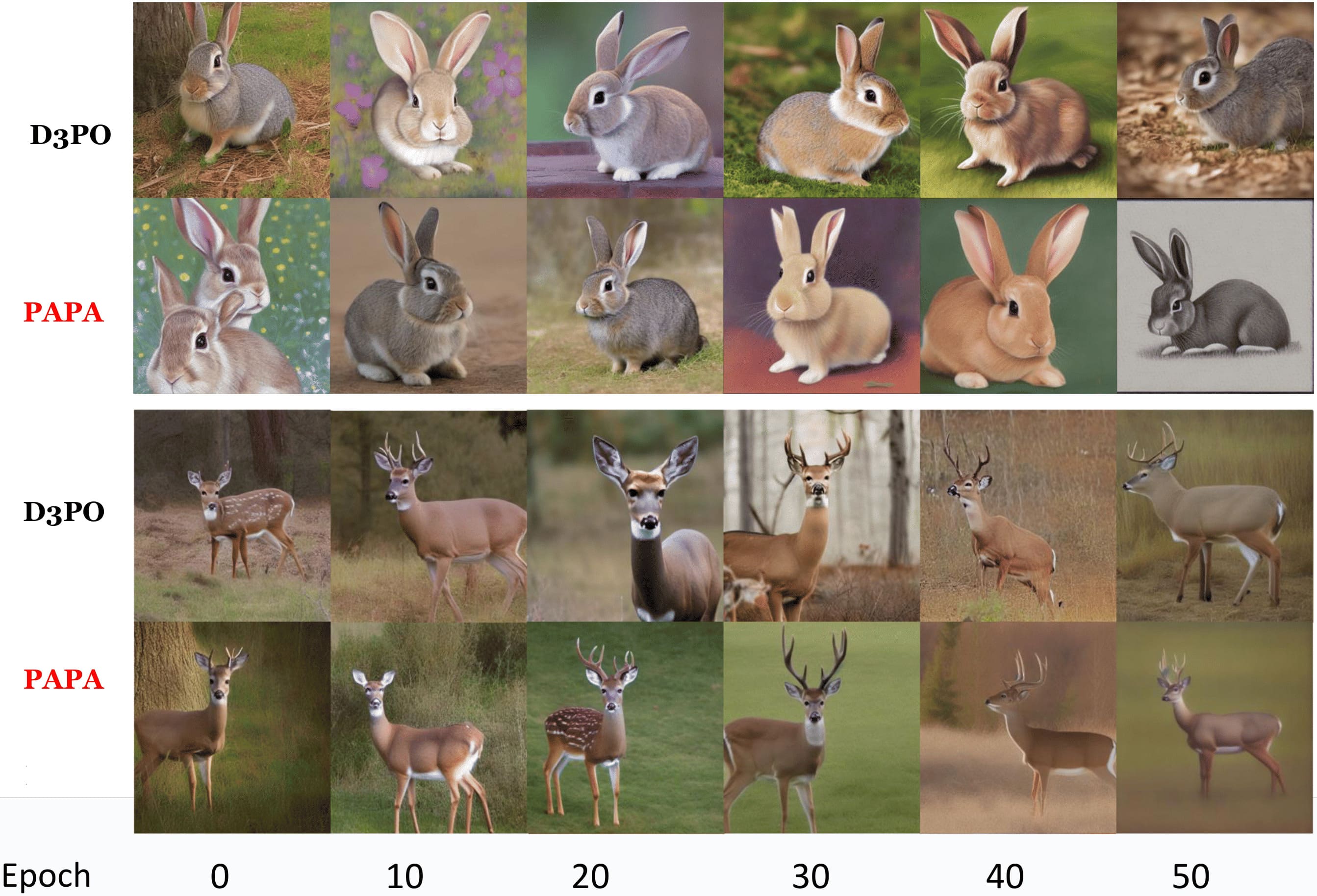}
    \caption{Additional visualizations of PAPA and D3PO in generating \textit{rabbit} and \textit{deer} images with compressibility as the objective.}
    \label{fig:compress_vis}
    \vspace{-10pt}
\end{figure}
\begin{figure}[!h] 
    \centering
    \includegraphics[width=0.7\textwidth]{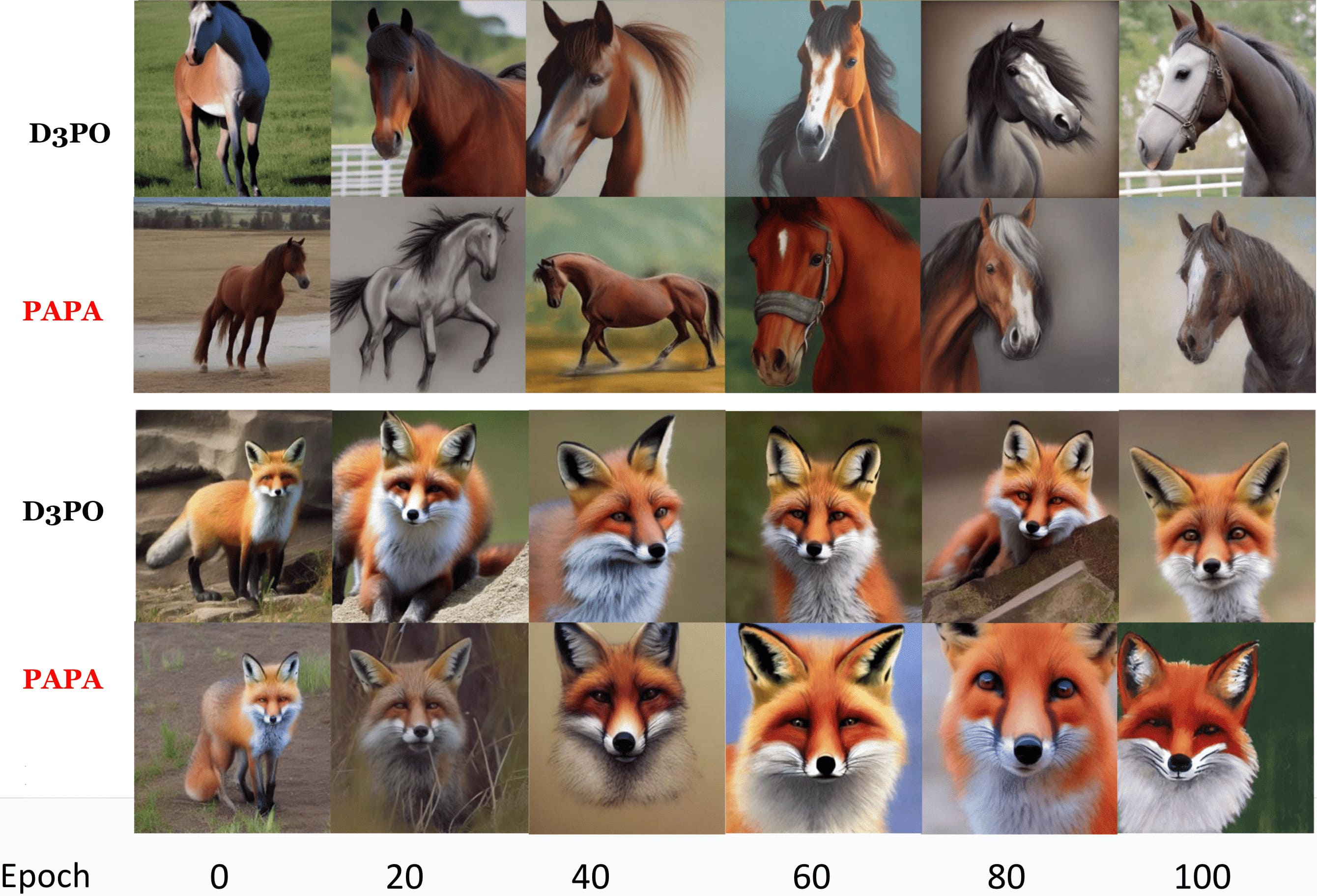}
    \caption{Additional visualizations of PAPA and D3PO in generating \textit{horse} and \textit{fox} images with the high aesthetic score objective.}
    \label{fig:aes_vis}
    \vspace{-10pt}
\end{figure}
Detailed comparison reveals that PAPA outperforms D3PO starting from early stages of interaction. Here, we present visualizations of generated images by D3PO and proposed PAPA across gradually increasing epochs. Figure~\ref{fig:compress_vis} presents samples generated with compressibility as a preference for prompts 'rabbit' and 'deer'. Figure~\ref{fig:aes_vis} presents generated examples for prompts 'horse' and 'fox' when aesthetic score is used as a preference.






\section{Insufficiency of Existing D3PO in Active Preference Alignment Tasks}
\begin{figure}[H] 
    \centering
    \includegraphics[width=0.90\textwidth]{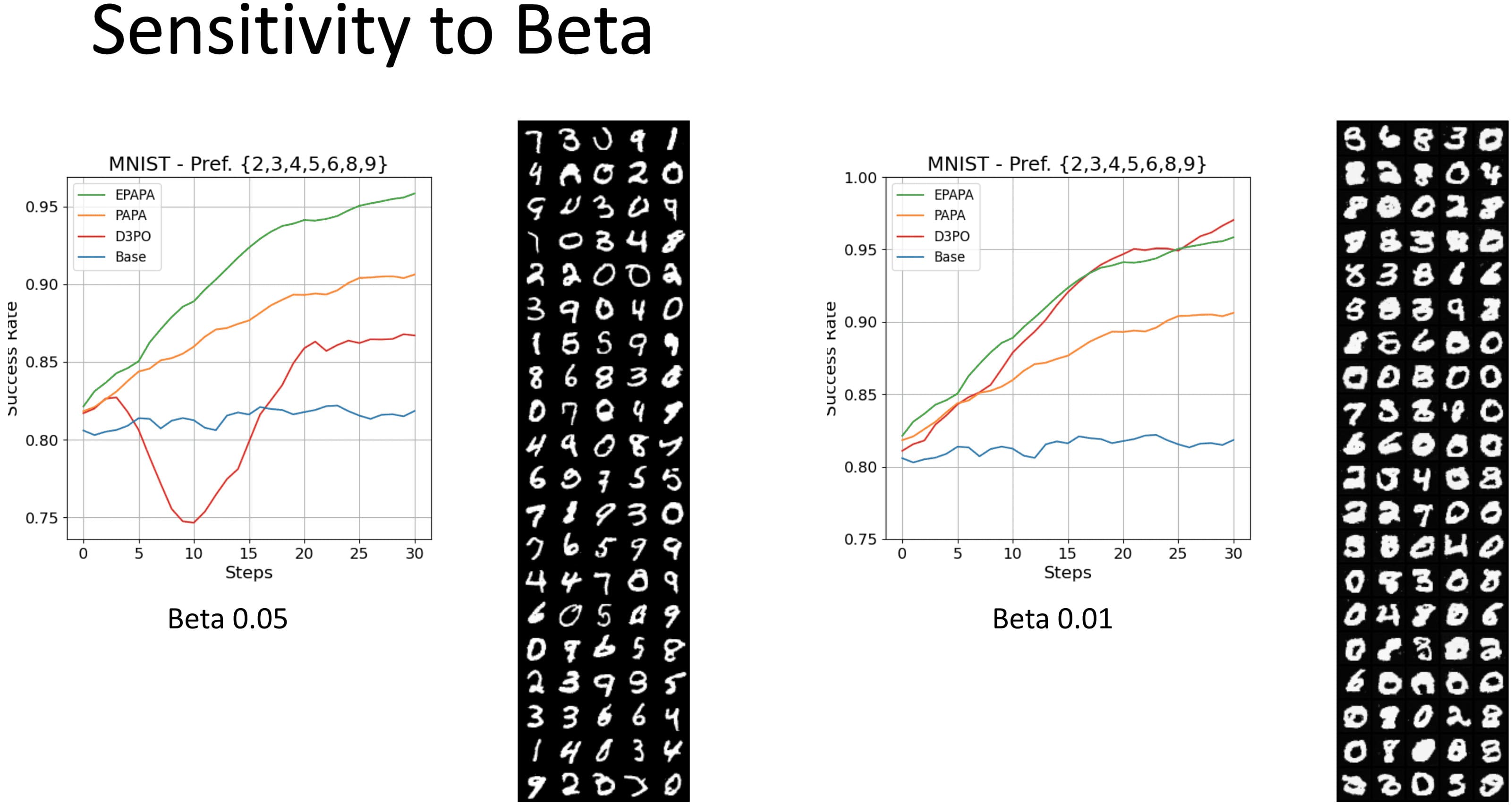}
    \caption{Instability of current preference alignment approach D3PO.}
    \label{fig:d3po_stability}
\end{figure}
Our extensive experimental analysis in the main paper shows D3PO's ineffectiveness in personalized active preference alignment. This section analyzes the underlying reasons. We initiate our analysis by conducting an experiment in which we systematically vary the value of $\beta$. This allows us to observe how adjustments to $\beta$—which governs the strength of the preference alignment loss (as proposed in~\cite{yang2024using}), as formally described below—affect the performance of preference alignment.
\vspace{10pt}

\begin{equation}
\mathcal{L}_i(\theta) = -\mathbb{E}_{(s_i, \sigma_w, \sigma_l)} \left[ \log \rho \left( \beta \log \frac{\pi_{\theta}(a_i^w | s_i^w)}{\pi_{\text{ref}}(a_i^w | s_i^w)} - \beta \log \frac{\pi_{\theta}(a_i^l | s_i^l)}{\pi_{\text{ref}}(a_i^l | s_i^l)} \right) \right]
\end{equation}

\vspace{10pt}

We present our empirical results in Figure~\ref{fig:d3po_stability}. Our observations indicate that increasing the value of $\beta$ leads to a decline in the success rate of D3PO, since a lower $\beta$ imposes a stronger emphasis on the preference alignment objective. Conversely, a higher $\beta$ relaxes this alignment, resulting in a notably reduced success rate, particularly when compared to PAPA and EPAPA. Intriguingly, we find an inverse relationship between the success rate and the quality of generated samples: high success rates correspond to lower sample quality, and vice versa. As the influence of preference alignment diminishes (i.e., as $\beta$ increases), the model remains closer to the base or reference model in parameter space, leading to improved sample quality but reduced alignment. Distinct from PAPA, D3PO lacks an independent parameter that controls sample quality without affecting the strength of the preference alignment objective. Consequently, unlike PAPA, D3PO cannot attain strong preference alignment without a substantial loss in the quality and diversity of generated samples.

\vspace{20pt}


\section{Results with Cifar-10}\label{r:cifar10}
In this section, we analyze our proposed approach using the Cifar-10 dataset. For pre-training the DDPM model on Cifar-10, we use the same architecture and hyperparameters specified in the open-source code available here\footnote{https://github.com/openai/improved-diffusion}.

\begin{table}[H]
    \centering
    \begin{minipage}{0.48\textwidth}
        \caption{Comparison of FID and IS scores on CIFAR-10 with preference set \{Pleane, Car, Truck, Ship\}.}
        \label{tb:cifar_fid}
        \smallskip
        \centering
        \begin{tabular}{p{1.18cm}p{2.10cm}p{2.10cm}}
            \toprule
            \multicolumn{3}{c}{\footnotesize{CIFAR-10: Preference set $s_1$ }} \\
            \midrule
            Method & \>\>FID$\downarrow$ & IS$\uparrow$   \\
            \midrule
            Base & 124.4 $\pm$ 6.0 & 6.6 $\pm$ 0.2  \\
            \hline 
            PAPA & 104.7 $\pm$ 5.0 & 5.3 $\pm$ 0.5   \\
            EPAPA & \textbf{95.9 $\pm$ 5.7} & \textbf{5.3 $\pm$ 0.5}  \\
            \bottomrule
            \bottomrule
        \end{tabular}
    \end{minipage}\hfill
    \begin{minipage}{0.44\textwidth}
        \centering
            \includegraphics[width=0.90\textwidth]{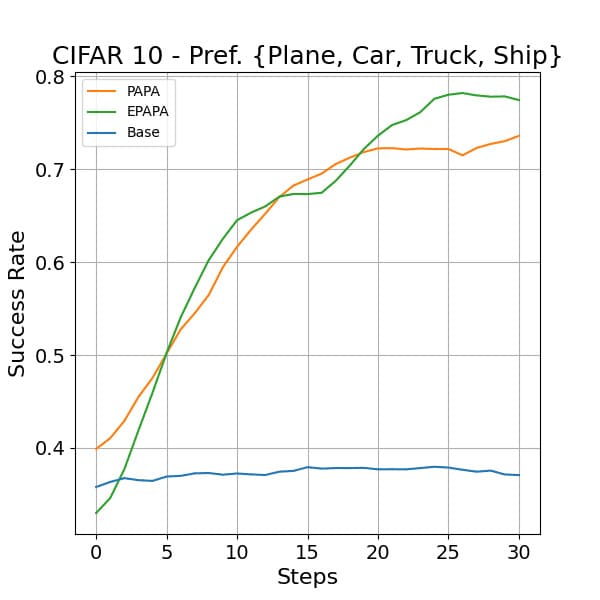} 
            \caption{Success rate (SR) comparisons with Cifar-10 across interaction steps.}
            \label{fig:sr_cifar1}
    \end{minipage}
\label{tb:cifar10}   
\end{table}

\begin{figure}[!h] 
    \centering
    \includegraphics[width=0.85\textwidth]{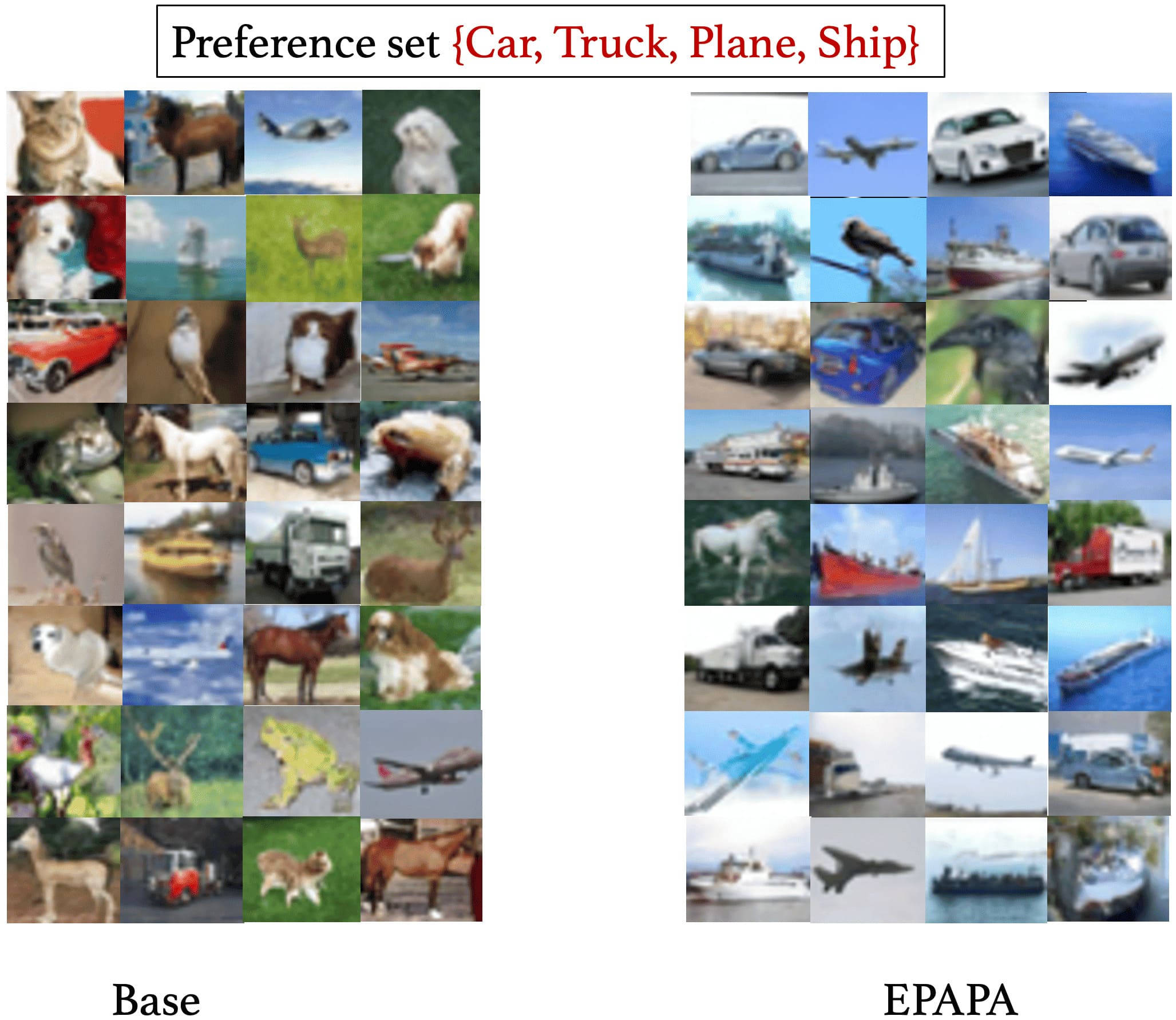}
    \caption{Additional Visualizations on the Efficacy of EPAPA on generating images aligning with the preference set.}
    \label{fig:cifar_p1}
\end{figure}

We evaluate performance using different user preference sets. Table~\ref{fig:sr_cifar1} presents the results based on the SR metric, with preference $s_1\in\{\text{Car, Truck,
Plane, Ship}\}$. The empirical results indicate that as fine-tuning advances and more user preference data is gathered, the SR significantly improves compared to the pre-trained model (denoted as \emph{Base}), highlighting the model's increasing capability to precisely align with user preferences. Moreover, we observe that EPAPA consistently outperforms PAPA, consistent with our findings across other datasets. This improvement underscores the effectiveness of EPAPA's sampling strategy, which combines the pre-trained diffusion model—adept at low-level denoising—with the fine-tuned model, which excels in high-level shaping. We assess sample quality and diversity using FID and IS metrics, with results summarized in Table~\ref{tb:cifar_fid}. The results demonstrate that our approach successfully preserves high quality, as evidenced by achieving an IS score comparable to the \emph{Base} model (trained on the preference set), while promoting diversity within the preferred set, indicated by a lower FID. Overall, these empirical findings highlight the impact of our proposed approaches in tackling the active preference alignment problem. 

\section{Qualitative Visualization from Cifar-10}
In this section, we provide a side-by-side visualization comparing samples produced by our proposed method and those generated by a baseline—specifically, a pretrained diffusion model trained on the CIFAR-10 dataset. For this qualitative analysis, we select a preference set $s \in \{ Car, Truck, Airplane, Ship \}$, as illustrated in Figure~\ref{fig:cifar_p1}. The results reveal that EPAPA produces samples closely aligned with user preferences, whereas the baseline diffusion model frequently generates samples outside the specified preferences. Additionally, samples from EPAPA maintain both quality and diversity within the chosen preference set. These further visualizations from CIFAR-10 underscore the effectiveness of our approach in a standard computer vision context.


\section{Effect of $\beta$}\label{app:beta}
This section examines the impact of $\beta$ on performance. To this end, we perform experiments with different values of $\beta$ and evaluate the results based on Success Rate, FID, and IS scores. For this analysis, we select $\{$ Boot, Sandal, Sneaker $\}$ as the preference set. The comparison of FID and IS scores is presented in Table~\ref{tb:gamaa}, while the Success Rate comparison is shown in Table~\ref{fig:gamaaSR}. Our empirical results reveal that selecting a high value for $\beta$ improves sample quality but significantly hinders preference alignment. On the other hand, a very low $\beta$ greatly enhances preference alignment, but at the expense of sample quality. Therefore, a mid-range value like 0.009 strikes the ideal balance, offering improved preference alignment while maintaining high sample quality—making it the most effective choice for practical applications.

\begin{table}[H]
    \centering
    \begin{minipage}{0.48\textwidth}
        \caption{Comparison of FID and IS scores for different values of $\mathbb{\beta}$.}
        \label{tb:gamaa}
        \smallskip
        \centering
        \begin{tabular}{p{1.18cm}p{2.10cm}p{2.10cm}}
            \toprule
            \multicolumn{3}{c}{Eval: Effect of $\beta$} \\
            \midrule
            $\beta$ & FID$\downarrow$ & IS$\uparrow$   \\
            \midrule
            0.0 & 39.93$\pm$18.3 & 2.8$\pm$0.38   \\
            \hline 
            0.001 & 47.98$\pm$15.53 & 2.3$\pm$0.40   \\
            0.005 & 238.90$\pm$74.02 & 2.4$\pm$0.43  \\
            \textbf{0.009} & \textbf{36.72}$\pm$\textbf{16.9} & \textbf{2.9}$\pm$\textbf{0.34}  \\
            0.05 & 401.70$\pm$115.58 & 1.7$\pm$0.84  \\
            \bottomrule
        \end{tabular}
    \end{minipage}\hfill
    \begin{minipage}{0.44\textwidth}
            \includegraphics[width=0.90\textwidth]{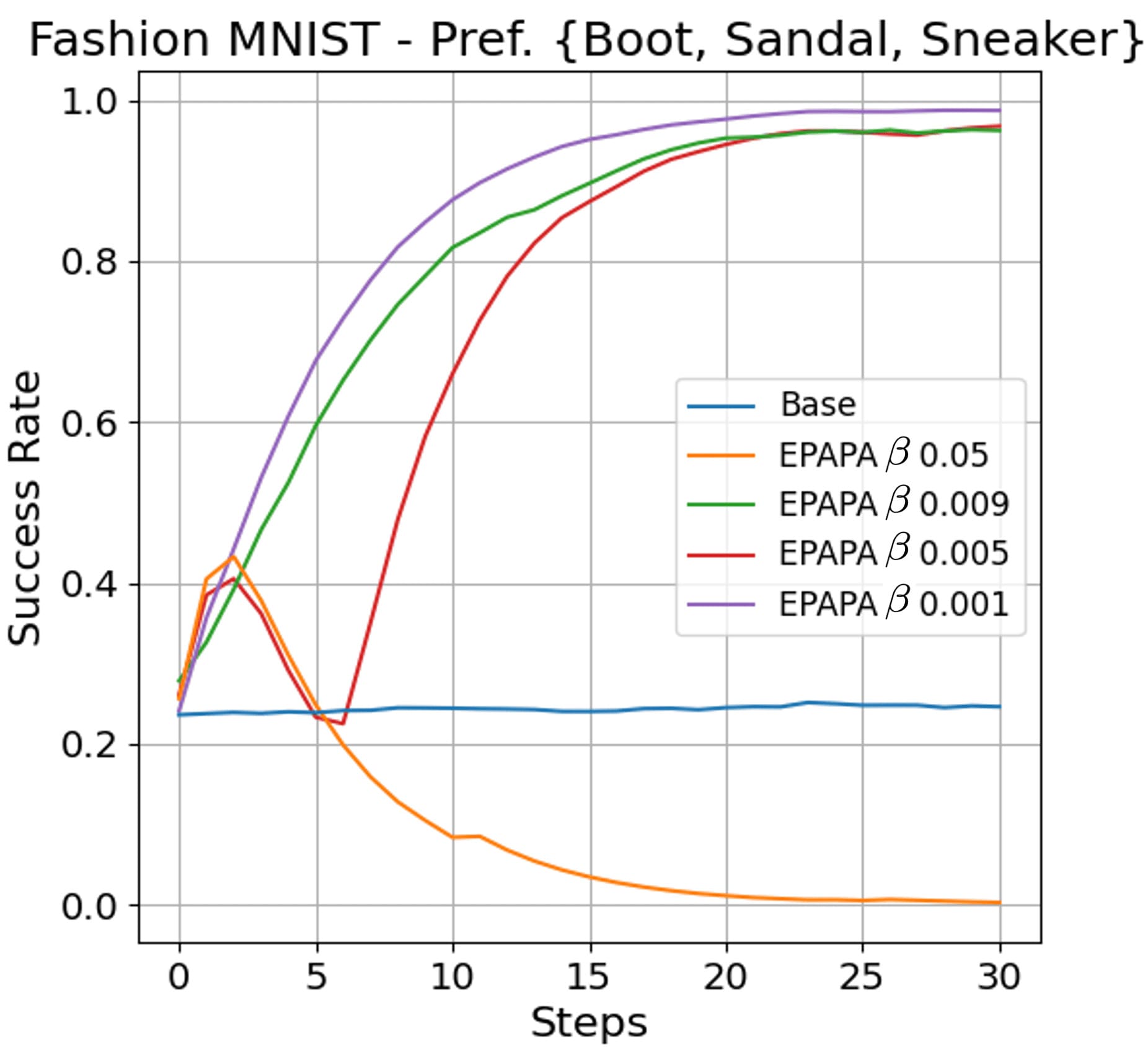} 
            \caption{Analyzing the Effect of $\beta$.}
            \label{fig:gamaaSR}
    \end{minipage}
\end{table}

\clearpage


\section{Impact of ${L}^{QPDE}$ on Success Rate}
\label{app:srqpde}

\begin{minipage}{1.0\textwidth}
    \begin{minipage}{0.5\textwidth}
    In the main paper (see Section 6), we examine the effect of $\mathcal{L}^{\mathrm{QPDE}}$ on the quality and diversity of the generated samples. Here, we focus on its impact on the success rate. For this analysis, we select $\{$ Boot, Sandal, Sneaker$\}$ as the preference set and present the results in Figure~\ref{fig:sr_n}. We observe that the model fine-tuned with and without $\mathcal{L}^{QPDE}$ achieves similar success rates. This suggests that while $\mathcal{L}^{\mathrm{QPDE}}$ does not directly influence the success rate, it significantly enhances the diversity and quality of the generated samples.
    \end{minipage}
    \hfill
    \begin{minipage}{0.45\textwidth}
        \begin{figure}[H] 
             \centering
             \includegraphics[width=0.50\textwidth]{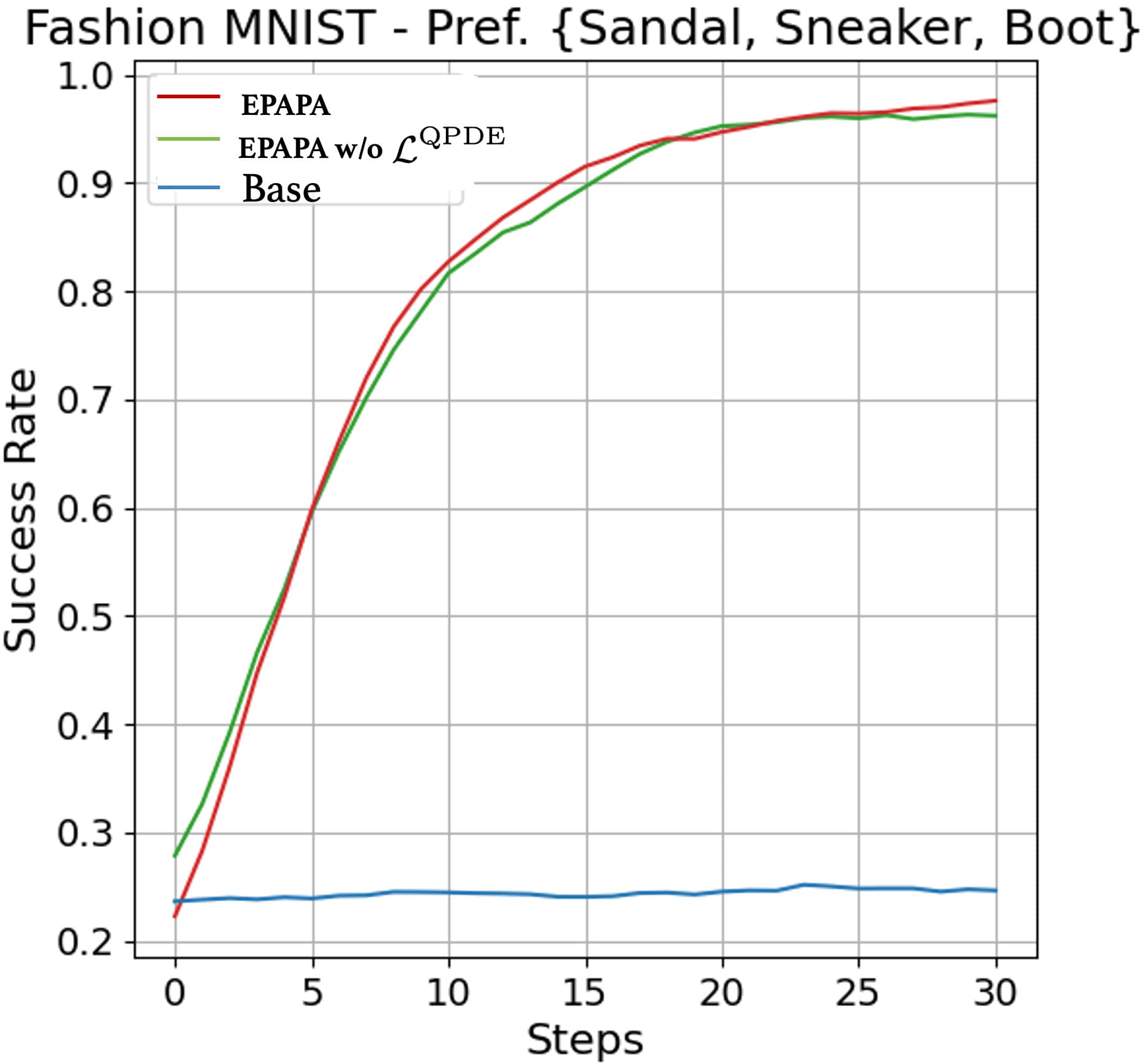}
             \caption{\small{Analyzing $\mathcal{L}^{QPDE}$.}}
             \label{fig:sr_n}
     \end{figure}
    \end{minipage}

\end{minipage}



\section{Qualitative Analysis on the Importance of ${L}^{{QPDE}}$}\label{app:qqpde}
In the main paper (see Section~6), we study the role of $\mathcal{L}^{\mathrm{QPDE}}$ in improving sample quality and diversity. In this section, we provide additional qualitative results that further highlight the impact of $\mathcal{L}^{\mathrm{QPDE}}$ on both sample quality and diversity. In Figure~\ref{fig:vis_qpde}, we compare the samples generated by models fine-tuned with and without $\mathcal{L}^{\mathrm{QPDE}}$. For this comparison, we choose $\{$ Boot, Sandal, Sneaker$\}$ as the preference set. We observe a noticeable improvement in both the quality and diversity of the generated samples when $\mathcal{L}^{\mathrm{QPDE}}$ is included in the objective.

\begin{figure}[H] 
    \centering
    \includegraphics[width=0.4\textwidth]{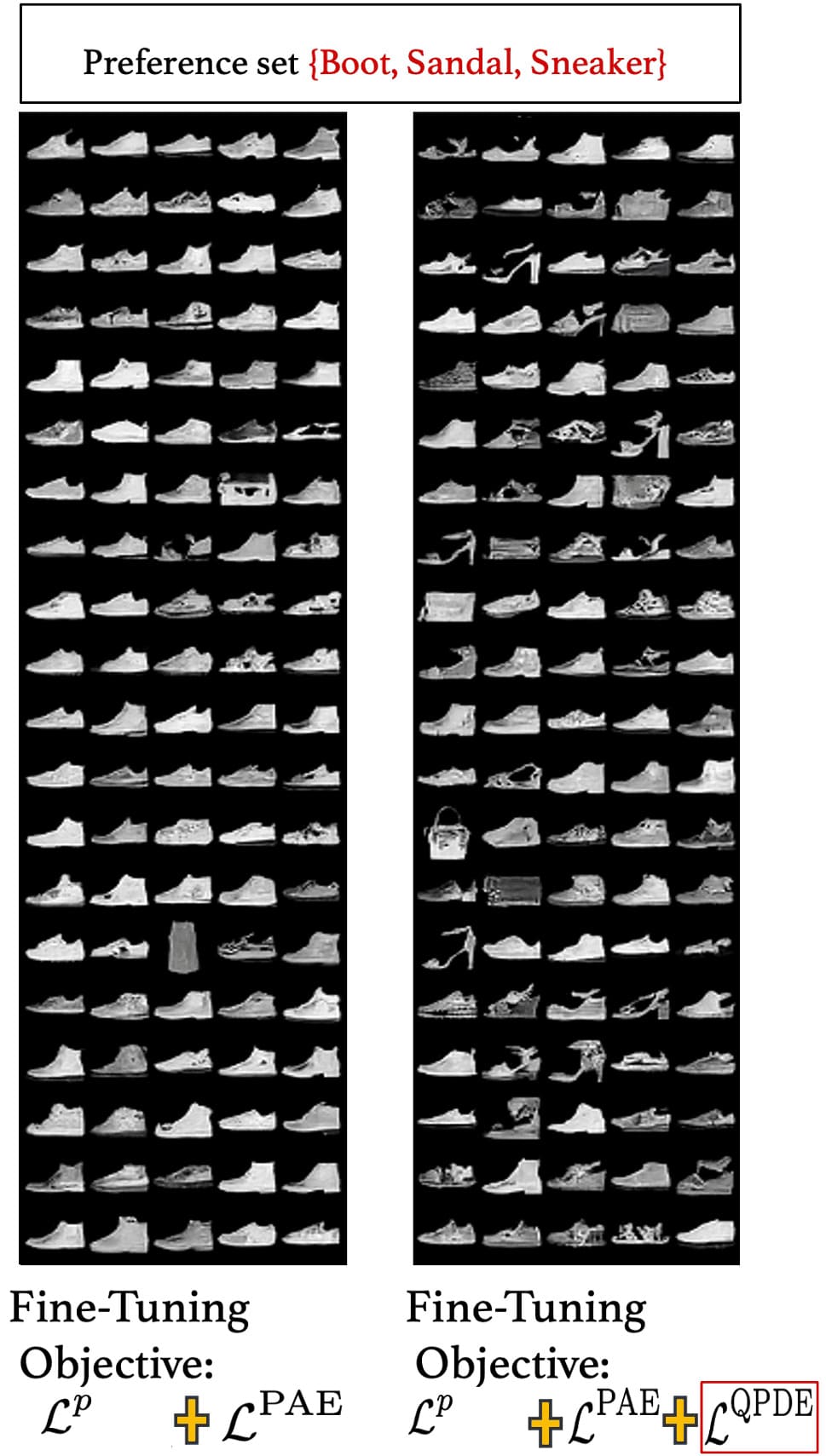}
    \caption{Visualizing the importance of $\mathcal{L}^{\mathrm{QPDE}}$ on Sample Quality $\&$ Diversity.}
    \label{fig:vis_qpde}
\end{figure}



\section{Visualizations of Insufficiency of \emph{${L}^p$}}

In the main paper (see Section~5), we argue that the $\mathcal{L}^p$ objective alone is insufficient for effectively addressing active preference alignment problems. Through a series of ablation studies, we emphasize the critical role of other components within the PAPA and EPAPA frameworks (see Section~6). To illustrate the limitations of $\mathcal{L}^p$, we provide a visualization showing its inability to generate diverse samples within the preference set. The results, shown in Figure~\ref{fig:papa_ins}, use $\{$ Boot, Sneaker, Sandal$\}$ as the preference set. We observe that the model trained with the $\mathcal{L}^p$ objective predominantly generates samples from a single preferred class (i.e., Boot), underscoring its failure to produce a diverse range of samples, which is also supported by Theorem 1. 

\begin{figure}[H] 
    \centering
    \includegraphics[width=0.5\textwidth]{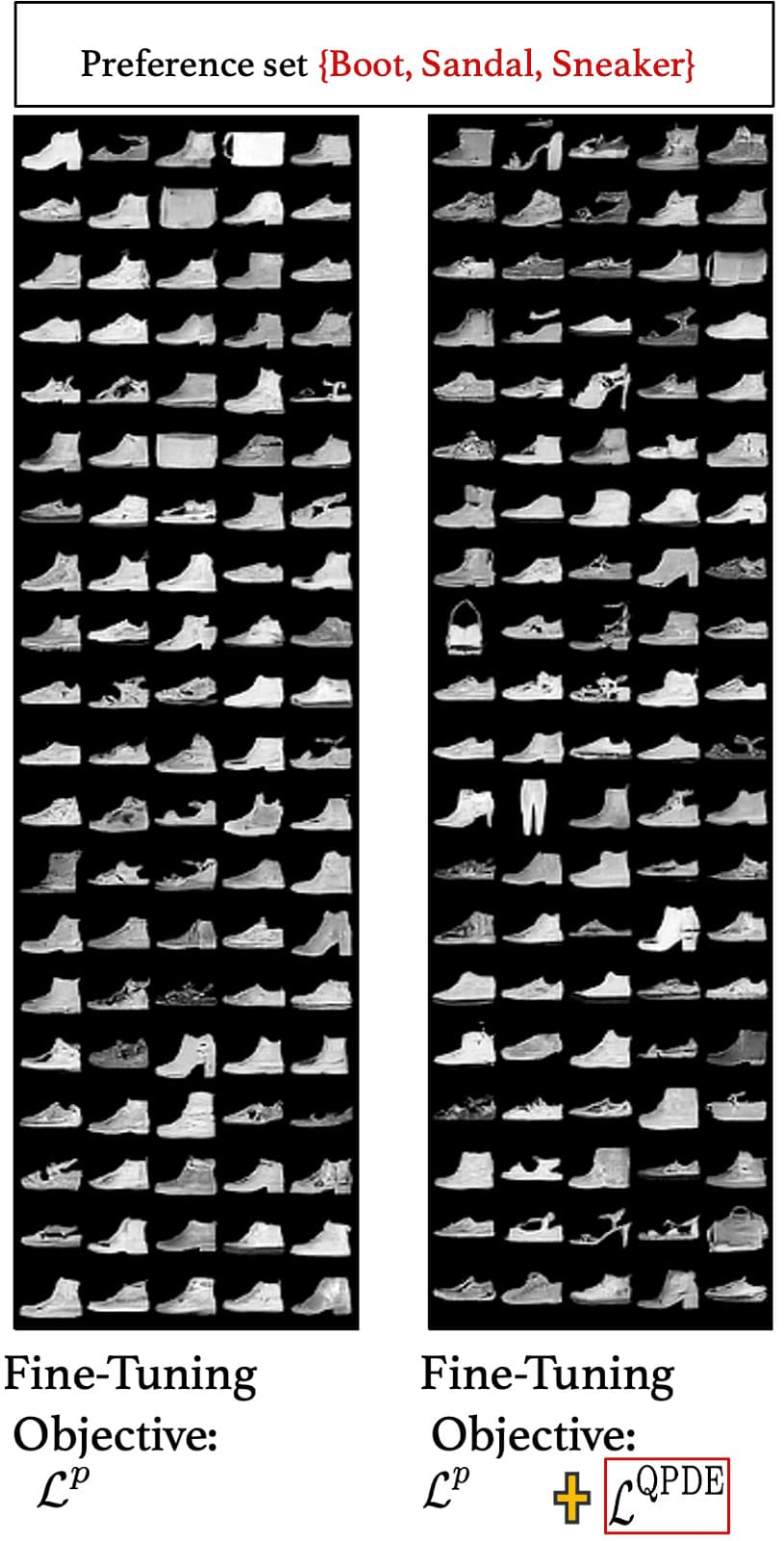}
    \caption{\small{Insufficiency of $\mathcal{L}^p$.}
}
    \label{fig:papa_ins}
\end{figure}

\section{Insufficiency of \emph{${L}^{np}$}}
One might wonder whether $\mathcal{L}^{np}$ alone is sufficient to address the active preference alignment task. However, it turns out that $\mathcal{L}^{np}$ by itself is inadequate for solving this problem. It faces challenges similar to those encountered when the model is fine-tuned using only $\mathcal{L}^{p}$. As depicted in Figure~\ref{fig:collapse}, we observe that when fine-tuned solely with $\mathcal{L}^{np}$, the model gradually forgets everything and eventually generates entirely black images. This suggests that the objective essentially leads to forgetting the non-preferred set derived from the user's feedback. Thus, relying exclusively on $\mathcal{L}^{np}$ during fine-tuning causes the model to unlearn prior knowledge. This observation underscores the critical role of $\mathcal{L}^{p}$ in maintaining a balance between relearning and forgetting (\textbf{"memory consolidation"}), which is essential for effective active preference alignment.

\begin{figure}[!h] 
    \centering
    \includegraphics[width=0.45\textwidth]{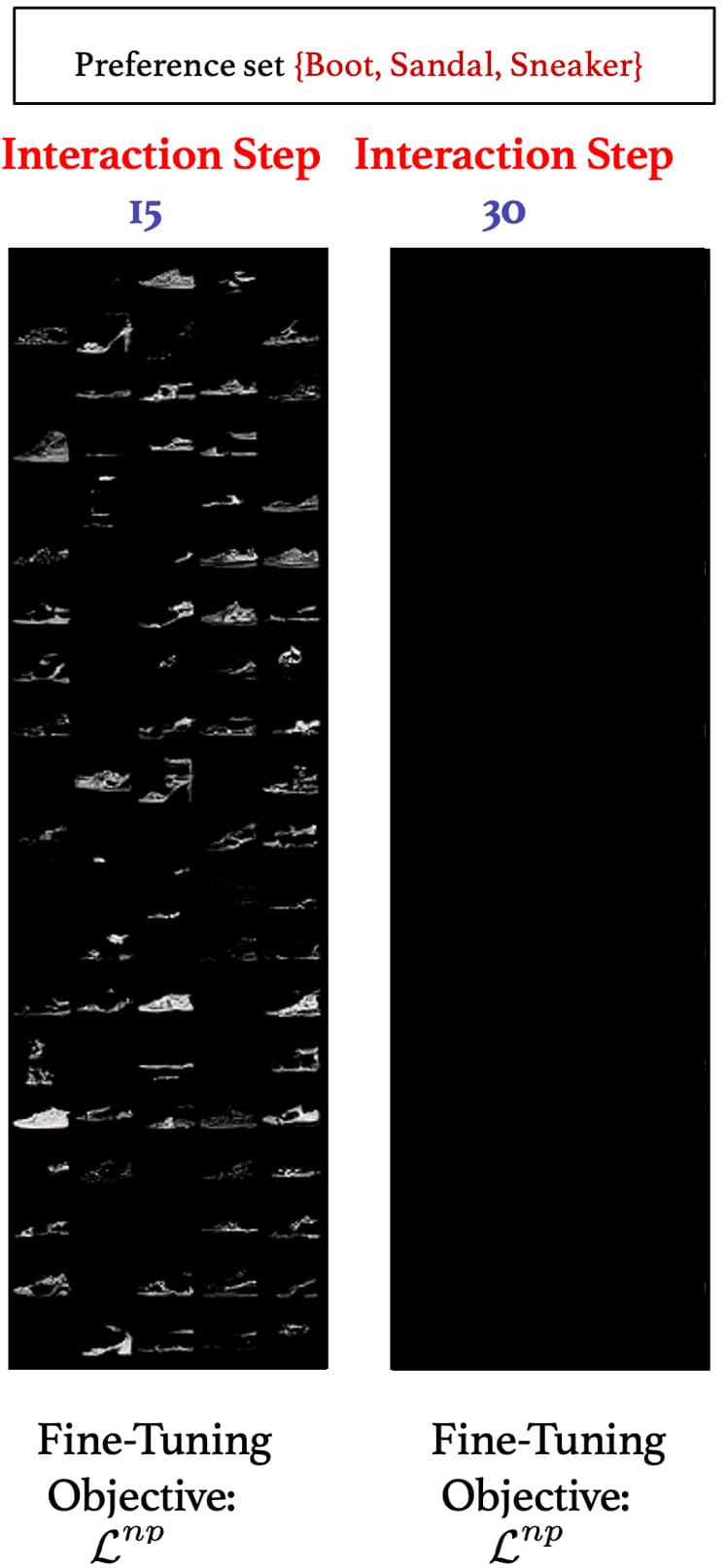}
    \caption{\small{Insufficiency of $\mathcal{L}^{np}$.}
}
    \label{fig:collapse}
\end{figure}


\section{Generalization of PAPA under Non-Binary Feedback}
\begin{figure}[!h] 
    \centering
    \includegraphics[width=0.70\textwidth]{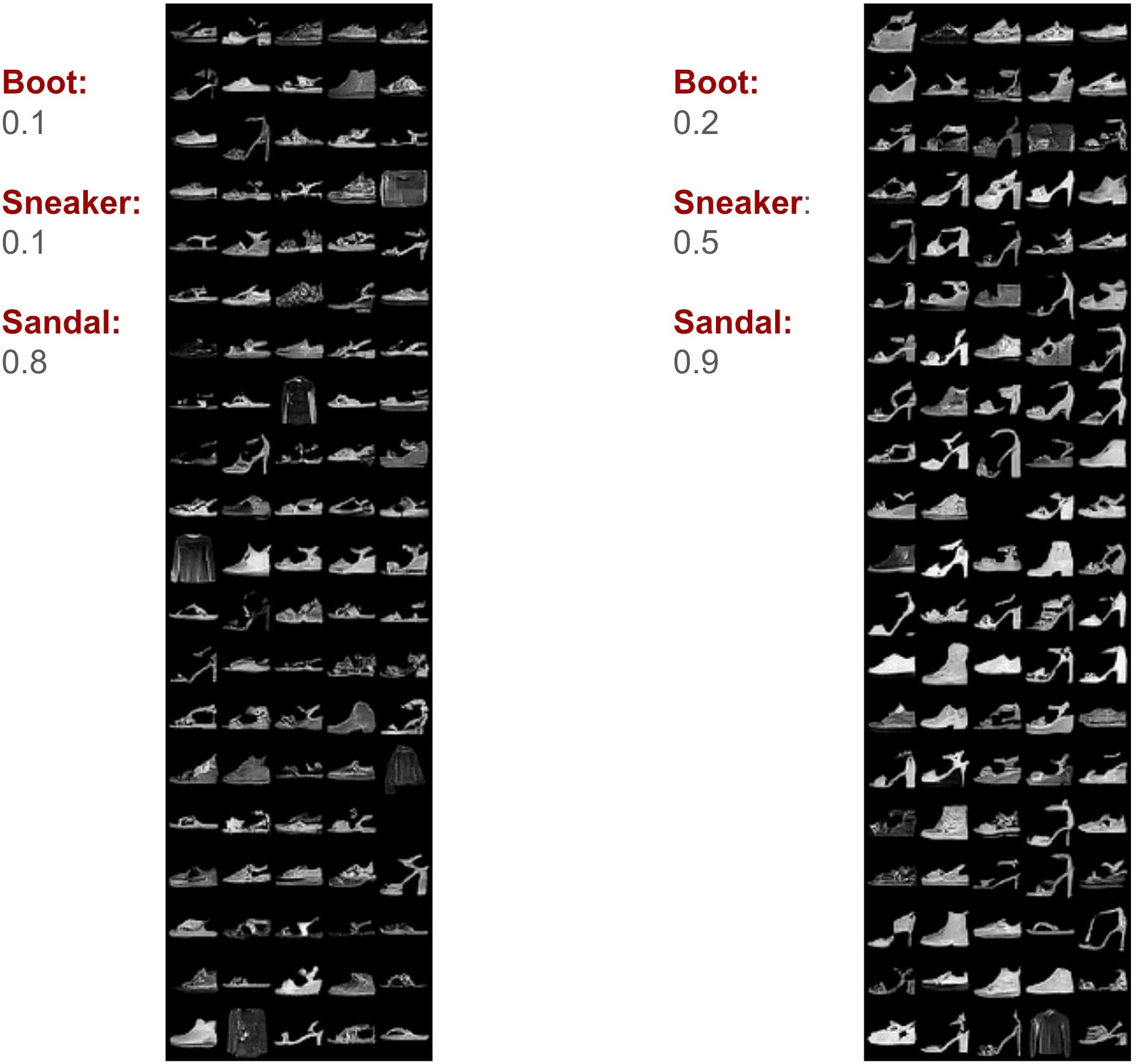}
    \caption{Visualizations of generated samples with EPAPA under non-binary feedback settings.}
    \label{fig:papa_nb}
\end{figure}
While our current formulation treats all preferred samples equally, the Preference Aligner objective (Eqn. 5) is flexible and can be naturally extended to handle richer feedback signals, such as ratings or auxiliary inputs. In particular, normalized rating scores can be incorporated as weighting coefficients in the $\mathcal{L}_{PAPA}$ term, allowing us to prioritize certain preferences more strongly than others. To explore this idea, we conducted a controlled experiment where one preferred class (e.g., Sandal) was assigned a higher normalized preference score of 0.9, compared to other preference classes (Boot with a preference score of 0.2, and Sneaker with 0.5). As feedback accumulated, the model adapted accordingly—generating noticeably more sandal samples, in alignment with the stronger preference signal. For a visualization of the generated samples, please see Figure~\ref{fig:papa_nb} (right). In Figure~\ref{fig:papa_nb} (left),  we depict the visualization of generated samples with a different normalized preference score across the same preference set. These additional results not only demonstrate the adaptability of PAPA to more expressive feedback modalities but also highlight its potential to serve as a broader framework for personalized generative modeling. We believe incorporating fine-grained feedback is an exciting avenue for future work, and PAPA is well-positioned to support it.

\section{Qualitative Comparisons with different $K$}\label{app:k}
\begin{figure}[!h] 
    \centering
    \includegraphics[width=0.60\textwidth]{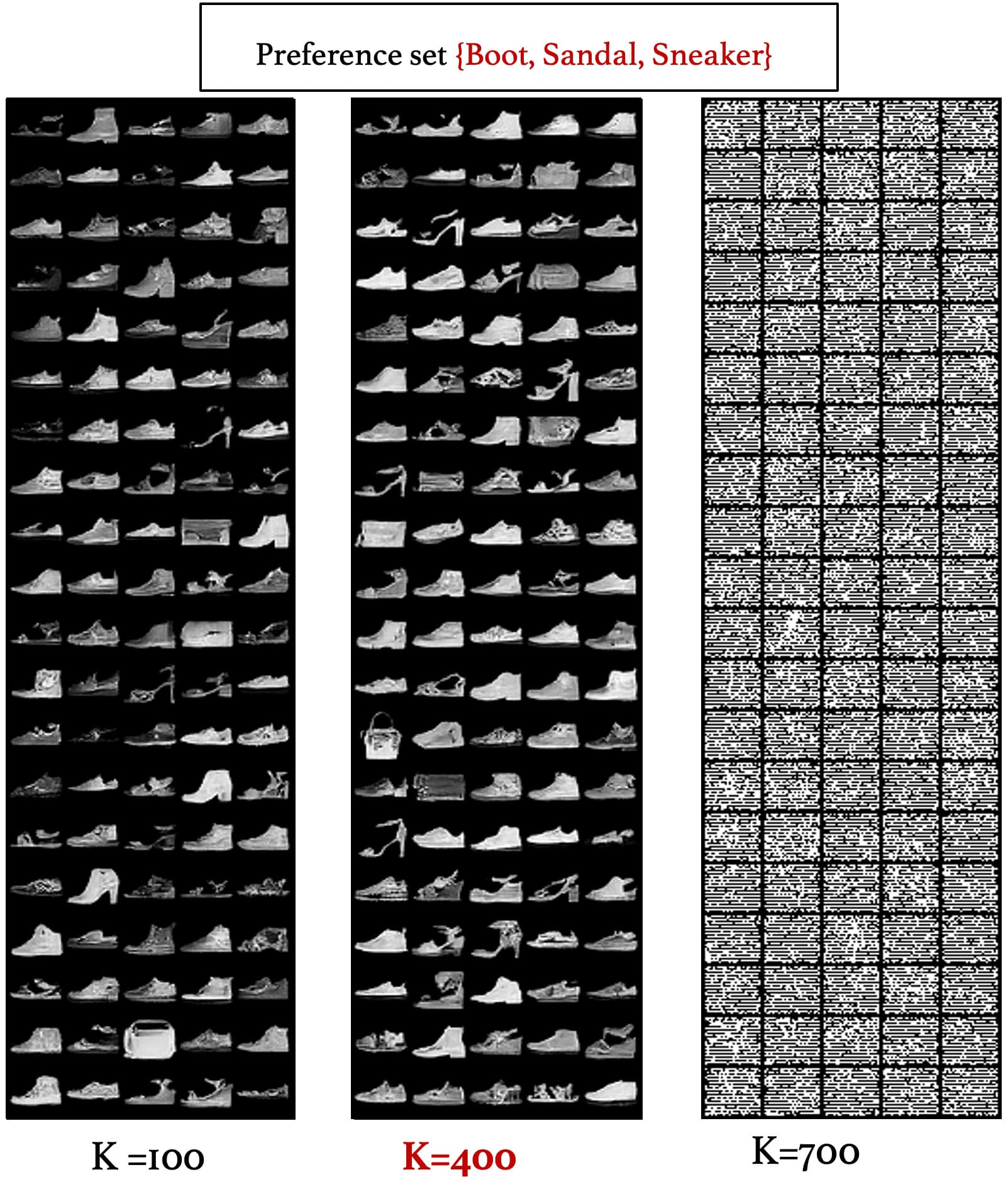}
    \caption{Visualizations with Different Values of $K$.}
    \label{fig:ablk}
\end{figure}
In the main paper (see Section~6), we provide a quantitative analysis of the performance for different values of $K$. Here, we offer a qualitative comparison, with the results shown in Figure~\ref{fig:ablk}. We choose $\{$ boot, Sandal, Sneaker $\}$ as the preference set for this analysis. As shown in Figure~\ref{fig:ablk}, the best results are achieved with  $K=400$, while larger values of $K$ lead to a decline in sample quality. These findings align with our hypothesis, reinforcing the importance of using the fine-tuned model as a denoiser for higher noise levels, and the pre-trained model for lower noise levels, to enhance performance.


\section{Additional Results on the Effects of $K$} 
As analyzed in the main paper (see Section~6.3), our experiments with Fashion-MNIST show that extreme values of $K$ are suboptimal: large $K$ values lead to poor denoising by pre-trained diffusion models, while very small $K$ values result in weaker alignment due to insufficient guidance. This motivates choosing $K$ in the mid-range of the reverse diffusion steps, where the model achieves a balance between generative quality and preference alignment. While the optimal $K$ may vary slightly across domains, we find that values near the midpoint of the reverse diffusion step generalize well in practice. To further validate this, we conducted additional experiments on MNIST and report our results in the following Table~\ref{tb:k_new}. We observe a consistent trend with Fashion-MNIST—optimal performance is achieved when the value of $K$ lies near the midpoint of the diffusion trajectory. We present the result in the following table, and our empirical outcomes are consistent with the result we observed with the other dataset (see Table~2 in the main paper).

\begin{table}[ht]
    \caption{Effect of $K$.}
    \label{tb:k_new}
    \centering
    \begin{tabular}{p{0.70cm} p{1.90cm} p{1.80cm}}
        \toprule
        $K$ & \>\>FID$\downarrow$ & \>\>IS$\uparrow$ \\
        \midrule
        100 & 18.87$\pm$5.5 & 2.0$\pm$0.04 \\
        400 & \textbf{16.25$\pm$3.03} & \textbf{2.0$\pm$0.05} \\
        700 & 30.95$\pm$4.02 & 2.0$\pm$0.03 \\
        \bottomrule
    \end{tabular}
\end{table}



\section{Stability Analysis of Our Proposed Approach}
To assess the stability of our proposed method, we conducted experiments across 3 independent random trials. The following Figure~\ref{fig:papa_stability} presents a success rate plot showing the mean and standard deviation calculated from these trials. The solid lines in the plot denote the mean, while the shaded regions denote the standard deviation. These results further reinforce the efficacy and stability of our proposed method. 

\begin{figure}[!h] 
    \centering
    \includegraphics[width=0.90\textwidth]{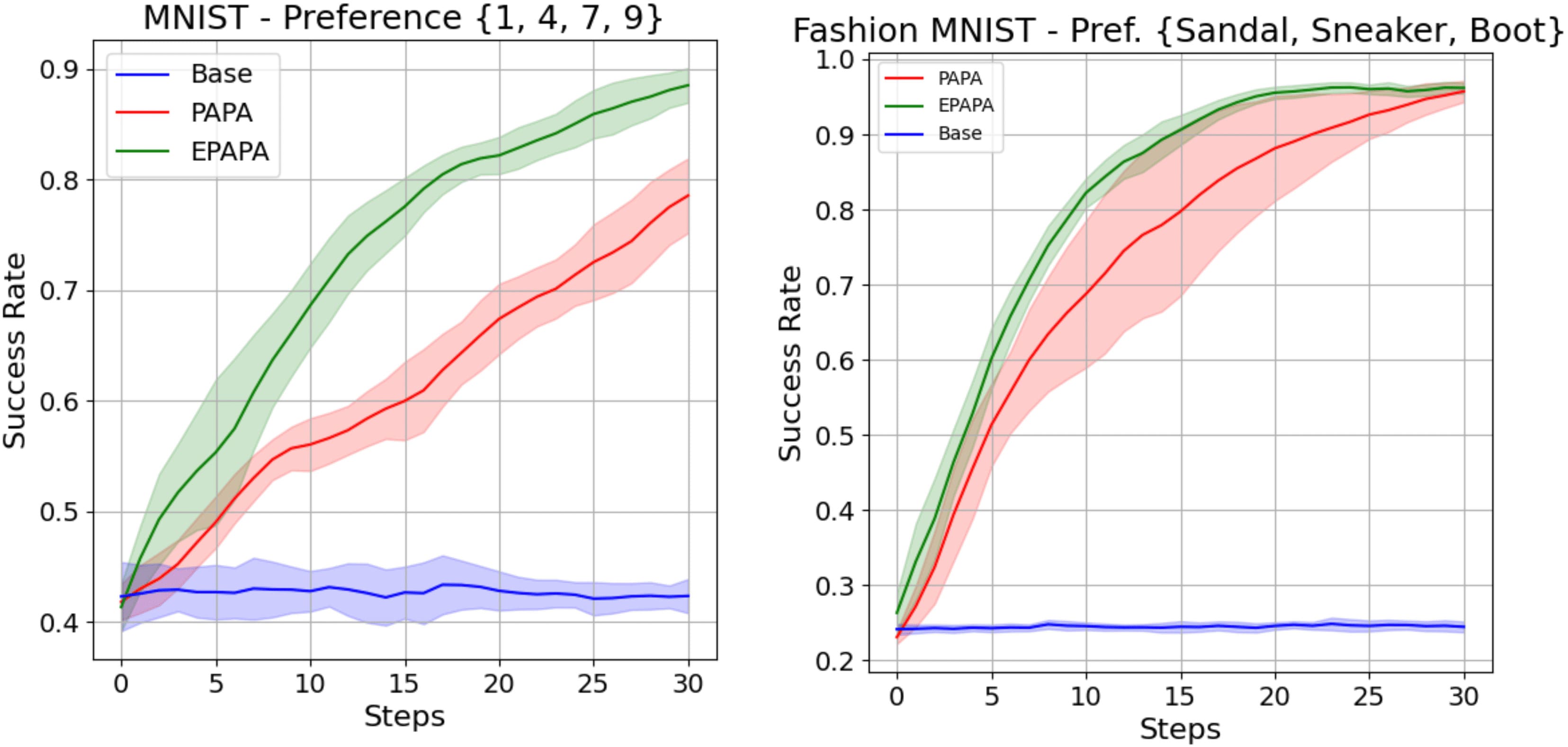}
    \caption{Visualizations of \textit{success rate} across different choices of preference sets using three independent trials.}
    \label{fig:papa_stability}
\end{figure}

\vspace{20pt}

\section{Details of the Predictive Model Used as the Proxy for User Feedback}\label{sec:pred}
 
In our work, instead of relying on user-provided feedback for preferences, we use a pre-trained classifier model to predict the class of the generated image and provide binary feedback based on whether the predicted class belongs to the user's preferred set. If the predicted class matches, feedback is positive; otherwise, it is negative. The classifier utilizes a custom neural network architecture with two 2D convolutional layers ([32, 64] filters), followed by two fully connected layers. Max-pooling is applied between the last convolutional layer and the first fully connected layer. ReLU activations are used after each layer, and dropout layers (with rates of 0.25 and 0.5) help prevent overfitting. We use cross-entropy loss to optimize the parameters of the classifier and leverage Adadelta~\cite{zeiler2012adadelta} as the optimizer and StepLR as the learning rate scheduler. Models trained on the MNIST and Fashion-MNIST datasets achieve accuracies of 98\%  and 94\%, respectively, making them reliable substitutes for human feedback in simulating user interaction. Nonetheless, human preferences are often nuanced and intricate, frequently surpassing what a typical classifier can capture. In our work, we adopt a pre-trained classifier as a practical and controlled proxy for user feedback, particularly to facilitate reproducible experiments in the absence of large-scale human-in-the-loop data. We see this as a foundational step toward more realistic preference modeling. Notably, our framework is modular by design and can seamlessly integrate richer forms of feedback—such as human responses or learned preference models—as they become available.

\clearpage

\section{Additional Visualizations of Preference Alignment with Diverse Preference Set}
Here, we present additional comparative visualizations of EPAPA, Base, and D3PO across diverse preference sets. We present the visualizations in Figure~\ref{fig:vis_add1}, again reinforcing the effectiveness of our proposed approach. 
\begin{figure}[H] 
    \centering
    \includegraphics[width=0.70\textwidth]{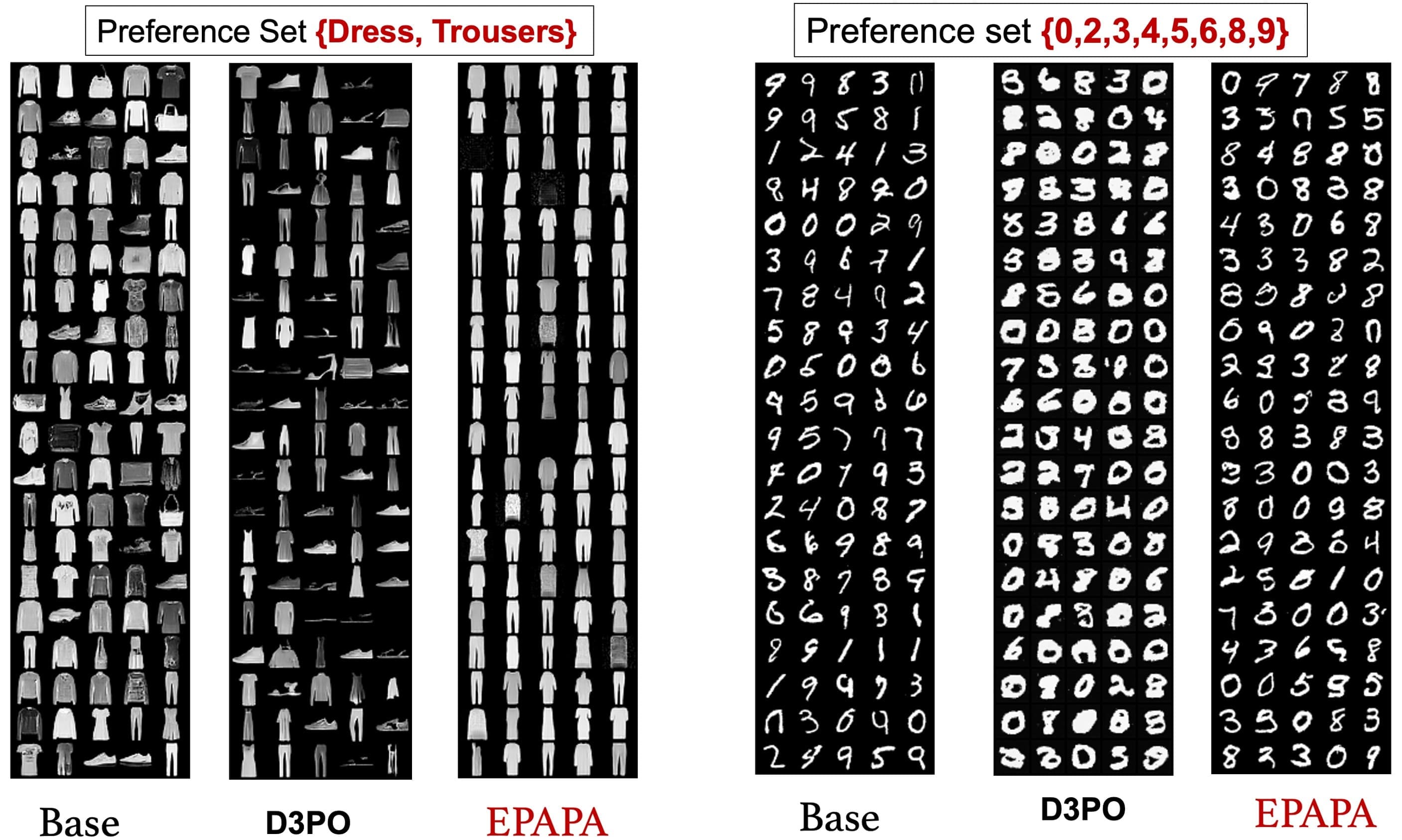}
    \caption{Additional visualizations of generated images using proposed EPAPA compared to base model and existing approach D3PO.}
    \label{fig:vis_add1}
\end{figure}

\section{Comparative Visualizations of PAPA and EPAPA}\label{app:pe}
In the main paper, we offer a quantitative analysis of PAPA and EPAPA across various experimental settings (see Section 6.1). Here, we provide a qualitative comparison between PAPA and EPAPA through visualizations, which are presented in Figure~\ref{fig:papa_epapa}. These qualitative visualizations indicate that the quality and diversity of the samples generated by EPAPA are notably superior to those produced by PAPA. The enhanced performance of EPAPA can be primarily attributed to its effective use of the pre-trained model during the reverse diffusion process at low noise levels.

\begin{figure}[!h] 
    \centering
    \includegraphics[width=0.7\textwidth]{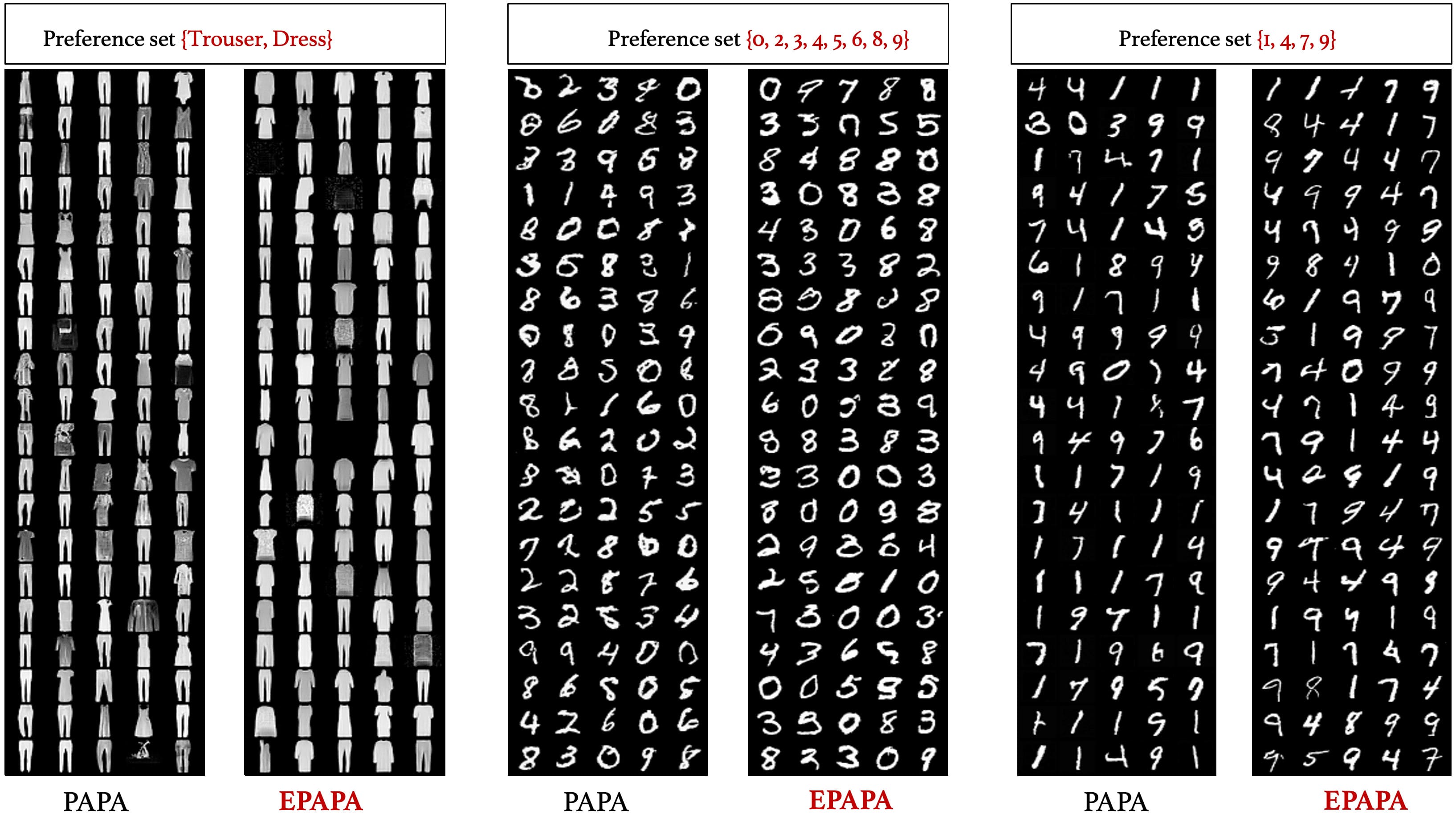}
    \caption{\small{Comparative visualizations of PAPA and EPAPA.}
}
    \label{fig:papa_epapa}
\end{figure}

\clearpage

\section{Additional Examples of Generated Samples Using a Different Preference Set from CIFAR-10}

In this section, we present visualizations with $s_2 \in \{ Horse, Deer, Cat, Dog\}$ as the preference set, as depicted in Figure~\ref{fig:cifar_p2}. These additional visualizations further reinforce the efficacy of EPAPA in generating samples aligning with diverse preference sets. 

\begin{figure}[!h] 
    \centering
    \includegraphics[width=0.60\textwidth]{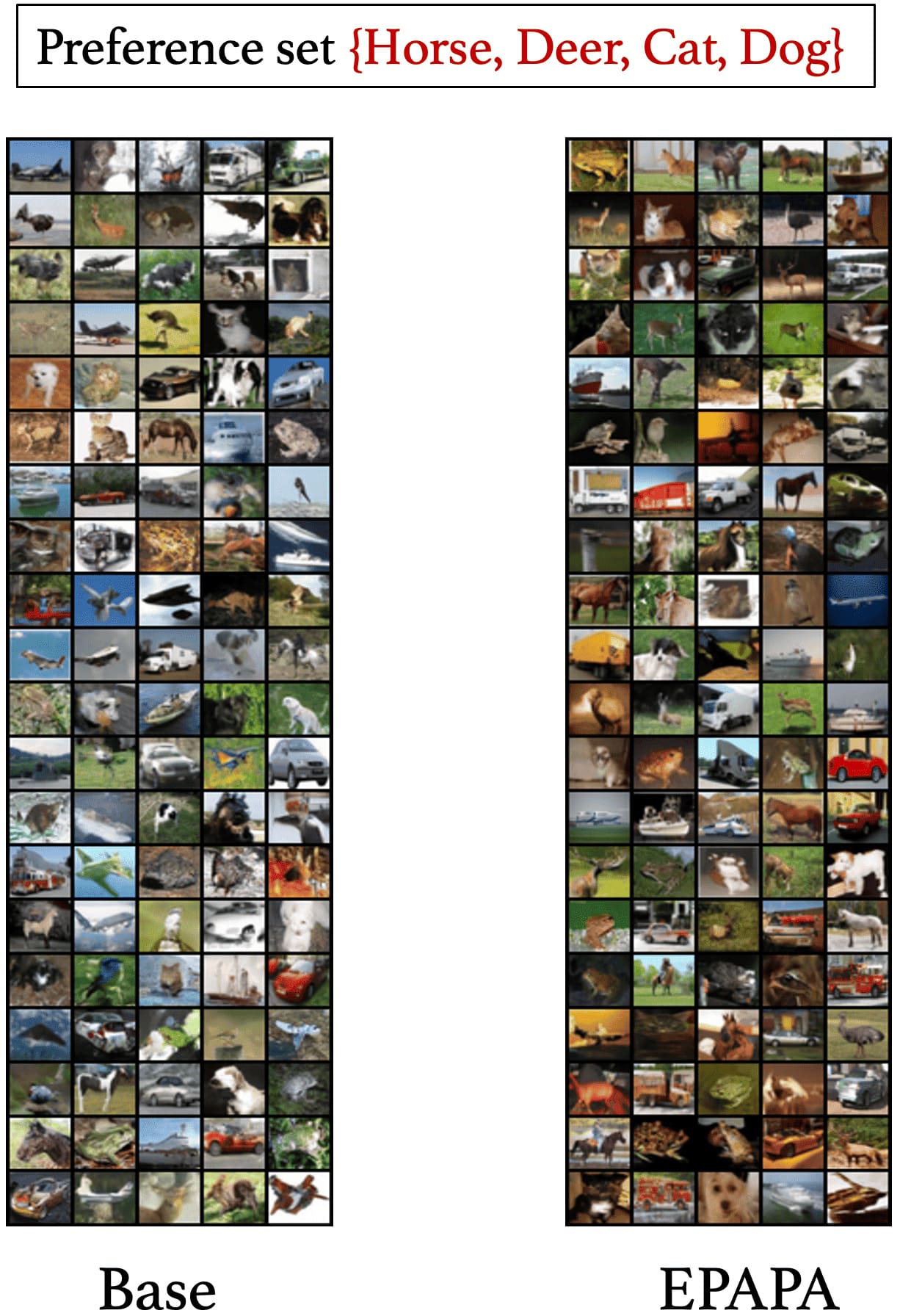}
    \caption{Additional visualizations on the efficacy of EPAPA on generating images aligning with the preference set.}
    \label{fig:cifar_p2}
\end{figure}

\clearpage

\section{Details of Computing Resource and Hyperparameters}\label{app:hyp}

Our diffusion model employs a UNet~\cite{ronneberger2015u} architecture as its backbone. For MNIST experiments, the network comprises six ResNet~\cite{he2016deep} layers and four attention heads, providing a robust framework for image generation tasks. Our implementation is configurable, allowing seamless adaptation for both grayscale and RGB images at higher resolutions. We incorporate a sinusoidal time embedding to effectively encode the time step of the diffusion process.

All experiments were conducted on Linux-based servers equipped with AMD Instinct MI250X GPUs with 64 GB memory each and NVIDIA GTX 2080 with 8GB memory. 
To efficiently utilize multiple GPUs within a single training run, we leverage Data Distributed Parallel (DDP) for scalable computation. Our code is publicly available at \url{https://github.com/NasikNafi/papa}. Here are the details of the hyperparameters used in the experiments for MNIST:

\begin{table}[H]
\begin{center}
\begin{tabular}{ll}
\multicolumn{1}{c}{\bf Hyperparameter}  &\multicolumn{1}{c}{\bf Values}
\\ \hline \\

\# num of timesteps &           1000 \\
noise scheduler $\beta$ start &           0.0001\\
noise scheduler $\beta$ end &           0.02\\
optimizer & ADAM\\
learning rate & 5e-4 \\
\# num of ResNet blocks &            6\\
\# num of attention heads &            4\\
K for EPAPA &  400 \\
QPDE coefficient $\beta$ & 0.009\\ 
time embedding dimension & 128 \\
total samples & 300 \\
\# num of interaction step &            30\\
\# num of samples per interaction &  10\\
\hline
\end{tabular}
\end{center}
\caption{Hyperparameters used for training and evaluation}
\label{table-hyp}
\end{table}

\section{Future Work}
In this work, we introduce a foundational framework for active preference alignment and conduct a comprehensive analysis of its core components, rigorously validating their roles and interactions using standard computer vision datasets. Looking forward, we are eager to expand this framework to tackle significant scientific challenges — such as active drug discovery and the generation of new metals and molecules — while dynamically accommodating target property preferences as they are sequentially revealed, thereby broadening the influence of active preference alignment in critical scientific domains.

\end{document}